\documentclass{article}

\usepackage[final]{neurips_2022}

\usepackage[utf8]{inputenc} %
\usepackage[T1]{fontenc}    %
\usepackage{hyperref}       %
\usepackage{url}            %
\usepackage{booktabs}       %
\usepackage{amsfonts}       %
\usepackage{nicefrac}       %
\usepackage{microtype}      %
\usepackage{xcolor}         %
\usepackage{verbatim}       %
\usepackage{bm}             %

\usepackage{amsmath,color,url}
\usepackage[capitalise]{cleveref}           %
\usepackage{setspace}
\usepackage{amssymb, amsthm}
\usepackage{makecell,multirow}
\usepackage{tikz,pgfplots}
\pgfplotsset{compat=1.17}
\usepackage{mathtools}
\usepackage{amsthm,threeparttable}
\usepackage{graphicx}
\usepackage[normalem]{ulem} %
\usepackage{algorithm,algpseudocode}
\usepackage{tikz, pgfplots, subfig} %
\usetikzlibrary{arrows,arrows.meta,positioning, fit,decorations.markings}
\usepackage[T1]{fontenc}

\tikzset{->-/.style={decoration={
  markings,
  mark=at position .5 with {\arrow{>}}},postaction={decorate}}}
\usepackage{selectp}

\newtheorem{theorem}{Theorem}
\newtheorem{definition}{Definition}
\newtheorem{lemma}{Lemma}
\newtheorem{proposition}{Proposition}

\newtheorem{corollary}[theorem]{Corollary}

\definecolor{fhcolor}{rgb}{0.523, 0.235, 0.625}

\newcommand{\mtt}[1]{\mathtt{#1}}

\providecommand{\realnum}{\mathbb{R}}

\DeclareMathOperator*{\argmin}{arg\min}
\DeclareMathOperator*{\argmax}{arg\max}
\DeclareMathOperator*{\lb}{\mathcal{L}}
\DeclareMathOperator*{\ub}{\mathcal{U}}
\DeclareMathOperator*{\mn}{\mathcal{M}}

\providecommand{\rebuttal}[1]{\textcolor{black}{#1}}

\title{Sound and Complete Verification of \\Polynomial Networks}
\author{%
  Elias Abad Rocamora\thanks{Work developed during an exchange coming from Universitat Politècnica de Catalunya (UPC), Spain. Currently at Universidad Carlos III de Madrid (UC3M).}\\
  LIONS, EPFL\\
  Lausanne, Switzerland\\
  \texttt{abad.elias00@gmail.com} \\
  \And
  Mehmet Fatih Sahin\\
  LIONS, EPFL\\
  Lausanne, Switzerland\\
  \texttt{mehmet.sahin@epfl.ch} \\
   \And
  Fanghui Liu\\
  LIONS, EPFL\\
  Lausanne, Switzerland\\
  \texttt{fanghui.liu@epfl.ch} \\
   \And
  Grigorios G Chrysos\\
  LIONS, EPFL\\
  Lausanne, Switzerland\\
  \texttt{grigorios.chrysos@epfl.ch} \\
   \And
  Volkan Cevher\\
  LIONS, EPFL\\
  Lausanne, Switzerland\\
  \texttt{volkan.cevher@epfl.ch} \\
}

\begin{document}

\maketitle

\begin{abstract}
Polynomial Networks (PNs) have demonstrated promising performance on face and image recognition recently. However, robustness of PNs is unclear and thus obtaining certificates becomes imperative for enabling their adoption in real-world applications. Existing verification algorithms on ReLU neural networks (NNs) based on classical branch and bound (BaB) techniques cannot be trivially applied to PN verification.
In this work, we devise a new bounding method, equipped with BaB for global convergence guarantees, called Verification of Polynomial Networks or VPN for short. %
One key insight is that we obtain much tighter bounds than the interval bound propagation (IBP) and DeepT-Fast \citep{Bonaert2021Fast} baselines.
This enables sound and complete PN verification with empirical validation on MNIST, CIFAR10 and STL10 datasets.
We believe our method has its own interest to NN verification. The source code is publicly available at \url{https://github.com/megaelius/PNVerification}.
\end{abstract}

\section{Introduction}
\label{sec:introduction}
Polynomial Networks (PNs) have demonstrated promising performance across image recognition and generation~\citep{Chrysos2021PolyNets, Chrysos2020NAPS}
being state-of-the-art on large-scale face recognition\footnote{\url{https://paperswithcode.com/sota/face-verification-on-megaface}}.%
Unlike the conventional Neural Networks (NNs), where non-linearitiy is introduced with the use of activation functions \citep{lecun2015deep}, PNs are able to learn non-linear mappings without the need of activation functions by exploiting multiplicative interactions (Hadamard products). Recent works have uncovered interesting properties of PNs in terms of model expressivity \citep{Fan2021Expressivity} and spectral bias~\citep{choraria2022the}. However, one critical issue before considering PNs for real-world applications is their robustness. %

Neural networks are prone to small (often imperceptible to the human eye), but malicious perturbations in the input data points~\citep{Szegedy2014, Goodfellow2015}. Those perturbations can have a detrimental effect on image recognition systems, e.g., as illustrated in face recognition~\citep{Goswami2019, zhong2019adversarial, dong2019efficient, Li2020FR}. Guarding against such attacks has so far proven futile~\citep{shafahi2018are, dou2018mathematical}. %
Instead, a flurry of research has been published on certifying robustness of NNs against this performance degradation 

\citep{Reluplex2017, Ehlers2017, Tjeng2019,Bunel2019BaB_first, Wang2021Beta-CROWN, ferrari2022mnBaB}. However, most of the verification algorithms for NNs are developed for the ReLU activation function by exploiting its piecewise linearity property and might not trivially extend to other nonlinear activation functions \citep{Wang2021Beta-CROWN}.%
Indeed, \citet{zhu2022LipschitzPN} illustrate that guarding PNs against adversarial attacks is challenging. Therefore, we pose the following question:
\begin{center}
\emph{Can we obtain certifiable performance for PNs against adversarial attacks?}    
\end{center}

In this work, we answer affirmatively and provide a method for the verification of PNs. Concretely, we take advantage of the twice-differentiable nature of PNs to build a lower bounding method based on  $\alpha$-convexification \citep{Adjiman1996alfa-conv}, which is integrated into a Branch and Bound (BaB) algorithm \citep{Land_Doig1960bab} to guarantee completeness of our verification method.
In order to use $\alpha$-convexification, a lower bound $\alpha$ of the minimum eigenvalue of the Hessian matrix over the possible perturbation set is needed. We use interval bound propagation together with the theoretical properties of the lower bounding Hessian matrix \citep{Adjiman1998alphabab}, in order to develop an algorithm to efficiently compute $\alpha$.

Our \emph{contributions} can be summarized as follows: {$(i)$} We propose the first algorithm for the verification of PNs. {$(ii)$} We thoroughly analyze the performance of our method by comparing it with a black-box solver, with an interval bound propagation (IBP) BaB algorithm \rebuttal{and with the zonotope-based abstraction method DeepT-Fast \citep{Bonaert2021Fast}}. {$(iii)$} We empirically show that using $\alpha$-convexitication for lower bounding provides tighter bounds than IBP \rebuttal{and DeepT-Fast} for PN verification. %
To encourage the community to improve the verification of PNs, we make our code publicly available in \url{https://github.com/megaelius/PNVerification}. %
The proposed approach can practically verify PNs and that could theoretically be applied for sound and complete verification of any twice-differentiable network. %

{\textbf{Notation:}} We use the shorthand $[n] := \{1,2,\dots, n\} $ for a positive integer $n$. We use bold capital (lowercase) letters, e.g., $\bm{X}$ ($\bm{x}$) for representing matrices (vectors). %
The $j^{\text{th}}$ column of a matrix $\bm{X}$ is given by $\bm{x}_{:j}$. The element in the $i^{\text{th}}$ row and $j^{\text{th}}$ column is given by $x_{ij}$, similarly, the $i^{\text{th}}$ element of a vector $\bm{x}$ is given by $x_{i}$. %
The element-wise (Hadamard) product, symbolized with $*$, of two matrices (or vectors) in $\realnum^{d_{1} \times d_{2}}$ (or $\realnum^{d}$) gives another matrix (or vector) in $\realnum^{d_{1} \times d_{2}}$ (or $\realnum^{d}$). 
The $\ell_{\infty}$ norm of a vector $\bm{x} \in \realnum^{d}$ is given by: $||\bm{x}||_{\infty} = \max_{i \in [d]} |x_{i}|$. Lastly, the operators $\lb$ and $\ub$ give the lower and upper bounds of a scalar, vector or matrix function by IBP, see \cref{subsec:IBP_PN}.

{\textbf{Roadmap:}} We provide the necessary background by introducing the PN architecture and formalizing the Robustness Verification problem in \cref{sec:background}. \cref{sec:method} provides a \textit{sound} and \textit{complete} method called VPN to tackle PN verification problem. %
\cref{sec:experiments} is devoted to experimental validation. Additional experiments, details and proofs are deferred to the appendix.

\section{Background}
\label{sec:background}
We give an overview of PN architecture in \cref{subsec:back_pn} and the robustness verification problem in \cref{subsec:back_verif}.
\vspace{-2.5mm}
\subsection{Polynomial Networks (PNs)}
\label{subsec:back_pn}
\vspace{-2.5mm}
Polynomial Networks (PNs) are inspired by the fact that any smooth function can be approximated via a polynomial expansion \citep{Stone1948Wierstrass}. 
However, the number of parameters increases exponentially with the polynomial degree, which makes it intractable to use high degree polynomials for high-dimensional data problems such as image classification where the input can be in the order of $10^5$ \citep{Imagenet2009}. \citet{Chrysos2021PolyNets} introduce a joint factorization of polynomial coefficients in a low-rank manner, reducing the number of parameters to linear with the polynomial degree and allowing the expression as a neural network (NN). We briefly recap one fundamental factorization below. %

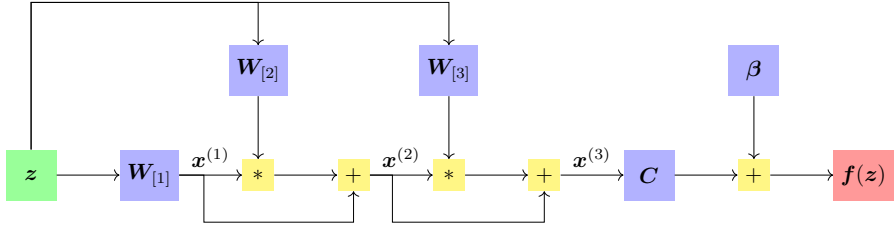
\begin{figure}[t]
    \begin{center}
    \resizebox{0.85\textwidth}{!}{
    \begin{tikzpicture}[
    greenbox/.style={rectangle, fill=green!40, minimum size=8mm},
    bluebox/.style={rectangle,fill=blue!30, minimum size=8mm},
    yellowbox/.style={rectangle,fill=yellow!60, minimum size=5mm},
    redbox/.style={rectangle,fill=red!40, minimum size=8mm},
    ]
    
    \node[greenbox] (z) {$\bm{z}$};
    
    \node[bluebox] (U1) [right = of z] {$\bm{W}_{[1]}$};
    \node[yellowbox] (h1) [right = of U1] {$*$};
    \node[yellowbox] (add1) [right = of h1] {$+$};
    \node[bluebox] (U2) [above = of h1] {$\bm{W}_{[2]}$};
    \node[yellowbox] (h2) [right = of add1] {$*$};
     \node[yellowbox] (add2) [right = of h2] {$+$};
    \node[bluebox] (U3) [above = of h2] {$\bm{W}_{[3]}$};
    
    \node[bluebox] (C) [right = of add2] {$\bm{C}$};
    
    \node[yellowbox] (add) [right = of C] {$+$};
    
    \node[bluebox] (b) [above = of add] {$\bm{\beta}$};
    
    \node[redbox] (output) [right = of add] {$\bm{f}(\bm{z})$};
    
    \draw[->] (z.east) to (U1.west);
    
    \draw[->] (z.north) -- (0,2.75) -| (U2.north);
    \draw[->] (z.north) -- (0,2.75) -| (U3.north);
    \draw[->] (U1.east) -- (2.75,0) -| (2.75,-0.75) -| (add1.south);
    \draw[->] (add1.east) -- (5.75,0) -| (5.75,-0.75) -| (add2.south);
    
    \draw[->] (U1.east) to node[above] {$\bm{x}^{(1)}$} (h1.west) ;
    \draw[->] (h1.east)  to  (add1.west);
    \draw[->] (U2.south)  to  (h1.north);
    \draw[->] (add1.east) to node[above] {$\bm{x}^{(2)}$} (h2.west);
    \draw[->] (h2.east)  to  (add2.west);
    \draw[->] (U3.south)  to  (h2.north);
    \draw[->] (add2.east) to node[above] {$\bm{x}^{(3)}$} (C.west);
    \draw[->] (C.east)  to  (add.west);
    \draw[->] (b.south)  to  (add.north);
    \draw[->] (add.east)  to  (output.west);
    
    \end{tikzpicture}}\hspace{2mm}
    \end{center}
    \caption{Third degree PN architecture. Blue boxes depict learnable parameters, yellow depict mathematical operations, the green and red boxes are the input and the output respectively. Note that no activation functions are involved, only element-wise (Hadamard) products $*$ and additions $+$. This figure represents the recursive formula of  \cref{eq:CCP_recursion}. 
    \vspace{-3mm}
    }
    \label{fig:CCP_arch}
\end{figure}
Let $N$ be the polynomial degree, $\bm{z} \in \realnum^{d}$ be the input vector, $d$, $k$ and $o$ be the input, hidden and output sizes, respectively. The recursive equation of PNs can be expressed as:
\begin{equation}
    \bm{x}^{(n)} = (\bm{W}_{[n]}^{\top}\bm{z})*\bm{x}^{(n-1)} + \bm{x}^{(n-1)}\,, \forall~ n \in [N]\,,
    \label{eq:CCP_recursion}
\end{equation}
where $\bm{x}^{(1)} = \bm{W}_{[1]}^{\top}\bm{z}$, $\bm{f}(\bm{z}) = \bm{C}\bm{x}^{(N)}+ \bm{\beta}$ and $*$ denotes the Hadamard product. $\bm{W}_{[n]} \in \realnum^{d \times k}$ and $\bm{C} \in \realnum^{o \times k}$ are weight matrices, $\bm{\beta} \in \mathbb{R}^{o}$ is a bias vector. A graphical representation of a third degree PN architecture corresponding to \cref{eq:CCP_recursion} can be found in \cref{fig:CCP_arch}. Further details on the factorization (as well as other factorizations) are deferred to the \cref{app:CCP} (\cref{app:NCP}).

\subsection{Robustness Verification}
\label{subsec:back_verif}
\vspace{-2.5mm}
Robustness verification \citep{Bastani2016constraints, Liu2021VerifAlgorithms} %
consists of verifying that a %
property regarding the input and output of a NN is satisfied, e.g. checking whether or not a small perturbation in the input will produce a change in the network output that makes it classify the input into another class. Let $f:{[0,1]}^{d} \to \realnum^{o}$ be a function, e.g., a NN or a PN, that classifies the input $\bm{z}$ into a class $c$, such that $c = \argmax{\bm{f}(\bm{z})}$. Our target is to verify that for any input satisfying a set of constraints $C_{\text{in}}$, the output of the network will satisfy a set of output constraints $C_{\text{out}}$. Mathematically, 
\begin{equation}
    \bm{z} \in C_{\text{in}} \implies \bm{f}(\bm{x}) \in C_{\text{out}}.
    \label{eq:Robust_def}
\end{equation}
In this work we focus on \emph{adversarial robustness} \citep{Szegedy2014, Carlini2017} in classification. Given an observation $\bm{z}_{0}$, let $t = \argmax{\bm{f}(\bm{z}_{0})}$ be the correct class, our goal is to check whether every input in a neighbourhood of $\bm{z}_{0}$, is classified as $t$.
In this work, we focus on adversarial attacks restricted to neighbourhoods defined in terms of $\ell_{\infty}$ norm, which is a popular norm-bounded attack in the verification community \citep{Liu2021VerifAlgorithms}. %
Then, the constraint sets become: %
\begin{equation}
    \begin{aligned}
    C_{\text{in}} & =  \{\bm{z}: ||\bm{z}-\bm{z}_{0}||_{\infty} \leq \epsilon, z_{i} \in [0,1], \forall i \in [d]\}\\
    & = \{\bm{z}: \max\{0,{z_{0}}_{i}-\epsilon\} \leq z_i \leq \min\{1,{z_{0}}_{i}+\epsilon\}, \forall i \in [d]\} \\
    C_{\text{out}} & = \{\bm{y}: y_{t} > y_j, \forall j \neq t \}\,.
    \end{aligned}
    \label{eq:input_output_constraints}
\end{equation}
In other words, we need an algorithm that given a function $f$, an input $\bm{z}_0$ and an adversarial budget $\epsilon$, checks whether \cref{eq:Robust_def} is satisfied. In the case of ReLU NNs, this has been proven to be an NP-complete problem \citep{Reluplex2017}. This can be reformulated as a constrained optimization problem. For every adversarial class $\gamma \neq t = \argmax{\bm{f}(\bm{z}_{0})}$, we can solve:
\begin{equation}
        \begin{aligned}
        \min_{\bm{z}} \quad & g(\bm{z}) = f(\bm{z})_t - f(\bm{z})_{\gamma} \quad
        \textrm{s.t.} \quad %
        \bm{z} \in \mathcal{C_{\text{in}}}\,.\\
        \end{aligned}
        \label{eq:verif_problem}
    \end{equation}
If the solution $\bm{z}^{*}$ with $v^{*} = f(\bm{z}^{*})_t - f(\bm{z}^{*})_{\gamma} \leq f(\bm{z})_t - f(\bm{z})_{\gamma}, \forall \bm{z} \in \mathcal{C_{\text{in}}}$ satisfies $v^{*} > 0$ then robustness is verified for the adversarial class ${\gamma}$. 

There are two main properties that a verification algorithm admits: \emph{soundness} and \emph{completeness}. An algorithm is \emph{sound} (\emph{complete}) if every time it verifies (falsifies) a property, it is guaranteed to be the correct answer. In practice, when an algorithm is guaranteed to provide the exact global minima of \cref{eq:verif_problem}, i.e., $v^{*}$, it is said to be \emph{sound} and \emph{complete} (usually referred in the literature as simply \emph{complete} \citep{ferrari2022mnBaB}), whereas if a lower bound of it is provided $\hat{v}^{*} \leq v^{*}$, the algorithm is \emph{sound} but not \emph{complete}. 
In our work, we will not consider just \emph{complete} verification, which simply aims at looking for adversarial examples, e.g., \citet{PGD_attacks2018}. For a deeper discussion on \emph{soundness} and \emph{completeness}, we refer to \citet{Liu2021VerifAlgorithms}.
\section{Method}
\label{sec:method}
Our method, called VPN, can be categorized in the the Branch and Bound (BaB) framework \citep{Land_Doig1960bab}, a well known approach to global optimization \citep{Horst1996GlobalOpt} %
and NN verification \citep{Bunel2019BaB_first}. This kind of algorithms ensures finding a global minima of the problem in \cref{eq:verif_problem} by recursively splitting the original feasible set into \rebuttal{smaller} sets (branching) where upper and lower bounds of the global minima are computed (bounding). This mechanism can be used to discard subsets where the global minima cannot be achieved (its lower bound is greater than the upper bound of another subset). 

Our method is based on a variant of BaB algorithm, i.e., $\alpha$-BaB \citep{Adjiman1998alphabab}, which is characterized for using $\alpha$-convexification \citep{Adjiman1996alfa-conv} for computing a lower bound of the global minima of each subset.
To be specific, $\alpha$-convexification aims to obtain a convex lower bounding function of any twice-differentiable function $f: \realnum^{d} \to \realnum$.
\rebuttal{In \citet{Adjiman1996alfa-conv}, they propose two methods:}
\begin{itemize}
    \item \rebuttal{\textbf{Uniform diagonal shift (single $\alpha$)}\vspace{-1.5mm}
    \begin{equation}
        g_{\alpha}(\bm{z}; \alpha, \bm{l}, \bm{u}) = g(\bm{z}) + \alpha\sum_{i = 1}^{d}(z_{i} - l_{i})(z_{i} - u_{i}) \,,
        \label{eq:alpha_conv}
    \end{equation}
    is its $\alpha$-convexified version, note that $z_{i}$ is the $i^{\text{th}}$ element of vector $\bm{z}$. Let $\bm{H}_{g}(\bm{z}) = \nabla_{\bm{z} \bm{z}}^{2} g(\bm{z})$ be the Hessian matrix of $g$, $g_{\alpha}$ is convex in $\bm{z} \in [\bm{l},\bm{u}]$ for $\alpha \geq \max\{0,-\frac{1}{2}\min \{ \lambda_{\text{min}}(\bm{H}_{g}(\bm{z})) : \bm{z} \in [\bm{l},\bm{u}]\}\}$, where $\lambda_{\text{min}}$ is the minimum eigenvalue. Moreover, it holds that $g_{\alpha}(\bm{z}; \alpha, \bm{l}, \bm{u}) \leq g(\bm{z}), \forall \bm{z} \in [\bm{l},\bm{u}]$.}
    \item \rebuttal{\textbf{Non-uniform diagonal shift (multiple $\alpha$'s)}\vspace{-1.5mm}
    \begin{equation}
        g_{\bm{\alpha}}(\bm{z}; \bm{\alpha}, \bm{l}, \bm{u}) = g(\bm{z}) + \sum_{i = 1}^{d}\alpha_i(z_{i} - l_{i})(z_{i} - u_{i}) \,,
        \label{eq:alpha_conv_multi}
    \end{equation}
    is its $\alpha$-convexified version. %
    In \citet{Adjiman1998alphabab}, they show that for any vector $\bm{d} > \bm 0$, setting
    \begin{equation}
    \begin{aligned}
        \alpha_{i} \geq \max{\left\{0,-\frac{1}{2}\left(\lb(h_{g}(\bm{z})_{ii})-\sum_{j \neq i}\max\{|\lb(h_{g}(\bm{z})_{ij})|,|\ub(h_{g}(\bm{z})_{ij})|\}\frac{d_{j}}{d_{i}}\right)\right\}}
    \end{aligned}
    \label{eq:alphai}
    \end{equation}
    makes $g_{\bm{\alpha}}$ convex in $\bm{z} \in [\bm{l},\bm{u}]$. The choice of the vector $\bm{d}$ is arbitrary, but affects the final result. %
    For example, taking $\bm{d} = \bm{u}-\bm{l}$ yields better results than $\bm d = \bm 1$ in \citet{Adjiman1998alphabab}. %
    } 
    We need to remark that, though \cref{eq:alpha_conv} is a special case of \cref{eq:alpha_conv_multi}, the lower bound of $\alpha$ in \cref{eq:alpha_conv} via minimum eigenvalue of the lower bounding Hessian matrix cannot be regarded as a special case of \cref{eq:alphai}.
\end{itemize}

To make PN verification feasible via $\alpha$-convexification, we need to study IBP for PNs and design an efficient estimate on the $\alpha$ \rebuttal{($\bm{\alpha}$ in the case of Non-uniform diagonal shift)} parameter, which are our main technical contributions in the algorithmic aspect.
In our case, every feasible set, starting with the input set $\mathcal{C_{\text{in}}}$ (\cref{eq:input_output_constraints}), is split by taking the widest variable interval and dividing it in two by the middle point. This is a rather simple, but theoretically powerful strategy, see \cref{lem:widest_interval_convergence} in \cref{app:proofs}. Then, %
the upper bound of each subproblem is given by applying standard Projected Gradient Descent (PGD) \citep{Kelley1999Iterative} %
over the original objective function. This is a common approach to find adversarial examples \citep{PGD_attacks2018}, but as the objective is non-convex, it is not sufficient for sound and complete verification. The lower bound is given by applying PGD over the $\alpha$-convexified objective $g_{\alpha}$ \rebuttal{($g_{\bm{\alpha}}$)}, as it is convex, PGD converges to the global minima and a lower bound of the original objective. The $\alpha$ \rebuttal{($\bm{\alpha}$)} parameter is computed only once per verification problem. Further details on the algorithm and the proof of convergence of \cref{eq:verif_problem} exist in \cref{app:BaB_algorithm}. \rebuttal{A schematic of our method is available in \cref{fig:main}.}
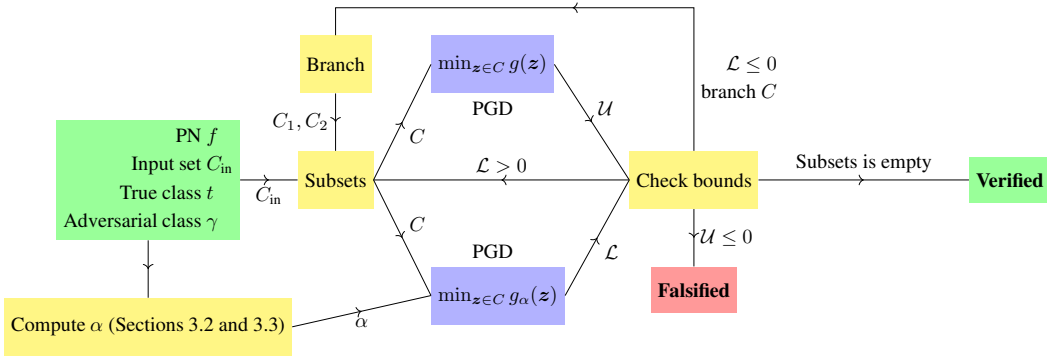
\begin{figure}[t]
    \begin{center}
    \resizebox{\textwidth}{!}{
    \begin{tikzpicture}[
    greenbox/.style={rectangle, fill=green!40, minimum size=8mm},
    bluebox/.style={rectangle,fill=blue!30, minimum size=10mm},
    yellowbox/.style={rectangle,fill=yellow!60, minimum size=10mm},
    redbox/.style={rectangle,fill=red!40, minimum size=8mm},
    ]
    
    \node[greenbox] (input) {$\begin{aligned}
\text{PN}~ & f\\
\text{Input set}~ & C_{\text{in}}\\
\text{True class}~ & t\\
\text{Adversarial class}~ & \gamma\\
\end{aligned}$};
    \node[yellowbox] (sets) [right = of input] {Subsets};
    \node[yellowbox] (branch) [above = of sets] {Branch};
    \node[yellowbox] (alpha) [below = of input] {Compute $\alpha$ (\cref{subsec:min_eig_lowerbound,subsec:min_eig_computation})};
    
     \node[bluebox] (upper) [above right = of sets] {$\min_{\bm{z} \in C} g(\bm{z})$};
     \node[bluebox] (lower) [below right = of sets] {$\min_{\bm{z} \in C} g_{\alpha}(\bm{z})$};
     
     \node[yellowbox] (check) at (9.5, 0) {Check bounds};
     \node[greenbox] (V) at (15, 0) {\textbf{Verified}};
     \node[redbox] (F) [below = of check] {\textbf{Falsified}};
     
     \node (pgdu) at (6,1.25) {PGD};
     \node (pgdl) at (6,-1.25) {PGD};
    
    \draw[->-] (input.south) to (alpha.north);
    \draw[->-] (input.east) to node[below] {$C_{\text{in}}$} (sets.west);
    
    \draw[->-] (sets.east) to node[below right] {$C$} (upper.west);
    \draw[->-] (sets.east) to node[above right] {$C$} (lower.west);
    \draw[->-] (alpha.east) to node[below] {$\alpha$} (lower.west);
    
    \draw[->-] (upper.east) to node[above right] {$\ub$} (check.west);
    \draw[->-] (lower.east) to node[below right] {$\lb$} (check.west);
    
    \draw[->-] (check.north) -- node[right] {$\begin{aligned} \lb \leq 0\\
    \text{branch } C\end{aligned}$} (9.5,3) -|  (branch.north);
    \draw[->-] (branch.south) to node[left] {$C_{1}, C_{2}$} (sets.north);
    
    \draw[->-] (check.west) to node[above] {$\lb > 0$} (sets.east);
    \draw[->-] (check.east) to node[above] {Subsets is empty} (V.west);
    \draw[->-] (check.south) to node[right] {$\ub \leq 0$} (F.north);

    \end{tikzpicture}}
    \end{center}
    \caption{\rebuttal{Overview of our branch and bound verification algorithm. Given a trained PN $f$, an input set $C_{\text{in}}$, the true class $t$ and an adversarial class $\gamma$, we check if an adversarial example exists (\textbf{Falsified}) or not (\textbf{Verified}). Note that the branching of a subset $C$ provides two smaller subsets $C_{1}$ and $C_{2}$. Also note that when $\lb > 0$, no subset is added to subsets.}}
    \vspace{-5mm}
    \label{fig:main}
\end{figure}

In \rebuttal{\cref{subsec:IBP_PN,subsec:min_eig_lowerbound,subsec:min_eig_computation}}, 
we detail our method under the uniform diagonal shift case to compute a lower bound on the minimum eigenvalue of the Hessian matrix into three main components: interval propagation, lower bounding Hessians, and fast estimation on such lower bounding via power method. \rebuttal{To conclude the description of our method, in \cref{subsec:multiple_alphas} we describe the $\bm{\alpha}$ estimation for the non-uniform diagonal shift case.}
\vspace{-2.5mm}
\subsection{Interval Bound Propagation through a PN}
\label{subsec:IBP_PN}
\vspace{-2.5mm}
Interval bound propagation (IBP) is a key ingredient of our verification algorithm. Suppose we have an input set defined by an $\ell_{\infty}$-norm ball like in \cref{eq:input_output_constraints}. This set can be represented as a vector of intervals  $[\bm{l}, \bm{u}] = ([l_1, u_1]^{\!\top}, [l_2, u_2]^{\!\top}, \cdots, [l_d, u_d]^{\!\top}) \in \mathbb{R}^{d \times 2}$, where $[l_i, u_i]$ are the lower and upper bound for the $i^{\text{th}}$ coordinate. %
Let $\lb$ and $\ub$ be the lower and upper bound IBP operators. Given this input set, we would like to obtain bounds on the output of the network $(f(\bm{z})_{i})$, %
the gradient $(\nabla_{\bm{z}} f(\bm{z})_{i})$, %
and the Hessian $(\nabla_{\bm{z} \bm{z}}^{2}f(\bm{z})_{i})$ for any $\bm{z} \in [\bm{l}, \bm{u}]$. The operators $\lb(g(\bm{z}))$ and $\ub(g(\bm{z}))$ of any function $g : \realnum^{d} \to \realnum$ satisfy:
\begin{equation}
    \begin{matrix}
        \lb(g(\bm{z})) \leq g(\bm{z}), &
        \ub(g(\bm{z})) \geq g(\bm{z}), \forall \bm{z} \in [\bm{l}, \bm{u}]\,.
    \end{matrix}
\end{equation}
We will define these upper and lower bound operators in terms of the operations present in a PN. Using interval propagation \citep{Moore2009IntroductionTI}, denote the positive part $w_{i}^{+} = \max\{0,w_{i}\}$ and the negative part $w_{i}^{-} = \min\{0,w_{i}\}$, $h_i(\bm{z})$ as any a real-valued function of $\bm{z}$, $i \in [d]$, we can define:
\begin{equation}
    \begin{aligned}
        \textbf{Identity} & 
        \left\{ 
        \begin{aligned}
            \lb(z_i) &= l_{i}\\
            \ub(z_i) &= u_{i}\\
        \end{aligned}
        \right. \\
        \textbf{linear mapping} & 
        \left\{ 
        \begin{aligned}
            \lb(\sum_{i}w_{i}h_{i}(\bm{z})) = \sum_{i}w_{i}^{+}\lb(h_i(\bm{z})) + w_{i}^{-}\ub(h_i(\bm{z}))\\
            \ub(\sum_{i}w_{i}h_{i}(\bm{z})) = \sum_{i}w_{i}^{-}\lb(h_i(\bm{z})) + w_{i}^{+}\ub(h_i(\bm{z}))\\
        \end{aligned}
        \right. \\
        \textbf{multiplication} & 
        \left\{ 
        \begin{aligned}
            S = \left\{\begin{aligned}\lb(h_1(\bm{z}))\lb(h_2(\bm{z})),\\\lb(h_1(\bm{z}))\ub(h_2(\bm{z})),\\\ub(h_1(\bm{z}))\lb(h_2(\bm{z})),\\\ub(h_1(\bm{z}))\ub(h_2(\bm{z}))\end{aligned}\right\}&, |S| = 4\\
            \lb(h_1(\bm{z})h_2(\bm{z})) = \min{S}&\,,\\
            \ub(h_1(\bm{z})h_2(\bm{z})) = \max{S}&\,,\\
        \end{aligned}
        \right. \\ 
    \end{aligned}
    \label{eq:IBP_basic_rules}
\end{equation}
where $|\cdot|$ is the set cardinality. Note that the set $S$ is equivalent to: \[S = \left\{ ab \big|~\forall a \in \left\{ \lb(h_1(\bm{z})),\ub(h_1(\bm{z})) \right\}, \forall b \in \left\{ \lb(h_2(\bm{z})),\ub(h_2(\bm{z})) \right\} \right\}\,.\]
With these basic operations, one can define bounds on any intermediate output, gradient or Hessian of a PN. For instance, the lower bound on the recursive formula from \cref{eq:CCP_recursion} can be expressed as:
\begin{equation}
    \lb(x_i^{(n)}) = \lb((\bm{w}_{[n]:i}^{\top}\bm{z})x_i^{(n-1)} + x_i^{(n-1)}) = \lb((\bm{w}_{[n]:i}^{\top}\bm{z} + 1)x_i^{(n-1)}),~~\forall i \in [k], n \in \{ 2, \dots, N\} \,,
\end{equation}
which only consists on a linear mapping and a multiplication of intervals. We extend the upper and lower bound ($\lb(\cdot)$ and $\ub(\cdot)$) operators to vectors and matrices in an entry-wise style:
\begin{equation}
    \begin{matrix}
        \lb({\bm{g}(\bm{z})}) = \left[\begin{matrix}
            \lb(g(\bm{z})_1)\\
            \lb(g(\bm{z})_2)\\
            \vdots\\
            \lb(g(\bm{z})_m)\\
        \end{matrix}\right] \in \mathbb{R}^m, & \begin{matrix}
        \lb({\bm{G}(\bm{z})}) = \left[\begin{matrix}
            \lb(g(\bm{z})_{11}) & \cdots & \lb(g(\bm{z})_{1m})\\
            \vdots & \ddots\\
            \lb(g(\bm{z})_{m1}) & & \lb(g(\bm{z})_{mm})\\
        \end{matrix}\right] \in \mathbb{R}^{m \times m}
    \end{matrix}\,.
    \end{matrix}
    \label{eq:IBP_vector_matrix}
\end{equation}
Note that \cref{eq:IBP_vector_matrix} is not limited to squared matrices and can hold for arbitrary matrix dimensions. 
One can directly use IBP to obtain bounds on the verification objective from \cref{eq:verif_problem} with a single forward pass of the bounds through the network and obtaining $\lb(g(\bm{z})) = \lb(f(\bm{z})_t) - \ub(f(\bm{z})_\gamma)$. 
In fact IBP is a common practice in NN verification to obtain fast bounds \citep{Wang2018intervals}.

\subsection{Lower bound of the minimum eigenvalue of the Hessian}
\label{subsec:min_eig_lowerbound}
\vspace{-2.5mm}

Here we describe our method to compute a lower bound on the minimum eigenvalue of the Hessian matrix in the feasible set.
Before deriving the lower bound, we need the first and second order partial derivatives of PNs.

Let $g(\bm{z}) = f(\bm{z})_{t} - f(\bm{z})_{a}$ be the objective function for $t = \argmax{\bm{f}(\bm{z}_{0})}$ and any $a \neq t$.
In order to compute the parameter $\alpha$ for performing $\alpha$-convexification, we need to know the structure of our objective function. In this section we compute the first and second order partial derivatives of the PN. The gradient and Hessian matrices of the objective function (see \cref{eq:verif_problem}) are given by:
\begin{equation}
    \nabla_{\bm{z}} g(\bm{z}) = \sum_{i = 1}^{k}(c_{t i} - c_{\gamma i}) \nabla_{\bm{z}} x_{i}^{(N)},~~ \bm{H}_{g}(\bm{z}) = \sum_{i = 1}^{k}(c_{t i} - c_{\gamma i}) \nabla_{\bm{z} \bm{z}}^{2} x_{i}^{(N)}\,.
    \label{eq:grad_and_hess_objective}
\end{equation}
We now define the gradients $\nabla_{\bm{z}} x_{i}^{(n)}$ and Hessians $\nabla_{\bm{z} \bm{z}}^{2} x_{i}^{(n)}$ of \cref{eq:CCP_recursion} in a recursive way:
\begin{equation}
    \nabla_{\bm{z}} x_{i}^{(n)}  = \bm{w}_{[n]:i} \cdot x_{i}^{(n-1)} + (\bm{w}_{[n]:i}^{\top}\bm{z} + 1) \cdot \nabla_{\bm{z}} x_{i}^{(n-1)}
    \label{eq:grad_CCP_recursive}
    \vspace{-2.5mm}
\end{equation}
\begin{equation}
    \nabla_{\bm{z} \bm{z}}^{2} x_{i}^{(n)} = \nabla_{\bm{z}}x_{i}^{(n-1)} \bm{w}_{[n]:i}^{\top} + \{\nabla_{\bm{z}}x_{i}^{(n-1)} \bm{w}_{[n]:i}^{\top}\}^\top + (\bm{w}_{[n]:i}^{\top}\bm{z} + 1) \nabla_{\bm{z} \bm{z}}^{2} x_{i}^{(n-1)}\,,
    \label{eq:hess_CCP_recursive}
\end{equation}
with $\nabla_{\bm{z}} x_{i}^{(1)} = \bm{w}_{[1]:i}$ %
and $\nabla_{\bm{z} \bm{z}}^{2} x_{i}^{(1)} = \bm{0}_{d\times d}$ being an all-zero matrix. In the next, we are ready to compute a lower bound on the minimum eigenvalue of the Hessian matrix in the feasible set.

Firstly, for any $\bm{z} \in [\bm{l}, \bm{u}]$ and any polynomial degree $N$, we can express the set of possible Hessians $\mathcal{H} = \{\bm{H}_{g}(\bm{z}) : \bm{z} \in [\bm{l}, \bm{u}]\}$ as an interval matrix. An interval matrix is a tensor $[\bm{M}] \in \realnum^{d \times d \times 2}$ where every position $[m]_{ij} = [\lb(m_{ij}), \ub(m_{ij})]$ is an interval. Therefore, if $\bm{H}_{g}(\bm{z})$ is bounded for $\bm{z} \in [\bm{l}, \bm{u}]$, then we can represent $\mathcal{H} = \{ \bm{H}_{g}(\bm{z}): \bm{H}_{g}(\bm{z}) \in [\bm{M}]\} = \{ \bm{H}_{g}(\bm{z}):  \lb(m_{ij}) \leq H_{g}(\bm{z})_{ij} \leq \ub(m_{ij}), \forall i,j \in [d]\}$.

Let $\lb(\bm{M})$ and $\ub(\bm{M})$ be the element-wise lower and upper bounds of a Hessian matrix, the lower bounding Hessian is defined as follows:
\begin{equation}
    \bm{L_{H}} = %
    \frac{\lb(\bm{M}) + \ub(\bm{M})}{2} + \text{diag}\left(\frac{\lb(\bm{M})\bm{1} - \rebuttal{\ub(\bm{M})}\bm{1}}{2} \right) \,,
\end{equation}
where $\bm{1}$ is an all-one vector and $\text{diag}(\bm{v})$ is a diagonal matrix with the vector $\bm{v}$ in the diagonal. %
Described in \citet{Adjiman1998alphabab}, This matrix satisfies that %
$\lambda_{\text{min}}(\bm{L_{H}}) \leq \lambda_{\text{min}}(\bm{H}_{g}(\bm{z})), \forall \bm{H}_{g}(\bm{z}) \in \mathcal{H}, \bm{z} \in [\bm{l}, \bm{u}]$.

Then, we can obtain the spectral radius $\rho(\bm{L_{H}})$ with a power method \citep{Mises1929PowerMethod}. As the spectral radius satisfies $\rho(\bm{L_{H}}) \geq |\lambda_{i}(\bm{L_{H}})|, \forall i \in [d]$, the following inequality holds:
\begin{equation}
    -\rho(\bm{L_{H}}) \leq \lambda_{\text{min}}(\bm{L_{H}}) \leq \lambda_{\text{min}}(\bm{H}_{g}(\bm{z})), \;\forall \bm{H}_{g}(\bm{z}) \in \mathcal{H}, \;\;\bm{z} \in [\bm{l}, \bm{u}]\,,
    \label{eq:bounds_min_eig}
\end{equation}
allowing us to use $\alpha = \frac{\rho(\bm{L_{H}})}{2} \geq \max\{0,-\frac{1}{2}\min \{ \lambda_{\text{min}}(\bm{H}_{f}(\bm{z})) : \bm{z} \in [\bm{l},\bm{u}]\}\}$.

\subsection{Efficient power method for spectral radius computation of the lower bounding Hessian}
\label{subsec:min_eig_computation}
\vspace{-2.5mm}
By using interval propagation, one can easily compute sound lower and upper bounds on each position of the Hessian matrix, compute the lower bounding Hessian %
and perform a power method with it to obtain the spectral radius $\rho$. However, this method would not scale well to high dimensional scenarios. For instance, in the STL10 dataset \citep{STL10} %
with the input dimension $d = 96\cdot96\cdot3 = 27,648$ in each color image, our Hessian matrix would %
require in the order of $O(d^{2}) = O(10^{9})$ %
real numbers to be stored. This makes it intractable to perform a power method over such an humongous matrix, or even to compute the lower bounding Hessian.
Alternatively, we take advantage of \rebuttal{the possibility of expressing the $\bm{L}_{\bm{H}}$ matrix as a sum of rank-1 matrices, to enable performing a power method over it.}

\textbf{Standard power method for spectral radius computation}

Given any squared and real valued matrix $\bm{M} \in \realnum^{d \times d}$ and an initial vector $\bm{v}_0 \in \realnum^{d}$ that is not an eigenvector of $\bm{M}$, the sequence:\vspace{-2.5mm}
\begin{equation}
    \bm{v}_{n} = \frac{\bm{M}(\bm{M}\bm{v}_{n-1})}{||\bm{M}(\bm{M}\bm{v}_{n-1})||_{2}}\,,
    \label{eq:power_method}
\end{equation}
converges to the eigenvector with the largest eigenvalue in absolute value, i.e. the eigenvector where the spectral radius is attained, being the spectral radius $\rho(\bm{M}) = \sqrt{||\bm{M}(\bm{M}\bm{v}_{n-1})||_{2}}$ \citep{Mises1929PowerMethod}.

\textbf{Power method over lower bounding Hessian of PNs}

We can employ IBP (\cref{subsec:IBP_PN}) in order to obtain an expression of the lower bounding Hessian ($\bm{L}_{\bm{H}}$) and evaluate \cref{eq:power_method} as:
\begin{equation}
    \begin{aligned}
    \bm{L_{H}}\bm{v} %
    & = \frac{\ub(\bm{H}_{g}(\bm{z}))\bm{v}  + \lb(\bm{H}_{g}(\bm{z}))\bm{v} }{2} + \left(\frac{\lb(\bm{H}_{g}(\bm{z}))\bm{1} - \ub(\bm{H}_{g}(\bm{z}))\bm{1}}{2}\right) * \bm{v}\,.
    \end{aligned}
    \label{eq:lower bounding_hessian_vector}
\end{equation}
Applying IBP on \cref{eq:grad_and_hess_objective} we obtain:
\begin{equation}
\begin{aligned}
    \lb(\bm{H}_{g}(\bm{z}))\bm{v} %
    & = \sum_{i = 1}^{k}(c_{t i} - c_{\gamma i})^{+} \lb(\nabla_{\bm{z} \bm{z}}^{2} x_{i}^{(N)})\bm{v} + \sum_{i = 1}^{k}(c_{t i} - c_{\gamma i})^{-} \ub(\nabla_{\bm{z} \bm{z}}^{2} x_{i}^{(N)})\bm{v}\\
    \ub(\bm{H}_{g}(\bm{z}))\bm{v} & = \sum_{i = 1}^{k}(c_{t i} - c_{\gamma i})^{-} \lb(\nabla_{\bm{z} \bm{z}}^{2} x_{i}^{(N)})\bm{v} + \sum_{i = 1}^{k}(c_{t i} - c_{\gamma i})^{+} \ub(\nabla_{\bm{z} \bm{z}}^{2} x_{i}^{(N)})\bm{v}\,.
    \label{eq:lower_upper_last_layer}
\end{aligned}
\end{equation}
We can recursively evaluate $\lb(\nabla_{\bm{z} \bm{z}}^{2} x_{i}^{(n)})\bm{v}$ and $\ub(\nabla_{\bm{z} \bm{z}}^{2} x_{i}^{(n)})\bm{v}$ efficiently as these matrices can be expressed as a sum of rank-$1$ matrices as below. %
\begin{proposition}
Let  $\delta \in [\lb(\delta), \ub(\delta)]$ be a real-valued weight, the matrix-vector products $\lb(\delta \cdot \nabla_{\bm{z} \bm{z}}^{2} x_{i}^{(n)})\bm{v}$ and $\ub(\delta \cdot \nabla_{\bm{z} \bm{z}}^{2} x_{i}^{(n)})\bm{v}$ can be evaluated as: %
\begin{equation}
    \begin{aligned}
    \lb(\delta \cdot \nabla_{\bm{z} \bm{z}}^{2} x_{i}^{(n)})\bm{v}  = &\lb(\delta \cdot \nabla_{\bm{z}}x_{i}^{(n-1)}) {\bm{w}_{[n]:i}^{+\top}}\bm{v} + \ub(\delta \cdot \nabla_{\bm{z}}x_{i}^{(n-1)}) {\bm{w}_{[n]:i}^{-\top}}\bm{v}\\ 
    & + {\bm{w}_{[n]:i}}^{+} \lb(\delta \cdot {\nabla_{\bm{z}}x_{i}^{(n-1)}}^{\top})\bm{v} + {\bm{w}_{[n]:i}}^{-} \ub(\delta \cdot {\nabla_{\bm{z}}x_{i}^{(n-1)}}^{\top})\bm{v}  \\
    & + \lb(\delta'\nabla_{\bm{z} \bm{z}}^{2} x_{i}^{(n-1)})\bm{v}\,,
    \end{aligned}
    \label{eq:lower_Hessian_vector}
\end{equation}
\begin{equation}
    \begin{aligned}
    \ub(\delta \cdot \nabla_{\bm{z} \bm{z}}^{2} x_{i}^{(n)})\bm{v}  = &\lb(\delta \cdot \nabla_{\bm{z}}x_{i}^{(n-1)}) {\bm{w}_{[n]:i}^{-\top}}\bm{v} + \ub(\delta \cdot \nabla_{\bm{z}}x_{i}^{(n-1)}) {\bm{w}_{[n]:i}^{+\top}}\bm{v}\\ 
    & + {\bm{w}_{[n]:i}}^{-} \lb(\delta \cdot {\nabla_{\bm{z}}x_{i}^{(n-1)}}^{\top})\bm{v} + {\bm{w}_{[n]:i}}^{+} \ub(\delta \cdot {\nabla_{\bm{z}}x_{i}^{(n-1)}}^{\top})\bm{v}  \\
    & + \ub(\delta'\nabla_{\bm{z} \bm{z}}^{2} x_{i}^{(n-1)})\bm{v}\,, \\
    \end{aligned}
    \label{eq:upper_Hessian_vector}
\end{equation}
where $ \delta' \in [\lb(\delta), \ub(\delta)] \cdot [\lb(\bm{w}_{[n]:i}^{\top}\bm{z} + 1), \ub(\bm{w}_{[n]:i}^{\top}\bm{z} + 1)]$ and vectors $\lb(\delta \cdot {\nabla_{\bm{z}}x_{i}^{(n-1)}})$ and $\ub(\delta \cdot {\nabla_{\bm{z}}x_{i}^{(n-1)}})$ can be obtained through IBP on \cref{eq:grad_CCP_recursive}.
\label{prop:lower bounding_Hessian_vector}
\end{proposition}
Lastly, by applying recursively \cref{prop:lower bounding_Hessian_vector} from $n = N$ to $n=1$, starting with $\delta = 1$, we can substitute the results on \cref{eq:lower_upper_last_layer} and then on \cref{eq:lower bounding_hessian_vector} to efficiently evaluate a step of the power method (\cref{eq:power_method}) without needing to store the lower bounding Hessian matrix or needing to perform expensive matrix-vector products.

Overall, our lower bounding method consists in computing a valid value of $\alpha$ that satisfies that the $\alpha$-convexified objective $g_{\alpha}$ is convex, following \cref{eq:verif_problem,eq:alpha_conv}. In particular, we use $\alpha = \frac{\rho(\bm{L_{H}})}{2}$. $\rho(\bm{L_{H}})$ is computed via a power method, where the main operation $\bm{L_{H}}\bm{v}$ is evaluated without the need to compute or store the $\bm{L_{H}}$ matrix. Provided this valid $\alpha$, we perform PGD over $g_{\alpha}$ and this provides a lower bound of the global minima of \cref{eq:verif_problem}.
\rebuttal{\subsection{Non-uniform diagonal shift}
\label{subsec:multiple_alphas}}
\vspace{-2.5mm}
\rebuttal{In order to obtain an estimate of the $\bm{\alpha}$ parameter as defined by \citet{Adjiman1998alphabab} in \cref{eq:alphai}, we make use of the rank-1 matrices IBP rules defined in \cref{subsec:IBP_matrices}. We also define the operator $\mn(\cdot)=\max\{|\lb(.)|, |\ub(.)|\}$ and certain useful properties about it in \cref{subsec:mn_operator}. Thanks to \cref{lem:maxnorm_matrix,lem:maxnorm_matrix_sum,lem:maxnorm_matrix_mult}, we can obtain an expression for $\mn\left(\bm{H}_{\bm{z}}(g)\right)$. This will be used to compute the vector $\bm{\alpha}$ in the Non-uniform diagonal shift scenario for $\alpha$-convexification.}
\begin{theorem}
\rebuttal{Let $f$ be a $N$-degree CCP PN defined as in \cref{eq:CCP_recursion}. Let $g(\bm{z}) = f(\bm{z})_{t} - f(\bm{z})_{\gamma}$ for any $t \neq \gamma, t \in [o], \gamma \in [o]$. Let $\bm{H}_{g}(\bm{z})$ be the Hessian matrix of g, the operation $\mn\left(\bm{H}_{g}(\bm{z})\right)$ results in:
\begin{equation}
    \mn\left(\bm{H}_{g}(\bm{z})\right) \leq \sum_{i = 1}^{k}|c_{ti} - c_{\gamma i}|\mn\left(\nabla_{\bm{z} \bm{z}}^{2} x_{i}^{(N)}\right)\,,
\end{equation}
where for $n = 2, ..., N$, we can express:
\begin{equation}
    \begin{aligned}
        \mn\left(\nabla_{\bm{z} \bm{z}}^{2} x_{i}^{(n)}\right) & \leq \mn\left(\nabla_{\bm{z}}x_{i}^{(n-1)}\right)|\bm{w}_{[n]i:}|^{\top} + |\bm{w}_{[n]i:}|\mn\left(\nabla_{\bm{z}}x_{i}^{(n-1)}\right)^{\top}\\
        & \quad + \mn\left(\bm{w}_{[n]i:}^{\top}\bm{z} + 1\right)\mn\left(\nabla_{\bm{z} \bm{z}}^{2} x_{i}^{(n-1)}\right)\,.
    \end{aligned}
\end{equation}
Lastly, for $n = 1$, $\mn\left(\nabla_{\bm{z} \bm{z}}^{2} x_{i}^{(1)}\right) = \bm{0}_{d \times d}$ is a $d \times d$ matrix full of zeros.}
\label{theo:mn_PN}
\end{theorem}
\rebuttal{As one can observe, in \cref{theo:mn_PN}, the matrix $\mn\left(\bm{H}_{g}(\bm{z})\right)$ is expressed as a sum of rank-1 matrices. This allows to efficiently compute $\sum_{j \neq i}\mn(h_{g}(\bm{z})_{ij})\frac{d_{i}}{d_{j}}$, which is necessary for the right term in \cref{eq:alphai}. 
For the left term in \cref{eq:alphai}, we can efficiently compute the lower bound of the diagonal of the Hessian matrix by using the rules present in \cref{subsec:IBP_PN,subsec:IBP_matrices}.}
\section{Experiments}
\label{sec:experiments}
In this Section we show the efficiency of our method by comparing against a simple Black-box solver. Tightness of bounds is also analyzed in comparison with IBP \rebuttal{and DeepT-Fast \citep{Bonaert2021Fast}, a zonotope based verification method able to handle multiplications tighter than IBP}. Finally, a study of the performance of our method in different scenarios is performed. %
Unless otherwise specified, every network is trained for $100$ epochs with Stochastic Gradiend Descent (SGD), with a learning rate of $0.001$, which is divided by $10$ at epochs $[40,60,80]$, momentum $0.9$, weight decay $5\cdot 10^{-5}$ and batch size $128$. We thoroughly evaluate our method over the popular image classification datasets MNIST \citep{lecun1998gradient}, CIFAR10 \citep{CIFAR10} and STL10 \citep{STL10}. Every experiment is done over the first $1000$ images of the test dataset, this is a common practice in verification \citep{Singh2019}. For images that are correctly classified by the network, we sequentially verify robustness against the remaining classes in decreasing order of network output. Each verification problem is given a maximum execution time of $60$ seconds, we include experiments with different time limits in \cref{app:Appendix}. Note that the execution time can be longer as %
execution is cut in an asynchronous way, i.e., after we finish the iteration of the BaB algorithm where the time limit is reached. \rebuttal{{All of our experiments were conducted on a single GPU node equipped with a 32 GB NVIDIA V100 PCIe}}. %

\subsection{Comparison with a Black-box solver}
\label{subsec:experiments_gurobi}
\vspace{-2.5mm}
In this experiment, we compare the performance of our BaB verification algorithm with the Black-box
solver Gurobi \citep{gurobi}. Gurobi can globally solve Quadratically Constrained Quadratic Programs whether they are convex or not. As this solver cannot extend to higher degree polynomial functions, we train $2^{\text{nd}}$ degree PNs with hidden size $k = 16$ %
to compare the verification time of our method with Gurobi. In order to do so, we express the verification objective as a quadratic form $g(\bm{z}) = f(\bm{z})_{t} - f(\bm{z})_{a} = \bm{z}^{\top}\bm{Q}\bm{z} + \bm{q}^{\top}\bm{z} + c$ this together with the input constraints $\bm{z} \in [\bm{l}, \bm{u}]$ is fed to Gurobi and optimized until convergence.
\begin{table}[tb]
    \centering
    \caption{Verification results for $2^{\text{nd}}$ degree PNs. Columns {\tt\#F}, {\tt\#T} and {\tt\#t.o.} refer to the number of images where robustness is falsified, verified and timed-out respectively. When comparing with a black-box solver, our method is much faster and can scale to higher dimensional inputs. This is due to our efficient exploitation of the low-rank factorization of PNs.}
    \resizebox{\textwidth}{!}{%
    \begin{tabular}{c|ccc|cccc|ccccc}
    \multirow{2}{*}{Dataset} & \multirow{2}{*}{Model} & \multirow{2}{*}{Correct} & \multirow{2}{*}{$\epsilon$} & \multicolumn{4}{c}{VPN (Our method)} & \multicolumn{4}{c}{Gurobi}   \\ 
 &  &  &  & time & F & T & t.o. & time & F & T & t.o. \\ 
\hline 
\multirow{1}{*}{MNIST} & \multirow{1}{*}{$2\times16$} & $961$ & $0.00725$ & $\mathbf{1.76}$ & $37$ & $924$ & $0$ & $16.6$ & $37$ & $924$ & $0$\\ 
$(1\times28\times28)$ &  &  & $0.013$ & $\mathbf{1.78}$ & $71$ & $890$ & $0$ & $15.13$ & $71$ & $890$ & $0$\\ 
 &  &  & $0.05$ & $\bf{1.43}$ & $682$ & $267$ & $12$ & $6.25$ & $691$ & $270$ & $0$\\ 
 &  &  & $0.06$ & $\mathbf{1.5}$ & $790$ & $155$ & $16$ & $4.47$ & $799$ & $162$ & $0$\\ 
\hline 
\multirow{1}{*}{CIFAR10} & \multirow{1}{*}{$2\times16$} & $460$ & $1/610$ & $\mathbf{1.03}$ & $90$ & $370$ & $0$ & $328.0$ & $90$ & $370$ & $0$\\ 
$(3\times32\times32)$ &  &  & $1/255$ & $\mathbf{1.0}$ & $183$ & $277$ & $0$ & $250.07$ & $183$ & $277$ & $0$\\ 
 &  &  & $4/255$ & $\mathbf{0.92}$ & $427$ & $28$ & $5$ & $87.93$ & $429$ & $31$ & $0$\\ 
\hline 
\multirow{1}{*}{STL10} & \multirow{1}{*}{$2\times16$} & $362$ & $1/610$ & $\mathbf{5.06}$ & $142$ & $220$ & $0$ & \multicolumn{4}{c}{\multirow{3}{*}{out of memory}}\\ 
$(3\times96\times96)$ &  &  & $1/255$ & $\mathbf{3.61}$ & $246$ & $113$ & $3$ &  &  &  & \\ 
 &  &  & $4/255$ & $\mathbf{1.39}$ & $360$ & $1$ & $1$ &  &  &  & \\ 
\hline 
    \end{tabular}}
    \label{tab:preliminary_results}
    \vspace{-5mm}
\end{table}

The black-box solver approach neither scales to higher-dimensional inputs nor to higher polynomial degrees. With this approach we need $\mathcal{O}(d^{2})$ memory to store the quadratic form, which makes it unfeasible for datasets with higher resolution images than CIFAR10. On the contrary, as seen in \cref{tab:preliminary_results}, our approach does not need so much memory and can scale to datasets with larger input sizes like STL10. 

\subsection{Comparison with \rebuttal{IBP and DeepT-Fast}}
\label{subsec:experiments_IBP}
\vspace{-2.5mm}
In this experiment we compare the tightness of the lower bounds provided by IBP\rebuttal{, DeepT-Fast} and $\alpha$-convexification and their effectiveness when employed for verification. This is done by executing one upper bounding step with PGD and one lower bounding step for each lower bounding method over the initial feasible set provided by $\epsilon$ (see \cref{eq:input_output_constraints}). We compare the average of the distance from each lower bound to the PGD upper bound over the first $1000$ images of the MNIST dataset for PNs with hidden size $k = 25$ and degrees ranging from $2$ to $7$. We also evaluate verified accuracy of $2^{\text{nd}}$ (PN\_Conv2) and $4^{\text{th}}$ (PN\_Conv4) PNs with \rebuttal{IBP, DeepT-Fast and $\alpha$-convexification in the Uniform diagonal shift setup}. We employ a maximum time of $120$ seconds. For details on the architecture of these networks, we refer to \cref{app:Appendix}.
\begin{figure}[ht]
    \centering
    \scalebox{0.365}{\input{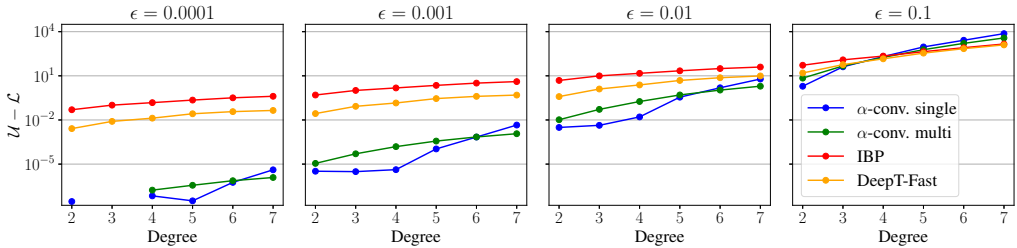}}
    \caption{\rebuttal{Average difference in log-scale between PGD upper bound ($\ub$) and lower bound ($\lb$) provided by BP (red), DeepT-Fast \citep{Bonaert2021Fast} (orange), $\alpha$-convexification with Uniform diagonal shift (blue) and $\alpha$-convexification with Non-uniform diagonal shift (green) of the first $1000$ images of the MNIST dataset. $\alpha$-convexification bounds are significantly tighter than IBP and DeepT-Fast for small $\epsilon$ values and all PN degrees from 2 to 7.}}
    \label{fig:IBP_alpha_bounds_comparison}
\end{figure}

\begin{table*}[!htb]
        \centering
        \fontsize{7}{8}\selectfont
        \begin{threeparttable}
               \caption{Verification results with our method employing IBP\rebuttal{, DeepT-Fast} and $\alpha$-convexification for lower bounding the objective. {\tt Acc.\%} is the clean accuracy of the network, {\tt Ver.\%} is the verified accuracy and {\tt U.B.} its upper bound. When using $\alpha$-convexification bounds we get verified accuracies really close to the upper bound, while when using IBP verified accuracy is $0$ for every network-$\epsilon$ pair, which makes it unsuitable for PN verification.
    }
   \label{tab:IBP_vs_alpha_conv_verif}
               \begin{tabular}{cccc|cc|cc|cc|cccccccccccccc}
               \toprule
        & & & & \multicolumn{2}{c|}{\multirow{2}{*}{IBP}} & \multicolumn{2}{c|}{\rebuttal{DeepT-Fast}} & \multicolumn{2}{c|}{VPN (ours)} &\\
        \multirow{2}{*}{Dataset} & \multirow{2}{*}{Model} & \multirow{2}{*}{Acc.\%}& \multirow{2}{*}{$\epsilon$}  &  &  & \multicolumn{2}{c|}{\rebuttal{\citep{Bonaert2021Fast}}} & \multicolumn{2}{c|}{($\alpha$-convexification)}\\ 
        &  &  &  &  Time(s) & Ver.\% & \rebuttal{Time(s)} & \rebuttal{Ver.\%} & Time(s) & Ver.\% & U.B.\\ 
        \midrule
        \multirow{3}{*}{MNIST} & \multirow{3}{*}{PN\_Conv4} & \multirow{3}{*}{$98.6$} & $0.015$ & $0.3$ & $0.0$ & \rebuttal{$2.3$} & \rebuttal{$91.3$} & $50$ & $\mathbf{96.3}$ & $96.4$ \\
        &  &  & $0.026$ & $0.4$ & $0.0$ & \rebuttal{$3.7$} & \rebuttal{$59.7$} & $69$ & $\mathbf{92.9}$ & $94.8$\\
        &  &  & \rebuttal{$0.3$} & \rebuttal{$0.6$} & \rebuttal{$0.0$} & \rebuttal{$0.7$} & \rebuttal{$0.0$} & \rebuttal{$13.8$} & \rebuttal{$0.0$} & \rebuttal{$0.0$}\\
        \midrule
        \multirow{4}{*}{CIFAR10}& \multirow{2}{*}{PN\_Conv2} & \multirow{2}{*}{$63.5$} & $1/255$ & $0.3$ & $0.0$ & \rebuttal{$2.0$} & \rebuttal{$23.3$} & $136.2$ & $\mathbf{44.4}$ & $44.6$ \\
        & & & $2/255$ & $0.5$ & $0.0$ & \rebuttal{$0.6$} & \rebuttal{$1.4$} & $89.2$ & $\mathbf{25.4}$ & $27.5$\\
        & \multirow{2}{*}{PN\_Conv4} & \multirow{2}{*}{$62.6$} & $1/255$ & $0.4$ & $0.0$ & \rebuttal{$2.2$} & \rebuttal{$19.5$} & $274.6$ & $\mathbf{45.5}$ & $46.7$ \\
        & & & $2/255$ & $0.5$ & $0.0$ & \rebuttal{$0.5$} & \rebuttal{$0.5$} & $224.1$ & $\mathbf{16.5}$ & $30.5$\\
        \midrule
        \multirow{1}{*}{STL10\tnote{*}} & \multirow{1}{*}{PN\_Conv4} & \multirow{1}{*}{$38.1$} & $1/255$ & $3.4$ & $0.0$ & \rebuttal{$26.0$} & \rebuttal{$14.7$} & $2481.0$ & $\mathbf{21.7}$ & $21.9$ \\
\bottomrule
    \end{tabular}
                \begin{tablenotes}
                        \footnotesize
                        \item * Results obtained in the first $360$ images of the dataset due to the longer running times because of the larger input size of STL10. 
                \end{tablenotes}
        \end{threeparttable}
        \vspace{-3mm}
\end{table*}

When using IBP, we get a much looser lower bound than with $\alpha$-convexification, see \cref{fig:IBP_alpha_bounds_comparison}. Only for high-degree, high-$\epsilon$ combinations IBP lower bounds are closer to the PGD upper bound. In practice, this is not a problem for verification, as for epsilons in the order of $0.1$, it is really easy to find adversarial examples with PGD and there will be no accuracy left to verify. \rebuttal{DeepT-Fast significantly outperforms IBP bounds across all degrees and $\epsilon$ values. But, as observed in \cref{fig:IBP_alpha_bounds_comparison}, except for big $\epsilon$ values, its performance is still far from the one provided by both $\alpha$-convexification methods. When comparing both $\alpha$-convexification methods (blue and green lines in \cref{fig:IBP_alpha_bounds_comparison}), we observe that for small degree PNs ($N < 5$), in the Uniform diagonal shift case we are able to obtain tighter bounds.}

The looseness of the IBP lower bound is confirmed when comparing the verified accuracy with IBP \rebuttal{and the rest of lower bounding methods}, see \cref{tab:IBP_vs_alpha_conv_verif}. With $\alpha$-convexification, we are able to verify the accuracy of $2^{\text{nd}}$ and $4^{\text{th}}$ order PNs almost exactly (almost no gap between the verified accuracy and its upper bound) in every studied dataset, while with \rebuttal{IBP}, we are not able to verify robustness for a single image in any network-$\epsilon$ pair, confirming the fact that IBP cannot be used for PN verification. \rebuttal{The improvements in the bounds when utilizing DeepT-Fast instead of IBP is clearly seen in verification results in \cref{tab:IBP_vs_alpha_conv_verif}. With DeepT-Fast, we are able to effectively verify PNs faster than with $\alpha$-convexification, but achieving a much lower verified accuracy ({\tt Ver\%}) than with $\alpha$-convexification. As a reference, for CIFAR10 PN\_Conv2 at $\epsilon = 2/255$, with $\alpha$-convexification we obtain $25.4\%$ verified accuracy, while with DeepT-Fast, we can just obtain $1.4\%$ verified accuracy. It is worth highlighting that DeepT-Fast can also scale to verify networks trained on STL10.}

\section{Conclusion}
\label{sec:conclusion}

We propose a novel $\alpha$-BaB global optimization algorithm to verify polynomial networks (PNs).
We exhibit that our method outperforms existing methods, such as black-box solvers, IBP and \rebuttal{DeepT-Fast \citep{Bonaert2021Fast}}.  %
Our method enables verification in datasets such as STL10, which includes RGB images of $96\times 96$ resolution. This is larger than the images typically used in previous verification methods \rebuttal{, we note that existing methods like IBP and DeepT-Fast are also able to scale to STL10 buth with a lower verified accuracy}. Our method can further encourage the community to extend verification to a broader class of functions as well as conduct experiments in datasets of higher resolution. 
We believe can be extended to cover other twice-differentiable networks in the future.

\textbf{Limitations:} As discussed in \cref{subsec:experiments_limitations}, our verification method does not scale to high-degree PNs. Even though we can verify high-accuracy PNs (see \cref{tab:IBP_vs_alpha_conv_verif}), we are still far from verifying the top performing deep PNs studied in \citet{Chrysos2021PolyNets}. 
Another problem that we share with ReLU NN verifiers is the scalability to networks with larger input size \citep{Wang2021Beta-CROWN}. In this work we are able to verify networks trained in STL10 \citep{STL10}, but these networks are shallow, yet their verification still takes a long time, see \cref{tab:IBP_vs_alpha_conv_verif}.

\textbf{Acknowledgements:} We are deeply thankful to the reviewers for providing constructive feedback.  
Research was sponsored by the Army Research Office and was accomplished under Grant Number W911NF-19-1-0404. This project has received funding from the European Research Council (ERC) under the European Union's Horizon 2020 research and innovation programme (grant agreement number 725594 - time-data). This work was supported by the Swiss National Science Foundation (SNSF) under  grant number 200021\_178865. This project has received funding from the European Research Council (ERC) under the European Union's Horizon 2020 research and innovation programme (grant agreement n° 725594 - time-data). This work was supported by Zeiss. This work was supported by SNF project – Deep Optimisation of the Swiss National Science Foundation (SNSF) under grant number 200021\_205011.

\clearpage

\bibliographystyle{plainnat}
\bibliography{main}
\clearpage

\appendix
\clearpage
\section*{Contents of the appendix}
\rebuttal{%
The societal impact of our method is discussed in \cref{sec:impact}. We cover the related work in \cref{sec:related}.} \cref{app:background} provides a more detailed coverage of PN architectures. In \cref{app:BaB_algorithm}, we include the pseudocode of our algorithms and we provide an analysis of the complexity. Experiments and ablation studies are available in \cref{app:Appendix}. %
To conclude, in \cref{app:proofs}, we include all of our proofs.
\section{Societal impact}
\label{sec:impact}
The performance on standard image classification benchmarks has increased substantially the last few years, owing to the success of neural networks. Their success enables their adoption in tackling real-world problems. However, robustness and trustworthiness of neural networks is of critical importance before their adoption in real-world applications. Our method is a verifier that focuses on polynomial networks and enables the complete verification of PNs. Therefore, we expect that by using the proposed method, certain properties of the robustness could be verified in a principled way. We expect this to have a predominantly positive societal impact as either a tool for pre-trained models or tool for certifying models as part of their debugging. However, it can also be used as a tool to find weaknesses of pretrained PNs by adversarial agents. 
\section{Related Work}
\label{sec:related}
In this Section, we give an overview of neural network verification and polynomial networks, that are centered around our target in this work.
\subsection{Neural Network Verification}
\label{subsec:rel_verif}

Early works on sound and complete NN verification were based on Mixed Integer Linear Programming (MILP) and Satisfiability Modulo Theory (SMT) solvers \citep{Reluplex2017, Ehlers2017, Bastani2016constraints, Tjeng2019} and were limited to both small datasets and networks.

The utilization of custom BaB algorithms enabled verification to scale to datasets and networks that are closer to those used in practice. \citet{Bunel2019BaB_first} review earlier methods like \citet{Reluplex2017} and show they can be formulated as BaB algorithms. BaDNB \citep{DePalmaBaB2021} proposes a novel branching strategy called Filtered Smart Branching and uses the Lagrangian decomposition-based bounding algorithm. $\beta$-CROWN \citep{Wang2021Beta-CROWN} proposes a bound propagation based algorithm. %
MN-BaB \citep{ferrari2022mnBaB} proposes a cost adjusted branching strategy and leverages multi-neuron relaxations and a GPU-based solver for bounds computing. Our work centers on the bounding algorithm by proposing a general convex lowerbound adapted to PNs.

BaB algorithms for ReLU networks focus their branching strategies on the activity of ReLU neurons. This has been observed to work better than input set branching for ReLU networks \citep{Bunel2019BaB_first}. Similarly to our method, \citet{Anderson2019}, \citet{Wang2018intervals}, \citet{Royo2019FastNN} use input set branching strategies.

\subsection{Polynomial Networks}
\label{subsec:rel_polynets}

First works have been focused on developing the foundations and showcasing the performance of PNs in different tasks \citep{Chrysos2021PolyNets,Chrysos2020NAPS}. Also, in \citet{Chrysos2022Class}, PN classifiers are formulated in a common framework where other previous methods like \citet{Wang2018NonLocal} can be framed. Lately, more emphasis has been put onto proving theoretical properties of PNs \citep{Fan2021Expressivity, choraria2022the}. In \citet{zhu2022LipschitzPN}, they derive Lipschitz constant and complexity bounds for two PN decompositions in terms of the $l_{\infty}$ and $l_{2}$ norms. They also analyze robustness of PNs against PGD adversarial attacks by measuring percentage of images where PGD fails to find an adversarial example, which is a complete but not sound verification method. Our verification method is sound and complete. %

\section{Background}
\label{app:background}

\subsection{Coupled CP decomposition (CCP)}
\label{app:CCP}

Relying on the CP decomposition \citep{Kolda2009Tensors}, the CCP decomposition provides us a core expression of PNs as used in \citet{Chrysos2020NAPS} to construct a generative model. Let $N$ be the polynomial degree, $\bm{z} \in \realnum^{d}$ the input vector, $d$, $k$ and $o$ the input, hidden and output sizes, the CCP decomposition can be expressed as:
\begin{equation}
    \bm{x}^{(n)} = (\bm{W}_{[n]}^{\top}\bm{z})*\bm{x}^{(n-1)} + \bm{x}^{(n-1)}\,, \forall n = 2, 3, \dots, N\,,
    \label{eq:CCP_recursion_annex}
\end{equation}
where $\bm{x}^{(1)} = \bm{W}_{[1]}^{\top}\bm{z}$, $\bm{f}(\bm{z}) = \bm{C}\bm{x}^{(N)}+ \bm{\beta}$ and $*$ denotes the Hadamard product. 

For example, the second order CCP factorization will lead to the following formulation:
\begin{equation}
    \begin{matrix}
        \bm{x}^{(1)} = \bm{W}_{[1]}^{\top}\bm{z}, &
        \bm{x}^{(2)} = \bm{W}_{[2]}^{\top}\bm{z} * \bm{x}^{(1)} + \bm{x}^{(1)}, &
        \bm{f}(\bm{z}) = \bm{C}\bm{x}^{(2)} + \bm{\beta}\,,
    \end{matrix}
    \label{eq:CCP_2nd_def}
\end{equation}
where $\bm{W}_{[1]} \in \realnum^{d \times k}$, $\bm{W}_{[2]} \in \realnum^{d \times k}$ and $\bm{C} \in \realnum^{o \times k}$ are weight matrices, $\bm{\beta} \in \realnum^{o}$ is a vector.

\subsection{Nested coupled CP decomposition (NCP)}
\label{app:NCP}

\begin{figure}[h]
    \centering
    \resizebox{\textwidth}{!}{
    \begin{tikzpicture}[
    greenbox/.style={rectangle, fill=green!40, minimum size=8mm},
    bluebox/.style={rectangle,fill=blue!30, minimum size=8mm},
    yellowbox/.style={rectangle,fill=yellow!60, minimum size=5mm},
    redbox/.style={rectangle,fill=red!40, minimum size=8mm},
    ]
    
    \node[greenbox] (z) {$\bm{z}$};
    
    \node[bluebox] (A1) [right = of z] {$\bm{W}_{[1]}$};
    \node[bluebox] (S2) [right = of A1] {$\bm{S}_{[2]}$};
    \node[yellowbox] (add1) [right = of S2] {$+$};
    \node[bluebox] (b2) [below = of add1] {$\bm{b}_{[2]}$};
    \node[yellowbox] (h2) [right = of add1] {$*$};
    \node[bluebox] (A2) [above = of h2] {$\bm{W}_{[2]}$};
    \node[bluebox] (S3) [right = of h2] {$\bm{S}_{[3]}$};
    \node[yellowbox] (add2) [right = of S3] {$+$};
    \node[bluebox] (b3) [below = of add2] {$\bm{b}_{[3]}$};
    \node[yellowbox] (h3) [right = of add2] {$*$};
    \node[bluebox] (A3) [above = of h3] {$\bm{W}_{[3]}$};

    \node[bluebox] (C) [right = of h3] {$\bm{C}$};
    
    \node[yellowbox] (add) [right = of C] {$+$};
    
    \node[bluebox] (b) [above = of add] {$\bm{\beta}$};
    
    \node[redbox] (output) [right = of add] {$\bm{f}(\bm{z})$};
    
    \draw[->] (z.east) to (A1.west);
    
    \draw[->] (z.north) -- (0,2.75) -| (A2.north);
    \draw[->] (z.north) -- (0,2.75) -| (A3.north);
    
    \draw[->] (A1.east) to node[above] {$\bm{x}^{(1)}$} (S2.west);
    \draw[->] (S2.east)  to  (add1.west);
    \draw[->] (A2.south)  to  (h2.north);
    \draw[->] (add1.east)  to  (h2.west);
    \draw[->] (b2.north)  to  (add1.south);
    \draw[->] (h2.east) to node[above] {$\bm{x}^{(2)}$} (S3.west);
    \draw[->] (S3.east)  to  (add2.west);
    \draw[->] (A3.south)  to  (h3.north);
    \draw[->] (b3.north)  to  (add2.south);
    \draw[->] (add2.east)  to  (h3.west);
    \draw[->] (h3.east) to node[above] {$\bm{x}^{(3)}$} (C.west);
    \draw[->] (C.east)  to  (add.west);
    \draw[->] (b.south)  to  (add.north);
    \draw[->] (add.east)  to  (output.west);
    
    \end{tikzpicture}}
    \caption{Third order NCP network architecture.}
    \label{fig:NCP}
\end{figure}
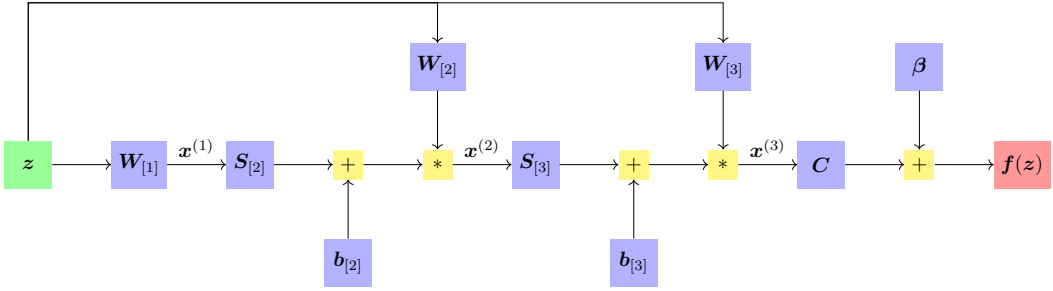

The NCP model leverages a joint hierarchical decomposition,
which provided strong results in both generative and discriminative tasks in \citet{Chrysos2021PolyNets}. It can be expressed with the following recursive relation:
\begin{equation}
    \bm{x}^{(n)} = (\bm{W}_{[n]}^{\top}\bm{z})*(\bm{S}_{[n]}^{\top}\bm{x}^{(n-1)} + \bm{b}_{[n]})\,,
    \label{eq:NCP_recursion}
\end{equation}
for $n \in [N-1] + 1$ with $\bm{x}^{(1)} = \bm{W}_{[1]}^{\top}\bm{z}$ and $\bm{f}(\bm{z}) = \bm{C}\bm{x}^{(N)}+ \bm{\beta}$.

We present the first and second order partial derivatives of NCP below.
\begin{equation}
    \nabla_{\bm{z}} x_{i}^{(n)} = \bm{w}_{[n]:i}\cdot (\bm{s}_{[n]:i}^{\top}\bm{x}^{(n-1)} + b_{[n]i}) + (\bm{w}_{[n]:i}^{\top}\bm{z})\cdot (\sum_{j = 1}^{k}s_{[n]ji}\nabla_{\bm{z}} x_{j}^{(n-1)})
\label{eq:grad_NCP_recursive}
\end{equation}
\begin{equation}
\begin{aligned}
    \nabla_{\bm{z} \bm{z}}^{2} x_{i}^{(n)} & = \bm{w}_{[n]:i}(\sum_{j = 1}^{k}s_{[n]ji}\nabla_{\bm{z}} x_{j}^{(n-1)})^{\top} + (\sum_{j = 1}^{k}s_{[n]ji}\nabla_{\bm{z}} x_{j}^{(n-1)})\bm{w}_{[n]:i}^{\top} \\
    & \quad +(\bm{w}_{[n]:i}^{\top}z)\cdot(\sum_{j = 1}^{k}s_{[n]ji} \nabla_{\bm{z} \bm{z}}^{2} x_{j}^{(n-1)})\,.\\
\end{aligned}
\label{eq:hessian_NCP_recursive}
\end{equation}
Our method can be easily extended to this type of PNs. As a reference, we obtain a verified accuracy of $76.2\%$ with an upper bound of $76.4\%$ with an $2 \times 25$ NCP at $\epsilon = 0.026$.

\subsection{Product of Polynomials}
\label{app:Prod_Poly}

In practice, \citet{Chrysos2021PolyNets} report that to reduce further the parameters they are often stacking sequentially a number of polynomials, see \cref{fig:Prod_poly}. That results in a setting that is referred to as product of polynomials, with the highest degree of expansion being defined as the product of all the degrees of the individual polynomials. Hopefully, as we will demonstrate below, our formulation can be extended to this setting.

Let $\bm{x} = \bm{x}^{(N_1)}$ be the output of a PN $f_{1} : \realnum^{d} \to \realnum^{k}$ of $N_{1}^{\text{th}}$-degree and $\bm{y} = \bm{y}^{(N_2)}$ be the output of a PN $f_{2} : \realnum^{k} \to \realnum^{k}$ of $N_{2}^{\text{th}}$-degree, with $\bm{x}$ and $\bm{y}$ coming either from \cref{eq:CCP_recursion_annex} or \cref{eq:NCP_recursion}. The product of polynomials $f_{i}$ and $f_{2}$ is defined as $f : \realnum^{d} \to \realnum^{o}$:

\begin{equation}
    \bm{f}(\bm{z}) = \bm{C}\bm{f}_{2}(\bm{f}_{1}(\bm{z})) + \bm{\beta}\,,
\end{equation}
where $\bm{C} \in \realnum^{o \times k}$ and $\bm{\beta} \in \realnum^{o}$.

We present the first and second order partial derivatives of the Product of polynomials below.

Let $\bm{z}$ be the input, $\bm{x} = \bm{f}_1(\bm{z})$ and $\bm{y} = \bm{f}_2(\bm{x})$ by applying the chain rule of partial derivatives, we can obtain: 
\begin{equation}
    \nabla_{\bm{z}} y_{i} = \sum_{j = i}^{k} \frac{\partial y_{i}}{\partial x_{j}}\nabla_{\bm{z}} x_{j},~~ \nabla_{\bm{z} \bm{z}}^{2} y_{i} = \sum_{j = i}^{k} \frac{\partial y_{i}}{\partial x_{j}}\nabla_{\bm{z} \bm{z}}^{2} x_{j} + \bm{J}_{\bm{z}}^{\top}(\bm{x})\nabla_{\bm{x} \bm{x}}^{2} y_{i} \bm{J}_{\bm{z}}(\bm{x})\,,
    \label{eq:grad_hess_prod_poly}
\end{equation}
where $\bm{J}_{\bm{z}}^{\top}(\bm{x}) \in \realnum^{k \times d}$ is the Jacobian matrix.

\begin{figure}[h]
    \centering
    \resizebox{\textwidth}{!}{
    \begin{tikzpicture}[
    greenbox/.style={rectangle, fill=green!40, minimum size=8mm},
    bluebox/.style={rectangle,fill=blue!30, minimum size=8mm},
    yellowbox/.style={rectangle,fill=yellow!60, minimum size=5mm},
    redbox/.style={rectangle,fill=red!40, minimum size=8mm},
    ]
    
    \node[greenbox] (z) {$\bm{z}$};
    
    \node[bluebox] (U1) [right = of z] {$\bm{W}_{x[1]}$};
    \node[yellowbox] (h1) [right = of U1] {$*$};
    \node[yellowbox] (add1) [right = of h1] {$+$};
    \node[bluebox] (U2) [above = of h1] {$\bm{W}_{x[2]}$};

    \node[bluebox] (U3) [right = of add1] {$\bm{W}_{y[1]}$};
    \node[yellowbox] (h2) [right = of U3] {$*$};
    \node[yellowbox] (add2) [right = of h2] {$+$};
    
    \node[bluebox] (U4) [above = of h2] {$\bm{W}_{y[2]}$};
    
    \node[bluebox] (C) [right = of add2] {$\bm{C}$};
    
    \node[yellowbox] (add) [right = of C] {$+$};
    
    \node[bluebox] (b) [above = of add] {$\bm{\beta}$};
    
    \node[redbox] (output) [right = of add] {$\bm{f}(\bm{z})$};
    
    \draw[->] (z.east) to (U1.west);
    
    \draw[->] (z.north) -- (0,2.75) -| (U2.north);
    \draw[->] (U1.east) -- (2.75,0) -| (2.75,-0.75) -| (add1.south);
    \draw[->] (add1.east) -- (5.75,0) -| (5.75,2.75) -| (U4.north);
    \draw[->] (U3.east) -- (8,0) -| (8,-0.75) -| (add2.south);
    
    \draw[->] (U1.east) to node[above] {$\bm{x}^{(1)}$} (h1.west) ;
    \draw[->] (h1.east)  to  (add1.west);
    \draw[->] (U2.south)  to  (h1.north);
    \draw[->] (add1.east) to node[below] {$\bm{x}^{(2)}$} (U3.west);
    \draw[->] (h2.east)  to  (add2.west);
    \draw[->] (U3.east)  to node[above] {$\bm{y}^{(1)}$} (h2.west);
    \draw[->] (U4.south)  to  (h2.north);
    \draw[->] (add2.east) to node[above] {$\bm{y}^{(2)}$} (C.west);
    \draw[->] (C.east)  to  (add.west);
    \draw[->] (b.south)  to  (add.north);
    \draw[->] (add.east)  to  (output.west);
    
    \end{tikzpicture}}
    \caption{Product of two $2^{\text{nd}}$-degree CCP polynomials. The final classification layer of the first polynomial is dropped, the output $\bm{x}^{(2)}$ is fed into a second $2^{\text{nd}}$-degree polynomial.}
    \label{fig:Prod_poly}
\end{figure}
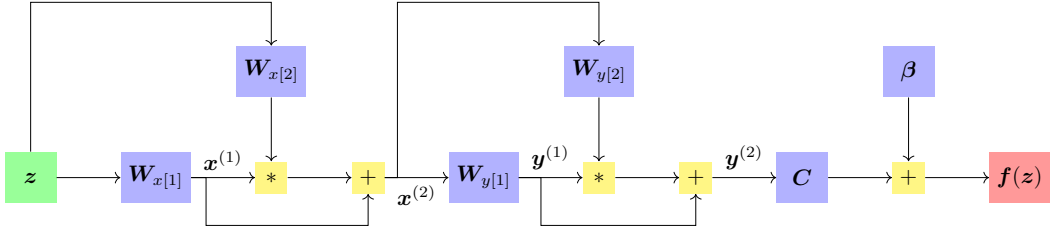

\subsection{Convolutional PNs (C-PNs)}
\label{subsec:conv_PNs}

As seen in \citet{Chrysos2021PolyNets}, the performance of PNs is boosted when employing convolution operators instead of standard linear mappings. Convolutions reduce the number of parameters and take advantage of the local 2D structure of the images. Let $\bm{\mathcal{Z}} \in \realnum^{c \times h \times w}$ be the input image with $c$, $h$ and $w$ being the number of channels, height and width respectively. The $N^{\text{th}}$-degree \textbf{CCP\_Conv} becomes:
\begin{equation}
    \bm{X}^{(n,i)} = (\bm{\mathcal{Z}} \circ \bm{\mathcal{W}}_{[n,i]})*\bm{X}^{(n-1,i)} + \bm{X}^{(n-1,i)}\,, \forall n \in [N-1] + 1, \forall i \in [q]\,,
    \label{eq:CCP_Conv_recursion_annex}
\end{equation}
where $\bm{X}^{(1,i)} = \bm{\mathcal{Z}} \circ \bm{\mathcal{W}}_{[1,i]}, \forall i \in [q]$, $\bm{f}(\bm{z}) = \bm{C}\bm{x}^{(N)}+ \bm{\beta}$ and $\bm{x}^{(N)} = \text{flat}(\bm{\mathcal{X}}^{(N)})$.

In order to verify C-PNs, we convert every convolutional layer into a linear layer via Toeplitz matrices (see \citet{AI2_2018}) resulting in an equivalent CCP PN.

\subsection{Details on IBP}
\label{subsec:IBP_details}

In this section we further elaborate on the IBP for the network outputs, the gradients and Hessians by applying the notions in \cref{subsec:IBP_PN}.

\textbf{IBP for network outputs}
In order to compute lower and upper bounds on the output of the network $\bm{f}(\bm{z})$, we apply interval arithmetic techniques in the recursive formulas \cref{eq:CCP_recursion,eq:NCP_recursion}. For both the CCP and the NCP cases, let $\hat{\bm{x}}^{(n)} = \bm{W}_{[n]}^{\top}\bm{z}$:
\begin{equation}
    \begin{aligned}
        \lb(\hat{\bm{x}}^{(n)}) = {\bm{W}_{[n]}^{\top}}^{+}\bm{l} + {\bm{W}_{[n]}^{\top}}^{-}\bm{u}\\
        \ub(\hat{\bm{x}}^{(n)}) = {\bm{W}_{[n]}^{\top}}^{+}\bm{u} + {\bm{W}_{[n]}^{\top}}^{-}\bm{l}\,.
    \end{aligned}
    \label{eq:bounds_xnhat}
\end{equation}
In the case of CCP PNs, the recursive formula in \cref{eq:CCP_recursion} becomes $\bm{x}^{(n)} = (\hat{\bm{x}}^{(n)} + \bm{1})*\bm{x}^{(n-1)}$, whereas for NCP, \cref{eq:NCP_recursion} becomes $\bm{x}^{(n)} = (\hat{\bm{x}}^{(n)})*(\bm{S}_{[n]}^{\top}\bm{x}^{(n-1)} + \bm{b}_{[n]})$ for any $n \in [N-1] + 1$, for $n=1$, in both architectures $\bm{x}^{(n)} = \hat{\bm{x}}^{(n)}$. Then, we can define the bounds with the following recursive formulas:
\begin{equation}
    \begin{aligned}
        \textbf{CCP} & \left\{\begin{aligned}
            S = \left\{\begin{aligned}
                (\lb(\hat{\bm{x}}^{(n)}) + 1)*\lb(\bm{x}^{(n-1)}),\\
                (\lb(\hat{\bm{x}}^{(n)}) + 1)*\ub(\bm{x}^{(n-1)}),\\
                (\ub(\hat{\bm{x}}^{(n)}) + 1)*\lb(\bm{x}^{(n-1)}),\\
                (\ub(\hat{\bm{x}}^{(n)}) + 1)*\ub(\bm{x}^{(n-1)}),\\
            \end{aligned}\right\}\\
            \lb(\bm{x}^{(n)}) = \min(S)\\
            \ub(\bm{x}^{(n)}) = \max(S)\\
        \end{aligned}\right.\\
        \textbf{NCP} & \left\{\begin{aligned}
            \lb(\bm{S}_{[n]}^{\top}\bm{x}^{(n-1)} + \bm{b}_{[n]}) = {\bm{S}_{[n]}^{\top}}^{+}\lb(\bm{x}^{(n-1)}) + {\bm{S}_{[n]}^{\top}}^{-}\ub(\bm{x}^{(n-1)}) + \bm{b}_{[n]}\\
            \ub(\bm{S}_{[n]}^{\top}\bm{x}^{(n-1)} + \bm{b}_{[n]}) = {\bm{S}_{[n]}^{\top}}^{+}\ub(\bm{x}^{(n-1)}) + {\bm{S}_{[n]}^{\top}}^{-}\lb(\bm{x}^{(n-1)}) + \bm{b}_{[n]}\\
            S = \left\{\begin{aligned}
                \lb(\bm{S}_{[n]}^{\top}\bm{x}^{(n-1)} + \bm{b}_{[n]})*\lb(\bm{x}^{(n-1)}),\\
                \lb(\bm{S}_{[n]}^{\top}\bm{x}^{(n-1)} + \bm{b}_{[n]})*\ub(\bm{x}^{(n-1)}),\\
                \ub(\bm{S}_{[n]}^{\top}\bm{x}^{(n-1)} + \bm{b}_{[n]})*\lb(\bm{x}^{(n-1)}),\\
                \ub(\bm{S}_{[n]}^{\top}\bm{x}^{(n-1)} + \bm{b}_{[n]})*\ub(\bm{x}^{(n-1)}),\\
            \end{aligned}\right\}\\
            \lb(\bm{x}^{(n)}) = \min(S)\\
            \ub(\bm{x}^{(n)}) = \max(S)\,,\\
        \end{aligned}\right.\\
    \end{aligned}
    \label{eq:bounds_xn}
\end{equation}
where the $\min$ and $\max$ operators are applied element-wise in sets of vectors or matrices. Finally, the output bounds are obtained with:
\begin{equation}
    \begin{aligned}
        \lb(\bm{f}(\bm{z})) = \bm{C}^{+}\lb(\bm{x}^{(N)}) + \bm{C}^{-}\ub(\bm{x}^{(N)}) + \bm{\beta}\\
        \ub(\bm{f}(\bm{z})) = \bm{C}^{+}\ub(\bm{x}^{(N)}) + \bm{C}^{-}\lb(\bm{x}^{(N)}) + \bm{\beta}\,.
    \end{aligned}
    \label{eq:bounds_output}
\end{equation}
These operations can be implemented as a forward pass through the PN.

\textbf{IBP for gradients}

The next step is to obtain bounds on the gradients of the PNs. Again, %
with the help of interval arithmetic theory, we can extend the recursive formulas \cref{eq:grad_CCP_recursive,eq:grad_NCP_recursive} for computing IBP bounds. In the case of CCP PNs:
\begin{equation}
    \begin{aligned}
        \lb(\bm{J}_{\bm{z}}^{\top}(\bm{x}^{(n)})) & = \lb(\bm{W}_{[n]}*\bm{x}^{(n-1)} + (\bm{W}_{[n]}^{\top}\bm{z} + 1)*\bm{J}_{\bm{z}}^{\top}(\bm{x}^{(n-1)})))\\ 
        & = \lb(\bm{W}_{[n]}*\bm{x}^{(n-1)}) + \lb((\bm{W}_{[n]}^{\top}\bm{z} + 1)*\bm{J}_{\bm{z}}^{\top}(\bm{x}^{(n-1)}))\\ & = \bm{W}_{[n]}^{+}*\lb(\bm{x}^{(n-1)}) + \bm{W}_{[n]}^{-}*\ub(\bm{x}^{(n-1)}) + \lb((\bm{W}_{[n]}^{\top}\bm{z} + 1)*\bm{J}_{\bm{z}}^{\top}(\bm{x}^{(n-1)}))\\
        \ub(\bm{J}_{\bm{z}}^{\top}(\bm{x}^{(n)})) & = \ub(\bm{W}_{[n]}*\bm{x}^{(n-1)} + (\bm{W}_{[n]}^{\top}\bm{z} + 1)*\bm{J}_{\bm{z}}^{\top}(\bm{x}^{(n-1)})))\\ 
        & = \ub(\bm{W}_{[n]}*\bm{x}^{(n-1)}) + \ub((\bm{W}_{[n]}^{\top}\bm{z} + 1)*\bm{J}_{\bm{z}}^{\top}(\bm{x}^{(n-1)}))\\ & = \bm{W}_{[n]}^{+}*\ub(\bm{x}^{(n-1)}) + \bm{W}_{[n]}^{-}*\lb(\bm{x}^{(n-1)}) + \ub((\bm{W}_{[n]}^{\top}\bm{z} + 1)*\bm{J}_{\bm{z}}^{\top}(\bm{x}^{(n-1)}))\,,\\
    \end{aligned}
    \label{eq:bounds_grad_CCP}
\end{equation}
with
\begin{equation}
    \begin{aligned}
        S = \left\{\begin{aligned}
            \lb(\bm{W}_{[n]}^{\top}\bm{z} + 1)*\lb(\bm{J}_{\bm{z}}^{\top}(\bm{x}^{(n-1)})),\\
            \lb(\bm{W}_{[n]}^{\top}\bm{z} + 1)*\ub(\bm{J}_{\bm{z}}^{\top}(\bm{x}^{(n-1)})),\\
            \ub(\bm{W}_{[n]}^{\top}\bm{z} + 1)*\lb(\bm{J}_{\bm{z}}^{\top}(\bm{x}^{(n-1)})),\\
            \ub(\bm{W}_{[n]}^{\top}\bm{z} + 1)*\ub(\bm{J}_{\bm{z}}^{\top}(\bm{x}^{(n-1)}))\\
        \end{aligned}\right\}\\
        \lb((\bm{W}_{[n]}^{\top}\bm{z} + 1)*\bm{J}_{\bm{z}}^{\top}(\bm{x}^{(n-1)})) = \min(S)\\
        \ub((\bm{W}_{[n]}^{\top}\bm{z} + 1)*\bm{J}_{\bm{z}}^{\top}(\bm{x}^{(n-1)})) = \max(S)\,,\\
    \end{aligned}
    \label{eq:bounds_grad_CCP_right_part}
\end{equation}
where the Hadamard product of a $\realnum^{k \times d}$ matrix with a $\realnum^{d}$ vector results in a $\realnum^{k \times d}$ matrix. The $\min$ and $\max$ operators are applied element-wise in sets of vectors or matrices.

In the case of NCP PNs:
\begin{equation}
    \begin{aligned}
        \lb(\bm{J}_{\bm{z}}^{\top}(\bm{x}^{(n)})) & = \lb(\bm{W}_{[n]}*(\bm{S}_{[n]}\bm{x}^{(n-1)} + \bm{b}_{[n]}) + (\bm{W}_{[n]}^{\top}\bm{z})*(\bm{S}_{[n]}\bm{J}_{\bm{z}}^{\top}(\bm{x}^{(n-1)})))\\ 
        & = \lb(\bm{W}_{[n]}*(\bm{S}_{[n]}\bm{x}^{(n-1)} + \bm{b}_{[n]})) + \lb((\bm{W}_{[n]}^{\top}\bm{z})*(\bm{S}_{[n]}\bm{J}_{\bm{z}}^{\top}(\bm{x}^{(n-1)})))\\
        & = \bm{W}_{[n]}^{+}*\lb(\bm{S}_{[n]}\bm{x}^{(n-1)} + \bm{b}_{[n]})) + \bm{W}_{[n]}^{-}*\ub(\bm{S}_{[n]}\bm{x}^{(n-1)} + \bm{b}_{[n]})) \\
        & + \lb((\bm{W}_{[n]}^{\top}\bm{z})*(\bm{S}_{[n]}\bm{J}_{\bm{z}}^{\top}(\bm{x}^{(n-1)})))\\
        \ub(\bm{J}_{\bm{z}}^{\top}(\bm{x}^{(n)})) & = \ub(\bm{W}_{[n]}*(\bm{S}_{[n]}\bm{x}^{(n-1)} + \bm{b}_{[n]}) + (\bm{W}_{[n]}^{\top}\bm{z})*(\bm{S}_{[n]}\bm{J}_{\bm{z}}^{\top}(\bm{x}^{(n-1)})))\\ 
        & = \ub(\bm{W}_{[n]}*(\bm{S}_{[n]}\bm{x}^{(n-1)} + \bm{b}_{[n]})) + \ub((\bm{W}_{[n]}^{\top}\bm{z})*(\bm{S}_{[n]}\bm{J}_{\bm{z}}^{\top}(\bm{x}^{(n-1)})))\\
        & = \bm{W}_{[n]}^{+}*\ub(\bm{S}_{[n]}\bm{x}^{(n-1)} + \bm{b}_{[n]})) + \bm{W}_{[n]}^{-}*\lb(\bm{S}_{[n]}\bm{x}^{(n-1)} + \bm{b}_{[n]})) \\
        & + \ub((\bm{W}_{[n]}^{\top}\bm{z})*(\bm{S}_{[n]}\bm{J}_{\bm{z}}^{\top}(\bm{x}^{(n-1)})))\,,\\
    \end{aligned}
    \label{eq:bounds_grad_NCP}
\end{equation}
with
\begin{equation}
    \begin{aligned}
        \lb(\bm{S}_{[n]}\bm{J}_{\bm{z}}^{\top}(\bm{x}^{(n-1)})) = \bm{S}_{[n]}^{+}\lb(\bm{J}_{\bm{z}}^{\top}(\bm{x}^{(n-1)})) + \bm{S}_{[n]}^{-}\ub(\bm{J}_{\bm{z}}^{\top}(\bm{x}^{(n-1)})) \\
        \ub(\bm{S}_{[n]}\bm{J}_{\bm{z}}^{\top}(\bm{x}^{(n-1)})) = \bm{S}_{[n]}^{+}\ub(\bm{J}_{\bm{z}}^{\top}(\bm{x}^{(n-1)})) + \bm{S}_{[n]}^{-}\lb(\bm{J}_{\bm{z}}^{\top}(\bm{x}^{(n-1)})) \\
        S = \left\{\begin{aligned}
            \lb(\bm{W}_{[n]}^{\top}\bm{z})*\lb(\bm{S}_{[n]}\bm{J}_{\bm{z}}^{\top}(\bm{x}^{(n-1)})),\\
            \lb(\bm{W}_{[n]}^{\top}\bm{z})*\ub(\bm{S}_{[n]}\bm{J}_{\bm{z}}^{\top}(\bm{x}^{(n-1)})),\\
            \ub(\bm{W}_{[n]}^{\top}\bm{z})*\lb(\bm{S}_{[n]}\bm{J}_{\bm{z}}^{\top}(\bm{x}^{(n-1)})),\\
            \ub(\bm{W}_{[n]}^{\top}\bm{z})*\ub(\bm{S}_{[n]}\bm{J}_{\bm{z}}^{\top}(\bm{x}^{(n-1)}))\\
        \end{aligned}\right\}\\
        \lb((\bm{W}_{[n]}^{\top}\bm{z})*(\bm{S}_{[n]}\bm{J}_{\bm{z}}^{\top}(\bm{x}^{(n-1)}))) = \min(S)\\
        \ub((\bm{W}_{[n]}^{\top}\bm{z})*(\bm{S}_{[n]}\bm{J}_{\bm{z}}^{\top}(\bm{x}^{(n-1)}))) = \max(S)\,.\\
    \end{aligned}
    \label{eq:bounds_grad_NCP_right_part}
\end{equation}
\cref{eq:bounds_grad_CCP} (\cref{eq:bounds_grad_NCP}) jointly with \cref{eq:bounds_grad_CCP_right_part} (\cref{eq:bounds_grad_NCP_right_part}) allows to recursively obtain gradient bounds for the CCP (NCP) PNs starting with $\lb(\bm{J}_{\bm{z}}^{\top}(\bm{x}^{(1)})) = \ub(\bm{J}_{\bm{z}}^{\top}(\bm{x}^{(1)})) = \bm{W}_{[1]}$. Finally, bounds of the gradients of the network output in both CCP and NCP cases are:
\begin{equation}
    \begin{aligned}
        \lb(\bm{J}_{\bm{z}}(\bm{f})) = \bm{C}^{+}\lb(\bm{J}_{\bm{z}}(\bm{x}^{(N)})) + \bm{C}^{-}\ub(\bm{J}_{\bm{z}}(\bm{x}^{(N)})) + \bm{\beta}\\
        \ub(\bm{J}_{\bm{z}}(\bm{f})) = \bm{C}^{+}\ub(\bm{J}_{\bm{z}}(\bm{x}^{(N)})) + \bm{C}^{-}\lb(\bm{J}_{\bm{z}}(\bm{x}^{(N)})) + \bm{\beta}\,.
    \end{aligned}
    \label{eq:bounds_grad_output}
\end{equation}
\textbf{IBP for Hessians}
Lastly, with help of interval arithmetic and aforementioned IBP for PN outputs and gradients, we can use recursive formulas \cref{eq:hess_CCP_recursive,eq:hessian_NCP_recursive} to compute bounds of the Hessian matrices.

In the case of CCP PNs:
\begin{equation}
    \begin{aligned}
        \lb(\bm{H}_{\bm{z}}(x_{i}^{(n)})) & = \lb(\nabla_{\bm{z}}x_{i}^{(n-1)} \bm{w}_{[n]:i}^{\top} + \{\nabla_{\bm{z}}x_{i}^{(n-1)} \bm{w}_{[n]:i}^{\top}\}^\top + (\bm{w}_{[n]:i}^{\top}\bm{z} + 1) \bm{H}_{\bm{z}}(x_{i}^{(n-1)}))\\
        & = \lb(\nabla_{\bm{z}}x_{i}^{(n-1)}) {\bm{w}_{[n]:i}^{+}}^{\top} + \ub(\nabla_{\bm{z}}x_{i}^{(n-1)}) {\bm{w}_{[n]:i}^{-}}^{\top} \\
        & + \bm{w}_{[n]:i}^{+}{\lb(\nabla_{\bm{z}}x_{i}^{(n-1)})}^{\top} + \bm{w}_{[n]:i}^{-}{\ub(\nabla_{\bm{z}}x_{i}^{(n-1)})}^{\top} \\
        & +  \lb((\bm{w}_{[n]:i}^{\top}\bm{z} + 1) \bm{H}_{\bm{z}}(x_{i}^{(n-1)}))\\
        \ub(\bm{H}_{\bm{z}}(x_{i}^{(n)})) & = \ub(\nabla_{\bm{z}}x_{i}^{(n-1)} \bm{w}_{[n]:i}^{\top} + \{\nabla_{\bm{z}}x_{i}^{(n-1)} \bm{w}_{[n]:i}^{\top}\}^\top + (\bm{w}_{[n]:i}^{\top}\bm{z} + 1) \bm{H}_{\bm{z}}(x_{i}^{(n-1)}))\\
        & = \ub(\nabla_{\bm{z}}x_{i}^{(n-1)}) {\bm{w}_{[n]:i}^{+}}^{\top} + \lb(\nabla_{\bm{z}}x_{i}^{(n-1)}) {\bm{w}_{[n]:i}^{-}}^{\top} \\
        & + \bm{w}_{[n]:i}^{+}{\ub(\nabla_{\bm{z}}x_{i}^{(n-1)})}^{\top} + \bm{w}_{[n]:i}^{-}{\lb(\nabla_{\bm{z}}x_{i}^{(n-1)})}^{\top} \\
        & +  \ub((\bm{w}_{[n]:i}^{\top}\bm{z} + 1) \bm{H}_{\bm{z}}(x_{i}^{(n-1)}))\,,\\
    \end{aligned}
    \label{eq:bounds_hess_CCP_recursive}
\end{equation}
where
\begin{equation}
    \begin{aligned}
        S = \left\{\begin{aligned}
            \lb(\bm{w}_{[n]:i}^{\top}\bm{z} + 1) \lb(\bm{H}_{\bm{z}}(x_{i}^{(n-1)})),\\
            \lb(\bm{w}_{[n]:i}^{\top}\bm{z} + 1) \ub(\bm{H}_{\bm{z}}(x_{i}^{(n-1)})),\\
            \ub(\bm{w}_{[n]:i}^{\top}\bm{z} + 1) \lb(\bm{H}_{\bm{z}}(x_{i}^{(n-1)})),\\
            \ub(\bm{w}_{[n]:i}^{\top}\bm{z} + 1) \ub(\bm{H}_{\bm{z}}(x_{i}^{(n-1)}))\\
        \end{aligned}\right\}\\
        \lb((\bm{w}_{[n]:i}^{\top}\bm{z} + 1) \bm{H}_{\bm{z}}(x_{i}^{(n-1)})) = \min(S)\\
        \ub((\bm{w}_{[n]:i}^{\top}\bm{z} + 1) \bm{H}_{\bm{z}}(x_{i}^{(n-1)})) = \max(S)\,.\\
    \end{aligned}
    \label{eq:bounds_hess_CCP_right_part}
\end{equation}
In the case of NCP PNs:
\begin{equation}
    \begin{aligned}
        \lb(\bm{H}_{\bm{z}}(x_{i}^{(n)})) & = \lb(\bm{s}_{[n]i}\bm{J}_{\bm{z}}^{\top}(\bm{x}^{(n-1)}) \bm{w}_{[n]:i}^{\top} + \{\bm{s}_{[n]i}\bm{J}_{\bm{z}}^{\top}(\bm{x}^{(n-1)}) \bm{w}_{[n]:i}^{\top}\}^\top + (\bm{w}_{[n]:i}^{\top}\bm{z})\sum_{j=1}^{k}s_{ij}\bm{H}_{\bm{z}}(x_{j}^{(n-1)}))\\
        & = \lb(\bm{s}_{[n]i}\bm{J}_{\bm{z}}^{\top}(\bm{x}^{(n-1)})) {\bm{w}_{[n]:i}^{+}}^{\top} + \ub(\bm{s}_{[n]i}\bm{J}_{\bm{z}}^{\top}(\bm{x}^{(n-1)})) {\bm{w}_{[n]:i}^{l}}^{\top} \\
        & + \bm{w}_{[n]:i}^{+}{\lb(\bm{s}_{[n]i}\bm{J}_{\bm{z}}^{\top}(\bm{x}^{(n-1)}))}^{\top} + \bm{w}_{[n]:i}^{-}{\ub(\bm{s}_{[n]i}\bm{J}_{\bm{z}}^{\top}(\bm{x}^{(n-1)}))}^{\top} \\
        & +  \lb((\bm{w}_{[n]:i}^{\top}\bm{z})\sum_{j=1}^{k}s_{ij}\bm{H}_{\bm{z}}(x_{j}^{(n-1)}))\\
        \ub(\bm{H}_{\bm{z}}(x_{i}^{(n)})) & = \ub(\bm{s}_{[n]i}\bm{J}_{\bm{z}}^{\top}(\bm{x}^{(n-1)}) \bm{w}_{[n]:i}^{\top} + \{\bm{s}_{[n]i}\bm{J}_{\bm{z}}^{\top}(\bm{x}^{(n-1)}) \bm{w}_{[n]:i}^{\top}\}^\top + (\bm{w}_{[n]:i}^{\top}\bm{z})\sum_{j=1}^{k}s_{ij}\bm{H}_{\bm{z}}(x_{j}^{(n-1)}))\\
        & = \ub(\bm{s}_{[n]i}\bm{J}_{\bm{z}}^{\top}(\bm{x}^{(n-1)})) {\bm{w}_{[n]:i}^{+}}^{\top} + \lb(\bm{s}_{[n]i}\bm{J}_{\bm{z}}^{\top}(\bm{x}^{(n-1)})) {\bm{w}_{[n]:i}^{l}}^{\top} \\
        & + \bm{w}_{[n]:i}^{+}{\ub(\bm{s}_{[n]i}\bm{J}_{\bm{z}}^{\top}(\bm{x}^{(n-1)}))}^{\top} + \bm{w}_{[n]:i}^{-}{\lb(\bm{s}_{[n]i}\bm{J}_{\bm{z}}^{\top}(\bm{x}^{(n-1)}))}^{\top} \\
        & +  \ub((\bm{w}_{[n]:i}^{\top}\bm{z})\sum_{j=1}^{k}s_{ij}\bm{H}_{\bm{z}}(x_{j}^{(n-1)}))\,,\\
    \end{aligned}
    \label{eq:bounds_hess_NCP_recursive}
\end{equation}
where
\begin{equation}
    \begin{aligned}
        \lb(\bm{s}_{[n]i}\bm{J}_{\bm{z}}^{\top}(\bm{x}^{(n-1)})) &= \bm{s}_{[n]i}^{+}\lb(\bm{J}_{\bm{z}}^{\top}(\bm{x}^{(n-1)})) + \bm{s}_{[n]i}^{-}\ub(\bm{J}_{\bm{z}}^{\top}(\bm{x}^{(n-1)}))\\
        \ub(\bm{s}_{[n]i}\bm{J}_{\bm{z}}^{\top}(\bm{x}^{(n-1)})) &= \bm{s}_{[n]i}^{+}\ub(\bm{J}_{\bm{z}}^{\top}(\bm{x}^{(n-1)})) + \bm{s}_{[n]i}^{-}\lb(\bm{J}_{\bm{z}}^{\top}(\bm{x}^{(n-1)}))\\
        \lb(\sum_{j=1}^{k}s_{ij}\bm{H}_{\bm{z}}(x_{j}^{(n-1)})) &= \sum_{j=1}^{k}s_{ij}^{+}\lb(\bm{H}_{\bm{z}}(x_{j}^{(n-1)})) + s_{ij}^{-}\ub(\bm{H}_{\bm{z}}(x_{j}^{(n-1)}))\\
        \ub(\sum_{j=1}^{k}s_{ij}\bm{H}_{\bm{z}}(x_{j}^{(n-1)})) &= \sum_{j=1}^{k}s_{ij}^{+}\ub(\bm{H}_{\bm{z}}(x_{j}^{(n-1)})) + s_{ij}^{-}\lb(\bm{H}_{\bm{z}}(x_{j}^{(n-1)}))\,,\\
    \end{aligned}
\end{equation}
and
\begin{equation}
    \begin{aligned}
        S = \left\{\begin{aligned}
            \lb(\bm{w}_{[n]:i}^{\top}\bm{z}) \lb(\sum_{j=1}^{k}s_{ij}\bm{H}_{\bm{z}}(x_{j}^{(n-1)})),\\
            \lb(\bm{w}_{[n]:i}^{\top}\bm{z}) \ub(\sum_{j=1}^{k}s_{ij}\bm{H}_{\bm{z}}(x_{j}^{(n-1)})),\\
            \ub(\bm{w}_{[n]:i}^{\top}\bm{z}) \lb(\sum_{j=1}^{k}s_{ij}\bm{H}_{\bm{z}}(x_{j}^{(n-1)})),\\
            \ub(\bm{w}_{[n]:i}^{\top}\bm{z}) \ub(\sum_{j=1}^{k}s_{ij}\bm{H}_{\bm{z}}(x_{j}^{(n-1)}))\\
        \end{aligned}\right\}\\
        \lb((\bm{w}_{[n]:i}^{\top}\bm{z})\sum_{j=1}^{k}s_{ij}\bm{H}_{\bm{z}}(x_{j}^{(n-1)})) = \min(S)\\
        \ub((\bm{w}_{[n]:i}^{\top}\bm{z})\sum_{j=1}^{k}s_{ij}\bm{H}_{\bm{z}}(x_{j}^{(n-1)})) = \max(S)\,.\\
    \end{aligned}
    \label{eq:bounds_hess_NCP_right_part}
\end{equation}
Similarly to gradient bounds, we can recursively compute $\lb(\bm{H}_{\bm{z}}(x_{i}^{(N)}))$ and $\ub(\bm{H}_{\bm{z}}(x_{i}^{(N)}))$ starting with $\lb(\bm{H}_{\bm{z}}(x_{i}^{(1)})) = \ub(\bm{H}_{\bm{z}}(x_{i}^{(1)})) = \bm{0}_{d \times d}$. Finally, bounds of the Hessian matrix of the output can be obtained with:
\begin{equation}
    \begin{aligned}
        \lb(\bm{H}_{\bm{z}}(f(\bm{z})_i)) &= \sum_{j = 1}^{k}c_{ij}^{+}\lb(\bm{H}_{\bm{z}}(x_{j}^{(N)})) + c_{ij}^{-}\ub(\bm{H}_{\bm{z}}(x_{j}^{(N)}))\\
        \ub(\bm{H}_{\bm{z}}(f(\bm{z})_i)) &= \sum_{j = 1}^{k}c_{ij}^{+}\ub(\bm{H}_{\bm{z}}(x_{j}^{(N)})) + c_{ij}^{-}\lb(\bm{H}_{\bm{z}}(x_{j}^{(N)}))\,.
    \end{aligned}
    \label{eq:bounds_hessian_output}
\end{equation}
\section{BaB algorithm for PN robustness verification}
\label{app:BaB_algorithm}

BaB algorithms \citep{Land_Doig1960bab} are a well known approach to global optimization \citep{Horst1996GlobalOpt}. These algorithms intend to find the global minima of a optimization problem in the form of \cref{eq:verif_problem}, therefore guaranteeing soundness and completeness for verification. In this Section we present the details of our BaB based verification algorithm, we prove the theoretical convergence of BaB to the global minima of \cref{eq:verif_problem} and the complexity of its key steps.

Firstly, the property that we want to verify or falsify is the one given by \cref{eq:Robust_def}. This property is defined by {$(i)$} the network $f$, {$(ii)$} the adversarial budget $\epsilon$, {$(iii)$} a correctly classified input $\bm{z}_{0} : \argmax \bm{f}(\bm{z}_{0}) = t$. In order to verify the property, %
it is necessary that for each $\gamma \neq t$, the global minima of \cref{eq:verif_problem} is greater than 0, i.e., $\forall \gamma \neq t : v^{*} = \min_{\bm{z} \in C_{\text{in}}} f(\bm{z})_{t} - f(\bm{z})_{\gamma} > 0$. On the contrary, in order to falsify \cref{eq:Robust_def}, it is sufficient that for any $\gamma \neq t$ the global minima of \cref{eq:verif_problem} is smaller or equal than 0, i.e., $ \exists \gamma \neq t : v^{*} = \min_{\bm{z} \in C_{\text{in}}} f(\bm{z})_{t} - f(\bm{z})_{\gamma} \leq 0$. In order to reduce execution times, we heuristically sort all the $\gamma$ %
in decreasing order by network output, $\{\gamma_{i} : \gamma_{i} \neq t, f(\bm{z}_{0})_{\gamma_{i}} \geq f(\bm{z}_{0})_{\gamma_{j}} \forall j > i \}$ and solve \cref{eq:verif_problem} until one global minima is smaller or equal to 0 or all global minimas are greater than 0.

In order to solve \cref{eq:verif_problem}, we use \cref{alg:bab}. Without any modifications, this algorithm converges to the global minima. However, for verification, it is sufficient to find that the lower bound of the global minima ({\tt global\_lb}) cannot be smaller than or equal to 0, or that the upper bound of the global minima in a subset ({\tt prob\_ub}) is smaller than 0, i.e., there exists an adversarial example in that subset. Therefore we can cut the execution early when employed for verification, these optional changes are highlighted in red in \cref{alg:bab}.

\begin{algorithm}[!htb]
    \caption{Branch and Bound, adapted from \cite{Bunel2019BaB_first}}\label{alg:bab}
    \small
    \begin{algorithmic}[1]
    \algrenewcommand\algorithmicindent{2.0em}%
        \Function{BaB}{$\mtt{f}, \mtt{l}, \mtt{u}, \mtt{t}, \mtt{\gamma}$}
            \State $\mtt{global\_ub} \gets \inf$
            \State $\mtt{global\_lb} \gets - \inf$
            \State $\mtt{\alpha} \gets \mtt{compute\_alpha}(\mtt{f}, \mtt{l}, \mtt{u}, \mtt{t}, \mtt{\gamma})$ \Comment{\textcolor{teal}{\cref{alg:pow}}}
            \State $\mtt{probs}\gets \left[ (\mtt{global\_lb},\mtt{l}, \mtt{u}) \right]$
            \While{$\mtt{global\_ub} - \mtt{global\_lb} > 10^{-6} \textcolor{red}{~\textbf{and}~ \mtt{global\_lb} \leq 0}$}
                \State $([], \mtt{l'}, \mtt{u'}) \gets \mtt{pick\_out}(\mtt{probs})$ \Comment{\textcolor{teal}{Take subset with the minimum lower bound}}
                \State $\left[ (\mtt{l\_1} ,\mtt{u\_1}), \dots, (\mtt{l\_s}, \mtt{u\_s}) \right] \gets \mtt{split}(\mtt{l'},\mtt{u'})$ \Comment{\textcolor{teal}{Split widest input variable interval in halves}}
                \For{$i = 1 \dots s $}
                    \State $\mtt{prob\_ub} \gets \mtt{compute\_UB}(\mtt{f}, \mtt{l\_i}, \mtt{u\_i}, \mtt{t}, \mtt{\gamma})$ \Comment{\textcolor{teal}{PGD over the original function $g(\bm{z})$}}
                    \State $\mtt{prob\_lb} \gets \mtt{compute\_LB}(\mtt{f}, \mtt{l\_i}, \mtt{u\_i}, \mtt{t}, \mtt{\gamma}, \mtt{\alpha})$\Comment{\textcolor{teal}{PGD over $g_{\alpha}(\bm{z},\alpha)$}}
                    \If{$\mtt{prob\_ub} < \mtt{global\_ub}$}
                        \State $\mtt{global\_ub} \gets \mtt{prob\_ub}$
                        \State $\mtt{prune\_problems}(\mtt{probs}, \mtt{global\_ub})$ \Comment{\textcolor{teal}{Remove if $\mtt{prob\_lb} > \mtt{global\_ub}$}}
                    \EndIf
                    \If{$\mtt{prob\_lb} < \mtt{global\_ub} \textcolor{red}{~\textbf{and}~ \mtt{prob\_lb} \leq 0}$}
                        \State $\mtt{probs}.\mtt{append}((\mtt{prob\_lb}, \mtt{l\_i}, \mtt{u\_i}))$
                    \EndIf
                    \textcolor{red}{\If{$\mtt{prob\_ub} \leq 0$} \Comment{\textcolor{teal}{An adversarial example was found}}
                        \State\Return $[], 0$
                    \EndIf}
                \EndFor
                \textcolor{red}{\If{$|\mtt{probs}| == 0$} \Comment{\textcolor{teal}{$\mtt{prob\_lb} > 0$ for every subset}}
                        \State\Return $[], 1$
                \EndIf}
                \State $\mtt{global\_lb} \gets \min\{ \mtt{lb}\ |\ (\mtt{lb}, [], []) \in \mtt{probs}\}$
            \EndWhile
            \State\Return $\mtt{global\_ub}, \mtt{global\_ub} > 0$
        \EndFunction
    \end{algorithmic}
\end{algorithm} 

\cref{alg:bab} can theoretically be applied to any twice-differentiable classifier provided we have a method for computing an $\alpha$ for obtaining valid $\alpha$-convexification bounds. For this matter, we propose \cref{alg:pow}, which again can be used for any twice-differentiable classifier if we are able to compute its lower bounding Hessian $\bm{L_{H}}$. In the PN case, we propose a method for evaluating the matrix-vector product $\bm{L_{H}}\bm{v}$, this is covered in \cref{alg:LHv}. \rebuttal{We note that in \cref{alg:pow}, for initializing the vector {\ttfamily v}, each position of the vector is randomly sampled in the [0,1] interval, then the resulting vector is rescaled so that $||\mtt{v}||_{2} = 1$.}
\begin{algorithm}[!htb]
    \caption{Power method for $\alpha$ estimation}\label{alg:pow}
    \small
    \begin{algorithmic}[1]
    \algrenewcommand\algorithmicindent{2.0em}%
        \Function{compute\_alpha}{$\mtt{f}, \mtt{l},\mtt{u},\mtt{t},\mtt{\gamma}$}
            \State $\mtt{v} \gets \mtt{init\_v}(\mtt{f},\mtt{l},\mtt{u},\mtt{t},\mtt{\gamma})$ \Comment{\textcolor{teal}{Ensure $\mtt{v}$ is not an eigenvector of $\bm{L_{H}}$ and $||\mtt{v}||_{2} = 1$}}
            \State $\mtt{prev\_v} \gets \bm{0}$
            \State $\mtt{r} \gets 0$
            \While{$||\mtt{v}-\mtt{prev\_v}||_{2} > 10^{-6} $}
                \State $\mtt{prev\_v} \gets \mtt{v}$
                \State $\mtt{v} \gets \mtt{evaluate\_LHv(\mtt{f}, \mtt{l},\mtt{u},\mtt{t},\mtt{\gamma}, \mtt{v})}$ \Comment{\textcolor{teal}{\cref{alg:LHv}}}
                \State $\mtt{r} \gets ||\mtt{v}||_{2}$
                \State $\mtt{v} \gets \mtt{v}/\mtt{r}$
                \State $\mtt{v} \gets \mtt{evaluate\_LHv(\mtt{f}, \mtt{l},\mtt{u},\mtt{t},\mtt{\gamma}, \mtt{v})}$\Comment{\textcolor{teal}{Evaluate $\bm{L_{H}} \bm{v}$ twice to deal with negative eigenvalues.}}
                \State $\mtt{r} \gets ||\mtt{v}||_{2}$
                \State $\mtt{v} \gets \mtt{v}/\mtt{r}$
            \EndWhile
            \State\Return $\mtt{r}/2$
        \EndFunction
    \end{algorithmic}
\end{algorithm}

\textbf{Complexity of $\bm{L_{H}}\bm{v}$ evaluation} \cref{alg:LHv} is governed by an outer loop which performs $N-1$ iterations, see line 5. Inside the loop, the most expensive operations are in lines 6, 9, 10, 11 and 12, the rest of operations can be performed in $O(d)$ time. In line 5, the bounds of the gradients are computed, this operation can be performed for every level $n \in [N]$ outside the main loop and store the results using $O(N \cdot d \cdot k)$ time. For lines 9, 10, 11 and 12, an outer loop with $k$ iterations is used for the summation. Then, inside the summation, four dot products plus vector-scalar multiplications are performed, leading to a time complexity of $O(k \cdot d)$. Overall, the complexity of \cref{alg:LHv} is $O(N\cdot d \cdot k)$. 
\begin{algorithm}[!htb]
    \caption{Evaluation of $\bm{L_{H}}\bm{v}$ for a CCP PN, implementation of \cref{prop:lower bounding_Hessian_vector}}\label{alg:LHv}
    \small
    \begin{algorithmic}[1]
    \algrenewcommand\algorithmicindent{2.0em}%
        \Function{evaluate\_LHv}{$\mtt{f}, \mtt{problem}, \mtt{v}, t, \gamma$}
            \State $\mtt{lw} \gets \bm{c}_{t:} - \bm{c}_{\gamma:}$
            \State $\mtt{uw} \gets \bm{c}_{t:} - \bm{c}_{\gamma:}$ \Comment{\textcolor{teal}{Initial upper and lower bounds of the $\delta$ weight are given by the last linear layer.}}
            \State $\mtt{result} \gets \bm{0}$
            \For{$n = \mtt{degree}(f) \dots 2$}
                \State $\mtt{lg}, \mtt{ug} \gets \lb(\bm{J}_{\bm{z}}(\bm{x}^{(n)})), \ub(\bm{J}_{\bm{z}}(\bm{x}^{(n)}))$
                \State $\mtt{S1} \gets \{\mtt{lg}\cdot\mtt{lw}, \mtt{lg}\cdot\mtt{uw}, \mtt{ug}\cdot\mtt{lw}, \mtt{ug}\cdot\mtt{uw}\}$
                \State $\mtt{lg}, \mtt{ug} \gets \mtt{min}(\mtt{S1}), \mtt{max}(\mtt{S1})$
                \State $\mtt{Lv} \gets \sum_{i = 1}^{k}{\bm{w}_{[n]:i}}^{+} \cdot \mtt{lg[i,:]}^{\top}\mtt{v}  + {\bm{w}_{[n]:i}}^{-} \cdot \mtt{ug[i,:]}^{\top}\mtt{v} + \mtt{lg[i,:]} \cdot {\bm{w}_{[n]:i}^{\top}}^{+}\mtt{v}  + \mtt{ug[i,:]} \cdot {\bm{w}_{[n]:i}^{\top}}^{-}\mtt{v}$
                \State $\mtt{Uv} \gets \sum_{i = 1}^{k}{\bm{w}_{[n]:i}}^{+} \cdot \mtt{ug[i,:]}^{\top}\mtt{v}  + {\bm{w}_{[n]:i}}^{-} \cdot \mtt{lg[i,:]}^{\top}\mtt{v} + \mtt{ug[i,:]} \cdot {\bm{w}_{[n]:i}^{\top}}^{+}\mtt{v}  + \mtt{lg[i,:]} \cdot {\bm{w}_{[n]:i}^{\top}}^{-}\mtt{v}$
                \State $\mtt{L1} \gets \sum_{i = 1}^{k}{\bm{w}_{[n]:i}}^{+} \cdot \mtt{lg[i,:]}^{\top}\bm{1}  + {\bm{w}_{[n]:i}}^{-} \cdot \mtt{ug[i,:]}^{\top}\bm{1} + \mtt{lg[i,:]} \cdot {\bm{w}_{[n]:i}^{\top}}^{+}\bm{1}  + \mtt{ug[i,:]} \cdot {\bm{w}_{[n]:i}^{\top}}^{-}\bm{1}$
                \State $\mtt{U1} \gets \sum_{i = 1}^{k}{\bm{w}_{[n]:i}}^{+} \cdot \mtt{ug[i,:]}^{\top}\bm{1}  + {\bm{w}_{[n]:i}}^{-} \cdot \mtt{lg[i,:]}^{\top}\bm{1} + \mtt{ug[i,:]} \cdot {\bm{w}_{[n]:i}^{\top}}^{+}\bm{1}  + \mtt{lg[i,:]} \cdot {\bm{w}_{[n]:i}^{\top}}^{-}\bm{1}$
                \State $\mtt{result} \gets \mtt{result} + (\mtt{Lv} + \mtt{Uv})/2 + ((\mtt{L1} - \mtt{U1})/2) * \mtt{v}$
                \State $\mtt{lx}$, $\mtt{ux} \gets \lb(\bm{W}_{[n]}^{\top} \bm{z} + 1), \ub(\bm{W}_{[n]}^{\top} \bm{z} + 1)$
                \State $\mtt{S2} \gets \{\mtt{lx}\cdot\mtt{lw}, \mtt{lx}\cdot\mtt{uw}, \mtt{ux}\cdot\mtt{lw}, \mtt{ux}\cdot\mtt{uw}\}$
                \State $\mtt{lw}, \mtt{uw} \gets \mtt{min}(\mtt{S2}), \mtt{max}(\mtt{S2})$ \Comment{\textcolor{teal}{Update bounds of the weight ($\lb{\delta'},\ub{\delta'}$), see \cref{prop:lower bounding_Hessian_vector}}}
                
            \EndFor
        \State\Return $\mtt{result}$
        \EndFunction
    \end{algorithmic}
\end{algorithm}

\section{Auxiliary experimental results and discussion}
\label{app:Appendix}

We start this Section by comparing the performance of our method with ReLU NN verification algorithms \cref{subsec:beta_crown}. We analyze a limitation of the method in \cref{subsec:experiments_limitations}. Then, in \cref{subsec:ablation_input} we perform an ablation study on the effect the input size of the network has in our PN verification algorithm.

\textbf{Additional notation}: In addition to the notation already defined in the main paper, we use $\circ$ for convolutions.

\subsection{Comparison with ReLU BaB verification algorithms.}
\label{subsec:beta_crown}

\begin{table}[!htb]
    \centering
    \caption{Description of convolutional PNs used in our experiments.}
    \begin{tabular}{c|ccccccccc}
        Name & degree / N & kernel size & stride & padding & channels  \\
        \hline
        PN\_Conv2 & $2$ & $5\times5$ & $2$ & $2$ & $32$\\
        PN\_Conv4 & $4$ & $7\times7$ & $4$ & $3$ & $64$\\
    \end{tabular}
    \label{tab:my_label}
\end{table}

Complete ReLU NN verification algorithms are usually benchmarked against other methods in the networks and $\epsilon$ values proposed by \citet{Singh2019}. %
ReLU NN verification algorithms mostly rely on the specific structure of the networks, e.g. ReLU activation, which makes a direct comparison with ReLU nets hard. However, we match these benchmarks by training PNs with same degree as the number of ReLU layers and similar number of parameters. For a detailed description of these networks check \cref{tab:comparable_fully_connected} and \cref{tab:comparable_convolutional}.

\begin{table}[!htb]
    \centering
    \caption{Verification results on our proposed PN verification benchmarks. We run our verification procedure over the first $1000$ images of the test split of each dataset. {\tt Time(s)} refers to the average running time per image when the verification was not timed out, i.e., we can verify or falsify the property in the given time limit, {\tt Ver.\%} is the verified accuracy and {\tt U.B.} its upper bound. We use the same $\epsilon$ values as \citet{Wang2021Beta-CROWN}. We observe that in comparison with $\beta$-CROWN, our method generally has a larger gap between verified accuracy and its upper bound.%
    }
    \begin{tabular}{c|cccc|cccccccccccccccc}
        & \multicolumn{4}{c|}{$\beta$-CROWN$^{*}$} & & & &\\
        \multirow{2}{*}{Dataset} & \multicolumn{4}{c|}{\cite{Wang2021Beta-CROWN}} & \multicolumn{4}{c}{VPN}\\ 
        & Model & Time(s) & Ver.\% & U.B. & Model & Time(s) & Ver.\% & U.B.\\ 
        \hline 
        \multirow{6}{*}{MNIST} & $6\times100$ & $102$ & $69.9$ & $84.2$ & $5\times30$ & $34$ & $24.7$ & $81.1$\\ 
        & $6\times200$ & $86$ & $77.4$ & $90.1$ & $5\times81$ & $60$ & $88.0$ & $91.1$\\ 
        & $9\times100$ & $103$ & $62.0$& $82.0$ & $8\times24$ & $33$ & $0.0$ & $80.2$\\ 
        & $9\times200$ & $95$ & $73.5$ & $91.1$ & $8\times70$ & $94$ & $0.6$ & $91.6$\\
        & ConvSmall & $7.0$ & $72.7$ & $73.2$ & PN\_Conv2 & $3.0$ & $3.5$ & $12.1$\\
        & ConvBig & $3.1$ & $79.3$ & $80.4$ & PN\_Conv4 & $8.9$ & $0.0$ & $0.0$ \\
        \hline
        \multirow{2}{*}{CIFAR10} & ConvSmall & $6.8$ & $46.3$  & $48.1$ & PN\_Conv2 & $63.4$ & $15.7$ & $16.1$\\
        & ConvBig & $15.3$ & $51.6$ & $55.0$ & PN\_Conv4 & $161.4$ & $9.4$ & $16.3$\\
        \multicolumn{9}{l}{\footnotesize $^*$Note: We report numbers from \citet{Wang2021Beta-CROWN}.}\\
    \end{tabular}
    \label{tab:comparison_alpha_crown}
\end{table}

When comparing with the SOTA verifier $\beta$-CROWN \citep{Wang2021Beta-CROWN}, we firstly observe that the upper bound of verified accuracy is very similar between the NNs and PNs fully connected benchmarks. However, for convolutional PNs (C-PNs) the upper bound is much lower, even $0$ for PN\_Conv4 trained in the MNIST dataset, see \cref{tab:comparison_alpha_crown}. Secondly, we observe that the gap between the upper bound and verified accuracy ({\tt{U.B.}} and {\tt{Ver.\%}} in \cref{tab:comparison_alpha_crown}) is small for $\beta$-CROWN. In our case, this gap is only small for the shallowest PN and smallest $\epsilon$ ($5\times 81$ with $\epsilon = 0.015$), obtaining $88.0 \%$ verified accuracy and outperforming $\beta$-CROWN with only $77.4 \%$. For PN\_Conv2 trained in CIFAR10 and evaluated at $\epsilon = 2/255$, our verified accuracy, $15.7 \%$, is also really close to the upper bound $16.1 \%$, but both are low in comparison with the $\beta$-CROWN equivalent.

\begin{table}[tb]
    \centering
    \caption{Comparison of PNs to the fully connected ReLU network benchmarks proposed in \citet{Singh2019}. {\tt\#Par.} refers to the number of parameters in the ReLU benchmark and our PNs respectively. We build PNs with same degree as number of non-linearities in their corresponding ReLU NN benchmark. We also adjust the hidden size so that the number of parameters is matched.}
    \begin{tabular}{ccc|cccc}
    ReLU NN & \#Par. & Acc.(\%) & PN & \#Par. & Acc.(\%)\\
    \hline
    $6 \times 100$ & $119,910$ & $96.0$ & $5\times30$ & $117,910$ & $97.2$\\
    $9 \times 100$ & $150,210$ & $94.7$ & $8\times24$ & $150,778$ & $97.5$\\
    $6 \times 200$ & $319,810$ & $97.2$ & $5\times81$ & $318,340$ & $96.2$\\
    $9 \times 200$ & $440,410$ & $95.0$ & $8\times70$ & $439,750$ & $96.7$\\
    
    \end{tabular}
    \label{tab:comparable_fully_connected}
\end{table}

\begin{table}[tb]
    \centering
    \caption{Comparison of convolutional PNs (C-PNs) to the Convolutional ReLU network benchmarks proposed in \citet{Singh2019}. {\tt\#Par.} refers to the number of parameters in the ReLU benchmark and our PNs respectively. We build C-PNs with same number of convolutional layers as their corresponding ReLU NN benchmark. %
    }
    \begin{tabular}{c|ccc|cccccc}
    Dataset & ReLU NN & \#Par. & Acc.(\%) & PN & \#Par. & Acc.(\%)\\
    \hline
    MNIST & Conv Small & $89,606$ & $98.0$ & PN\_Conv2 & $114,218$ & $98.0$\\
    MNIST & Conv Big & $1,974,762$ & $92.9$ & PN\_Conv4 & $834,762$ & $98.0$\\
    CIFAR10 & Conv Small & $125,318$ & $63.0$ & PN\_Conv2 & $133,994$ & $57.2$\\
    CIFAR10 & Conv Big & $2,466,858$ & $63.1$ & PN\_Conv4 & $845,514$ & $58.3$\\
    \end{tabular}
    \label{tab:comparable_convolutional}
\end{table}

\subsection{Limitations of the proposed method}
\label{subsec:experiments_limitations}

Scaling our method to high-degree settings is non-trivial due to IBP approximation errors, that accumulate when propagating the bounds through each layer of the PN. Nevertheless, empirical evidence demonstrates that a lower-degree PN is enough for obtaining comparative performance, e.g.,\citet{fan2020universal}.

\subsection{Effect of the input size in PN verification}
\label{subsec:ablation_input}

In this experiment we evaluate the effect of the input size in the verification of a PN. In order to evaluate this, we train $3$ different CCP networks over the STL10 dataset. Each model has been trained with STL10 images preprocessed with 3 different resizing factors. Every network is a CCP\_$4 \times 25$ trained with a learning rate of $10^{-4}$.

\begin{table}[!htb]
    \centering
    \caption{Input size ablation study: CCP\_$4 \times 25$ networks are trained over STL10 with three different input sizes. Both Verified accuracy ({\tt Ver.\%}) and its upper bound ({\tt U.B.}) increase when the input size is reduced for all the studied $\epsilon$ values.}
    \begin{tabular}{cccccccc}
        \multirow{2}{*}{Input size} & \multirow{2}{*}{Acc.\%} & \multirow{2}{*}{$\epsilon$} & \multicolumn{2}{c}{VPN}\\
        & & & Time(s) & Ver.\% & U.B.\\
        \hline
        $3 \times 96 \times 96$ & $35.1$ & $1/255$ & $47.4$ & $4.8$ & $10.6$\\
        (original)& & $2/255$ & $19.0$ & $0.0$ & $2.4$\\
        & & $8/255$ & $21.5$ & $0.0$ & $0.0$\\
        \hline
        $3 \times 64 \times 64$ & ${35.6}$ & $1/255$ & $52.1$ &  $11.7$ & $12.6$\\
        & & $2/255$ & $20.7$ & $0.6$ & $3.2$\\
        & & $8/255$ & $12.0$ & $0.0$ & $0.0$\\
        \hline
        $3 \times 32 \times 32$ &  $31.8$ &  $1/255$ & $14.7$ & $\mathbf{17.9}$ & $\mathbf{18.0}$\\
        & & $2/255$ & $21.4$ & $\mathbf{7.2}$ & $\mathbf{8.9}$\\
        & & $8/255$ & $16.2$ & $0.0$ & $\mathbf{0.1}$ \\
    \end{tabular}
    
    \label{tab:ablation_input}
\end{table}

As seen in \cref{tab:ablation_input}, downsampling the input images results in a decrease in the accuracy (i.e., from $35.1\%$ to $31.8\%$ at $32\times32$ resolution). However, downsampling the input improves the robustness of the network. For all $\epsilon$ values we observe less successful adversarial attacks, i.e., higher upper bound of the verified accuracy ({\tt U.B.}). In addition, the verification process is improved, we obtain higher values for verified accuracy ({\tt Ver.\%}), but also the gap with its upper bound is reduced. We believe this phenomenon %
is due to the networks learning more robust representations with a smaller input size \citep{raghunathan2020understanding,zhang2019theoretically}.%

\section{Proofs}
\label{app:proofs}
\subsection{IBP on rank-1 matrices}
\label{subsec:IBP_matrices}
\rebuttal{In order to facilitate the computation of bounds of the Hessian matrix, we introduce some general properties regarding the bounds IBP of rank-1 matrices.}
\begin{lemma}
\rebuttal{Let $\bm{M} = \bm{u}\bm{v}^{\top}$ be a rank-1 matrix defined by vectors $\bm{u} \in \realnum^{d_{1}}$ and $\bm{v} \in \realnum^{d_{2}}$. Let $\bm{u}^{+}$ and $\bm{v}^{+}$ be the positive parts and $\bm{u}^{-}$ and $\bm{v}^{-}$ the negative parts, satisfying $\bm{u} = \bm{u}^{+} + \bm{u}^{-}$ and $\bm{v} = \bm{v}^{+} + \bm{v}^{-}$. The matrix $\bm{M}$ can be decomposed as $\bm{M} = \bm{u}^{+}{\bm{v}^{+}}^{\top} + \bm{u}^{+}{\bm{v}^{-}}^{\top} + \bm{u}^{-}{\bm{v}^{+}}^{\top} +  \bm{u}^{-}{\bm{v}^{-}}^{\top}$.}
\label{lem:pos_neg_decom}
\end{lemma}
\begin{proof}[Proof of \cref{lem:pos_neg_decom}]
\rebuttal{Following $\bm{u} = \bm{u}^{+} + \bm{u}^{-}$ and $\bm{v} = \bm{v}^{+} + \bm{v}^{-}$:
\begin{equation}
    \begin{aligned}
        \bm{M} \coloneqq \bm{u}\bm{v}^{\top} = (\bm{u}^{+} + \bm{u}^{-})(\bm{v}^{+} + \bm{v}^{-})^{\top} = \bm{u}^{+}{\bm{v}^{+}}^{\top} + \bm{u}^{+}{\bm{v}^{-}}^{\top} + \bm{u}^{-}{\bm{v}^{+}}^{\top} +  \bm{u}^{-}{\bm{v}^{-}}^{\top}\\
    \end{aligned}
\end{equation}}
\end{proof}
\begin{lemma}
\rebuttal{Let $\bm{M} = \bm{u}\bm{v}^{\top}$ be a rank-1 matrix defined by vectors $\bm{u} \in \realnum^{d_{1}}$ and $\bm{v} \in \realnum^{d_{2}}$. Let $\bm{u}$ ($\bm{v}$) be lower and upper bounded by $\lb(\bm{u})$ and $\ub(\bm{u})$ ($\lb(\bm{v})$ and $\ub(\bm{v})$). The matrix $\bm{M}$ is lower and upper bounded by:
\begin{equation}
    \begin{aligned}
        \bm{M} & \geq \lb(\bm{M}) = \lb(\bm{u})^{+}{\lb(\bm{v})^{+}}^{\top} + \ub(\bm{u})^{+}{\lb(\bm{v})^{-}}^{\top} + \lb(\bm{u})^{-}{\ub(\bm{v})^{+}}^{\top} + \ub(\bm{u})^{-}{\ub(\bm{v})^{-}}^{\top}\\
        \bm{M} & \leq \ub(\bm{M}) = \ub(\bm{u})^{+}{\ub(\bm{v})^{+}}^{\top} + \lb(\bm{u})^{+}{\ub(\bm{v})^{-}}^{\top} + \ub(\bm{u})^{-}{\lb(\bm{v})^{+}}^{\top} + \lb(\bm{u})^{-}{\lb(\bm{v})^{-}}^{\top}
    \end{aligned}
\end{equation}}
\label{lem:bounds_rank1}
\end{lemma}
\begin{proof}[Proof of \cref{lem:bounds_rank1}]
\rebuttal{By applying the multiplication rule of IBP, see \cref{eq:IBP_basic_rules}, we can express the lower bound as:
\begin{equation}
    \begin{aligned}
        \bm{M} \geq \min\{&\lb(\bm{u})\lb(\bm{v})^{\top},\\
        &\lb(\bm{u})\ub(\bm{v})^{\top},\\
        &\ub(\bm{u})\lb(\bm{v})^{\top},\\
        &\ub(\bm{u})\ub(\bm{v})^{\top}\}\,.\\
    \end{aligned}
\end{equation}
Then, applying \cref{lem:pos_neg_decom} on each term we obtain:
\begin{equation}
    \begin{aligned}
        \bm{M} \geq \min\{&\textcolor{blue}{\lb(\bm{u})^{+}{\lb(\bm{v})^{+}}^{\top}} + \textcolor{orange}{\lb(\bm{u})^{+}{\lb(\bm{v})^{-}}^{\top}} + \textcolor{green}{\lb(\bm{u})^{-}{\lb(\bm{v})^{+}}^{\top}} + \textcolor{purple}{\lb(\bm{u})^{-}{\lb(\bm{v})^{-}}^{\top}},\\
        &\textcolor{blue}{\lb(\bm{u})^{+}{\ub(\bm{v})^{+}}^{\top}} + \textcolor{orange}{\lb(\bm{u})^{+}{\ub(\bm{v})^{-}}^{\top}} + \textcolor{green}{\lb(\bm{u})^{-}{\ub(\bm{v})^{+}}^{\top}} + \textcolor{purple}{\lb(\bm{u})^{-}{\ub(\bm{v})^{-}}^{\top}},\\
        &\textcolor{blue}{\ub(\bm{u})^{+}{\lb(\bm{v})^{+}}^{\top}} + \textcolor{orange}{\ub(\bm{u})^{+}{\lb(\bm{v})^{-}}^{\top}} + \textcolor{green}{\ub(\bm{u})^{-}{\lb(\bm{v})^{+}}^{\top}} + \textcolor{purple}{\ub(\bm{u})^{-}{\lb(\bm{v})^{-}}^{\top}},\\
        &\textcolor{blue}{\ub(\bm{u})^{+}{\ub(\bm{v})^{+}}^{\top}} + \textcolor{orange}{\ub(\bm{u})^{+}{\ub(\bm{v})^{-}}^{\top}} + \textcolor{green}{\ub(\bm{u})^{-}{\ub(\bm{v})^{+}}^{\top}} + \textcolor{purple}{\ub(\bm{u})^{-}{\ub(\bm{v})^{-}}^{\top}}\}\,.\\
    \end{aligned}
    \label{eq:lb_rank1_expanded}
\end{equation}
where color is used to group related terms and ease the reading. Grouping by color, it is easily found that:
\begin{equation}
    \begin{aligned}
        \textcolor{blue}{\lb(\bm{u})^{+}{\lb(\bm{v})^{+}}^{\top}} & = \min\{\textcolor{blue}{\lb(\bm{u})^{+}{\lb(\bm{v})^{+}}^{\top}}, \textcolor{blue}{\lb(\bm{u})^{+}{\ub(\bm{v})^{+}}^{\top}}, \textcolor{blue}{\ub(\bm{u})^{+}{\lb(\bm{v})^{+}}^{\top}}, \textcolor{blue}{\ub(\bm{u})^{+}{\ub(\bm{v})^{+}}^{\top}}\}\\
        \textcolor{orange}{\ub(\bm{u})^{+}{\lb(\bm{v})^{-}}^{\top}} & = \min\{\textcolor{orange}{\lb(\bm{u})^{+}{\lb(\bm{v})^{-}}^{\top}}, \textcolor{orange}{\lb(\bm{u})^{+}{\ub(\bm{v})^{-}}^{\top}}, \textcolor{orange}{\ub(\bm{u})^{+}{\lb(\bm{v})^{-}}^{\top}}, \textcolor{orange}{\ub(\bm{u})^{+}{\ub(\bm{v})^{-}}^{\top}}\}\\
        \textcolor{green}{\lb(\bm{u})^{-}{\ub(\bm{v})^{+}}^{\top}} & = \min\{\textcolor{green}{\lb(\bm{u})^{-}{\lb(\bm{v})^{+}}^{\top}}, \textcolor{green}{\lb(\bm{u})^{-}{\ub(\bm{v})^{+}}^{\top}}, \textcolor{green}{\ub(\bm{u})^{-}{\lb(\bm{v})^{+}}^{\top}}, \textcolor{green}{\ub(\bm{u})^{-}{\ub(\bm{v})^{+}}^{\top}}\}\\
        \textcolor{purple}{\ub(\bm{u})^{-}{\ub(\bm{v})^{-}}^{\top}} & = \min\{\textcolor{purple}{\lb(\bm{u})^{-}{\lb(\bm{v})^{-}}^{\top}}, \textcolor{purple}{\lb(\bm{u})^{-}{\ub(\bm{v})^{-}}^{\top}}, \textcolor{purple}{\ub(\bm{u})^{-}{\lb(\bm{v})^{-}}^{\top}}, \textcolor{purple}{\ub(\bm{u})^{-}{\ub(\bm{v})^{-}}^{\top}}\}\,.\\
    \end{aligned}
\end{equation}
Lastly, by substituting each term by the minimum of the terms with the same color on \cref{eq:lb_rank1_expanded}:
\begin{equation}
    \bm{M} \geq \textcolor{blue}{\lb(\bm{u})^{+}{\lb(\bm{v})^{+}}^{\top}} + \textcolor{orange}{\ub(\bm{u})^{+}{\lb(\bm{v})^{-}}^{\top}} + \textcolor{green}{\lb(\bm{u})^{-}{\ub(\bm{v})^{+}}^{\top}} + \textcolor{purple}{\ub(\bm{u})^{-}{\ub(\bm{v})^{-}}^{\top}} = \lb(\bm{M})\,.
\end{equation}
Analogously for the upper bound:
\begin{equation}
    \begin{aligned}
        \bm{M} \leq \max\{&\textcolor{blue}{\lb(\bm{u})^{+}{\lb(\bm{v})^{+}}^{\top}} + \textcolor{orange}{\lb(\bm{u})^{+}{\lb(\bm{v})^{-}}^{\top}} + \textcolor{green}{\lb(\bm{u})^{-}{\lb(\bm{v})^{+}}^{\top}} + \textcolor{purple}{\lb(\bm{u})^{-}{\lb(\bm{v})^{-}}^{\top}},\\
        &\textcolor{blue}{\lb(\bm{u})^{+}{\ub(\bm{v})^{+}}^{\top}} + \textcolor{orange}{\lb(\bm{u})^{+}{\ub(\bm{v})^{-}}^{\top}} + \textcolor{green}{\lb(\bm{u})^{-}{\ub(\bm{v})^{+}}^{\top}} + \textcolor{purple}{\lb(\bm{u})^{-}{\ub(\bm{v})^{-}}^{\top}},\\
        &\textcolor{blue}{\ub(\bm{u})^{+}{\lb(\bm{v})^{+}}^{\top}} + \textcolor{orange}{\ub(\bm{u})^{+}{\lb(\bm{v})^{-}}^{\top}} + \textcolor{green}{\ub(\bm{u})^{-}{\lb(\bm{v})^{+}}^{\top}} + \textcolor{purple}{\ub(\bm{u})^{-}{\lb(\bm{v})^{-}}^{\top}},\\
        &\textcolor{blue}{\ub(\bm{u})^{+}{\ub(\bm{v})^{+}}^{\top}} + \textcolor{orange}{\ub(\bm{u})^{+}{\ub(\bm{v})^{-}}^{\top}} + \textcolor{green}{\ub(\bm{u})^{-}{\ub(\bm{v})^{+}}^{\top}} + \textcolor{purple}{\ub(\bm{u})^{-}{\ub(\bm{v})^{-}}^{\top}}\}\,,\\
    \end{aligned}
    \label{eq:ub_rank1_expanded}
\end{equation}
where color is used to group related terms and ease the reading. Grouping by color, it is easily found that:
\begin{equation}
    \begin{aligned}
        \textcolor{blue}{\ub(\bm{u})^{+}{\ub(\bm{v})^{+}}^{\top}} & = \max\{\textcolor{blue}{\lb(\bm{u})^{+}{\lb(\bm{v})^{+}}^{\top}}, \textcolor{blue}{\lb(\bm{u})^{+}{\ub(\bm{v})^{+}}^{\top}}, \textcolor{blue}{\ub(\bm{u})^{+}{\lb(\bm{v})^{+}}^{\top}}, \textcolor{blue}{\ub(\bm{u})^{+}{\ub(\bm{v})^{+}}^{\top}}\}\\
        \textcolor{orange}{\lb(\bm{u})^{+}{\ub(\bm{v})^{-}}^{\top}} & = \max\{\textcolor{orange}{\lb(\bm{u})^{+}{\lb(\bm{v})^{-}}^{\top}}, \textcolor{orange}{\lb(\bm{u})^{+}{\ub(\bm{v})^{-}}^{\top}}, \textcolor{orange}{\ub(\bm{u})^{+}{\lb(\bm{v})^{-}}^{\top}}, \textcolor{orange}{\ub(\bm{u})^{+}{\ub(\bm{v})^{-}}^{\top}}\}\\
        \textcolor{green}{\ub(\bm{u})^{-}{\lb(\bm{v})^{+}}^{\top}} & = \max\{\textcolor{green}{\lb(\bm{u})^{-}{\lb(\bm{v})^{+}}^{\top}}, \textcolor{green}{\lb(\bm{u})^{-}{\ub(\bm{v})^{+}}^{\top}}, \textcolor{green}{\ub(\bm{u})^{-}{\lb(\bm{v})^{+}}^{\top}}, \textcolor{green}{\ub(\bm{u})^{-}{\ub(\bm{v})^{+}}^{\top}}\}\\
        \textcolor{purple}{\lb(\bm{u})^{-}{\lb(\bm{v})^{-}}^{\top}} & = \max\{\textcolor{purple}{\lb(\bm{u})^{-}{\lb(\bm{v})^{-}}^{\top}}, \textcolor{purple}{\lb(\bm{u})^{-}{\ub(\bm{v})^{-}}^{\top}}, \textcolor{purple}{\ub(\bm{u})^{-}{\lb(\bm{v})^{-}}^{\top}}, \textcolor{purple}{\ub(\bm{u})^{-}{\ub(\bm{v})^{-}}^{\top}}\}\,.\\
    \end{aligned}
\end{equation}
Lastly, by substituting each term by the minimum of the terms with the same color on \cref{eq:ub_rank1_expanded}:
\begin{equation}
    \bm{M} \leq \textcolor{blue}{\ub(\bm{u})^{+}{\ub(\bm{v})^{+}}^{\top}} + \textcolor{orange}{\lb(\bm{u})^{+}{\ub(\bm{v})^{-}}^{\top}} + \textcolor{green}{\ub(\bm{u})^{-}{\lb(\bm{v})^{+}}^{\top}} + \textcolor{purple}{\lb(\bm{u})^{-}{\lb(\bm{v})^{-}}^{\top}} = \ub(\bm{M})\,.
\end{equation}}
\end{proof}
\begin{lemma}
\rebuttal{Let $\bm{M} = \delta \cdot \bm{u}\bm{v}^{\top}$ be a rank-1 matrix defined by vectors $\bm{u} \in \realnum^{d_{1}}$ and $\bm{v} \in \realnum^{d_{2}}$ and a scalar $\delta \in \realnum$. Let $\bm{u}$ ($\bm{v}$) be lower and upper bounded by $\lb(\bm{u})$ and $\ub(\bm{u})$ ($\lb(\bm{v})$ and $\ub(\bm{v})$). Let $\delta$ be lower and upper bounded by $\lb(\delta)$ and $\ub(\delta)$. The matrix $\bm{M}$ is lower and upper bounded by:
\begin{equation}
    \begin{aligned}
        \bm{M} \geq \lb(\bm{M}) & = \lb(\delta)^{+}\left[\lb(\bm{u})^{+}{\lb(\bm{v})^{+}}^{\top} + \ub(\bm{u})^{-}{\ub(\bm{v})^{-}}^{\top}\right]\\
        & \quad + \ub(\delta)^{+}\left[\ub(\bm{u})^{+}{\lb(\bm{v})^{-}}^{\top} + \lb(\bm{u})^{-}{\ub(\bm{v})^{+}}^{\top}\right] \\
        & \quad + \lb(\delta)^{-}\left[\ub(\bm{u})^{+}{\ub(\bm{v})^{+}}^{\top} + \lb(\bm{u})^{-}{\lb(\bm{v})^{-}}^{\top}\right]\\
        & \quad + \ub(\delta)^{-}\left[\ub(\bm{u})^{-}{\lb(\bm{v})^{+}}^{\top} + \lb(\bm{u})^{+}{\ub(\bm{v})^{-}}^{\top}\right] \\
        & = \lb(\delta)^{+}\lb(\bm{u}\bm{v}^{\top})^{+} +  \ub(\delta)^{+}\lb(\bm{u}\bm{v}^{\top})^{-}\\
        & \quad + \lb(\delta)^{-}\ub(\bm{u}\bm{v}^{\top})^{+} + \ub(\delta)^{-}\ub(\bm{u}\bm{v}^{\top})^{-}\\
        \bm{M} \leq \ub(\bm{M}) & = \lb(\delta)^{+}\left[\ub(\bm{u})^{-}{\lb(\bm{v})^{+}}^{\top} + \lb(\bm{u})^{+}{\ub(\bm{v})^{-}}^{\top}\right]\\
        & \quad + \ub(\delta)^{+}\left[\ub(\bm{u})^{+}{\ub(\bm{v})^{+}}^{\top} + \lb(\bm{u})^{-}{\lb(\bm{v})^{-}}^{\top}\right] \\
        & \quad + \lb(\delta)^{-}\left[\ub(\bm{u})^{+}{\lb(\bm{v})^{-}}^{\top} + \lb(\bm{u})^{-}{\ub(\bm{v})^{+}}^{\top}\right]\\
        & \quad + \ub(\delta)^{-}\left[\lb(\bm{u})^{+}{\lb(\bm{v})^{+}}^{\top} + \ub(\bm{u})^{-}{\ub(\bm{v})^{-}}^{\top}\right] \\
        & = \lb(\delta)^{+}\ub(\bm{u}\bm{v}^{\top})^{-} +  \ub(\delta)^{+}\ub(\bm{u}\bm{v}^{\top})^{+}\\
        & \quad + \lb(\delta)^{-}\lb(\bm{u}\bm{v}^{\top})^{-} + \ub(\delta)^{-}\lb(\bm{u}\bm{v}^{\top})^{+}\\
    \end{aligned}
\end{equation}}
\label{lem:bounds_delta_rank1}
\end{lemma}
\begin{proof}[Proof of \cref{lem:bounds_delta_rank1}]
\rebuttal{By applying the multiplication rule of IBP, see \cref{eq:IBP_basic_rules}, we can express the lower bound as:
\begin{equation}
    \begin{aligned}
        \bm{M} \geq \min\{&\lb(\delta)\lb(\bm{u}\bm{v}^{\top}),\\
        &\lb(\delta)\ub(\bm{u}\bm{v}^{\top}),\\
        &\ub(\delta)\lb(\bm{u}\bm{v}^{\top}),\\
        &\ub(\delta)\ub(\bm{u}\bm{v}^{\top})\}\,.\\
    \end{aligned}
\end{equation}
Then, by decomposing each term via the possitive-negative decomposition of each operand, we arrive to:
\begin{equation}
    \begin{aligned}
        \bm{M} \geq \min\{&\textcolor{blue}{\lb(\delta)^{+}\lb(\bm{u}\bm{v}^{\top})^{+}} + \textcolor{orange}{\lb(\delta)^{+}\lb(\bm{u}\bm{v}^{\top})^{-}} + \textcolor{green}{\lb(\delta)^{-}\lb(\bm{u}\bm{v}^{\top})^{+}} + \textcolor{purple}{\lb(\delta)^{-}\lb(\bm{u}\bm{v}^{\top})^{-}},\\
        &\textcolor{blue}{\lb(\delta)^{+}\ub(\bm{u}\bm{v}^{\top})^{+}} + \textcolor{orange}{\lb(\delta)^{+}\ub(\bm{u}\bm{v}^{\top})^{-}} + \textcolor{green}{\lb(\delta)^{-}\ub(\bm{u}\bm{v}^{\top})^{+}} + \textcolor{purple}{\lb(\delta)^{-}\ub(\bm{u}\bm{v}^{\top})^{-}},\\
        &\textcolor{blue}{\ub(\delta)^{+}\lb(\bm{u}\bm{v}^{\top})^{+}} + \textcolor{orange}{\ub(\delta)^{+}\lb(\bm{u}\bm{v}^{\top})^{-}} + \textcolor{green}{\ub(\delta)^{-}\lb(\bm{u}\bm{v}^{\top})^{+}} + \textcolor{purple}{\ub(\delta)^{-}\lb(\bm{u}\bm{v}^{\top})^{-}},\\
        &\textcolor{blue}{\ub(\delta)^{+}\ub(\bm{u}\bm{v}^{\top})^{+}} + \textcolor{orange}{\ub(\delta)^{+}\ub(\bm{u}\bm{v}^{\top})^{-}} + \textcolor{green}{\ub(\delta)^{-}\ub(\bm{u}\bm{v}^{\top})^{+}} + \textcolor{purple}{\ub(\delta)^{-}\ub(\bm{u}\bm{v}^{\top})^{-}}\}\,.\\
    \end{aligned}
    \label{eq:lb_delta_rank1_expanded}
\end{equation}
where color is used to group related terms and ease the reading. Grouping by color, it is easily found that:
\begin{equation}
    \begin{aligned}
        \textcolor{blue}{\lb(\delta)^{+}\lb(\bm{u}\bm{v}^{\top})^{+}} & = \min\{\textcolor{blue}{\lb(\delta)^{+}\lb(\bm{u}\bm{v}^{\top})^{+}}, \textcolor{blue}{\lb(\delta)^{+}\ub(\bm{u}\bm{v}^{\top})^{+}}, \textcolor{blue}{\ub(\delta)^{+}\lb(\bm{u}\bm{v}^{\top})^{+}}, \textcolor{blue}{\ub(\delta)^{+}\ub(\bm{u}\bm{v}^{\top})^{+}}\}\\
        \textcolor{orange}{\ub(\delta)^{+}\lb(\bm{u}\bm{v}^{\top})^{-}} & = \min\{\textcolor{orange}{\lb(\delta)^{+}\lb(\bm{u}\bm{v}^{\top})^{-}}, \textcolor{orange}{\lb(\delta)^{+}\ub(\bm{u}\bm{v}^{\top})^{-}}, \textcolor{orange}{\ub(\delta)^{+}\lb(\bm{u}\bm{v}^{\top})^{-}}, \textcolor{orange}{\ub(\delta)^{+}\ub(\bm{u}\bm{v}^{\top})^{-}}\}\\
        \textcolor{green}{\lb(\delta)^{-}\ub(\bm{u}\bm{v}^{\top})^{+}} & = \min\{\textcolor{green}{\lb(\delta)^{-}\lb(\bm{u}\bm{v}^{\top})^{+}}, \textcolor{green}{\lb(\delta)^{-}\ub(\bm{u}\bm{v}^{\top})^{+}}, \textcolor{green}{\ub(\delta)^{-}\lb(\bm{u}\bm{v}^{\top})^{+}}, \textcolor{green}{\ub(\delta)^{-}\ub(\bm{u}\bm{v}^{\top})^{+}}\}\\
        \textcolor{purple}{\ub(\delta)^{-}\ub(\bm{u}\bm{v}^{\top})^{-}} & = \min\{\textcolor{purple}{\lb(\delta)^{-}\lb(\bm{u}\bm{v}^{\top})^{-}}, \textcolor{purple}{\lb(\delta)^{-}\ub(\bm{u}\bm{v}^{\top})^{-}}, \textcolor{purple}{\ub(\delta)^{-}\lb(\bm{u}\bm{v}^{\top})^{-}}, \textcolor{purple}{\ub(\delta)^{-}\ub(\bm{u}\bm{v}^{\top})^{-}}\}\,.\\
    \end{aligned}
\end{equation}
Lastly, by substituting each term by the minimum of the terms with the same color on \cref{eq:lb_rank1_expanded}:
\begin{equation}
    \begin{aligned}
        \bm{M} \geq \textcolor{blue}{\lb(\delta)^{+}\lb(\bm{u}\bm{v}^{\top})^{+}} + \textcolor{orange}{\ub(\delta)^{+}\lb(\bm{u}\bm{v}^{\top})^{-}} + \textcolor{green}{\lb(\delta)^{-}\ub(\bm{u}\bm{v}^{\top})^{+}} + \textcolor{purple}{\ub(\delta)^{-}\ub(\bm{u}\bm{v}^{\top})^{-}} = \lb(\bm{M})
    \end{aligned}
\end{equation}
Analogously for the upper bound:
\begin{equation}
    \begin{aligned}
        \bm{M} \leq \max\{&\textcolor{blue}{\lb(\delta)^{+}\lb(\bm{u}\bm{v}^{\top})^{+}} + \textcolor{orange}{\lb(\delta)^{+}\lb(\bm{u}\bm{v}^{\top})^{-}} + \textcolor{green}{\lb(\delta)^{-}\lb(\bm{u}\bm{v}^{\top})^{+}} + \textcolor{purple}{\lb(\delta)^{-}\lb(\bm{u}\bm{v}^{\top})^{-}},\\
        &\textcolor{blue}{\lb(\delta)^{+}\ub(\bm{u}\bm{v}^{\top})^{+}} + \textcolor{orange}{\lb(\delta)^{+}\ub(\bm{u}\bm{v}^{\top})^{-}} + \textcolor{green}{\lb(\delta)^{-}\ub(\bm{u}\bm{v}^{\top})^{+}} + \textcolor{purple}{\lb(\delta)^{-}\ub(\bm{u}\bm{v}^{\top})^{-}},\\
        &\textcolor{blue}{\ub(\delta)^{+}\lb(\bm{u}\bm{v}^{\top})^{+}} + \textcolor{orange}{\ub(\delta)^{+}\lb(\bm{u}\bm{v}^{\top})^{-}} + \textcolor{green}{\ub(\delta)^{-}\lb(\bm{u}\bm{v}^{\top})^{+}} + \textcolor{purple}{\ub(\delta)^{-}\lb(\bm{u}\bm{v}^{\top})^{-}},\\
        &\textcolor{blue}{\ub(\delta)^{+}\ub(\bm{u}\bm{v}^{\top})^{+}} + \textcolor{orange}{\ub(\delta)^{+}\ub(\bm{u}\bm{v}^{\top})^{-}} + \textcolor{green}{\ub(\delta)^{-}\ub(\bm{u}\bm{v}^{\top})^{+}} + \textcolor{purple}{\ub(\delta)^{-}\ub(\bm{u}\bm{v}^{\top})^{-}}\}\,.\\
    \end{aligned}
    \label{eq:ub_delta_rank1_expanded}
\end{equation}
where color is used to group related terms and ease the reading. Grouping by color, it is easily found that:
\begin{equation}
    \begin{aligned}
        \textcolor{blue}{\ub(\delta)^{+}\ub(\bm{u}\bm{v}^{\top})^{+}} & = \max\{\textcolor{blue}{\lb(\delta)^{+}\lb(\bm{u}\bm{v}^{\top})^{+}}, \textcolor{blue}{\lb(\delta)^{+}\ub(\bm{u}\bm{v}^{\top})^{+}}, \textcolor{blue}{\ub(\delta)^{+}\lb(\bm{u}\bm{v}^{\top})^{+}}, \textcolor{blue}{\ub(\delta)^{+}\ub(\bm{u}\bm{v}^{\top})^{+}}\}\\
        \textcolor{orange}{\lb(\delta)^{+}\ub(\bm{u}\bm{v}^{\top})^{-}} & = \max\{\textcolor{orange}{\lb(\delta)^{+}\lb(\bm{u}\bm{v}^{\top})^{-}}, \textcolor{orange}{\lb(\delta)^{+}\ub(\bm{u}\bm{v}^{\top})^{-}}, \textcolor{orange}{\ub(\delta)^{+}\lb(\bm{u}\bm{v}^{\top})^{-}}, \textcolor{orange}{\ub(\delta)^{+}\ub(\bm{u}\bm{v}^{\top})^{-}}\}\\
        \textcolor{green}{\ub(\delta)^{-}\lb(\bm{u}\bm{v}^{\top})^{+}} & = \max\{\textcolor{green}{\lb(\delta)^{-}\lb(\bm{u}\bm{v}^{\top})^{+}}, \textcolor{green}{\lb(\delta)^{-}\ub(\bm{u}\bm{v}^{\top})^{+}}, \textcolor{green}{\ub(\delta)^{-}\lb(\bm{u}\bm{v}^{\top})^{+}}, \textcolor{green}{\ub(\delta)^{-}\ub(\bm{u}\bm{v}^{\top})^{+}}\}\\
        \textcolor{purple}{\lb(\delta)^{-}\lb(\bm{u}\bm{v}^{\top})^{-}} & = \max\{\textcolor{purple}{\lb(\delta)^{-}\lb(\bm{u}\bm{v}^{\top})^{-}}, \textcolor{purple}{\lb(\delta)^{-}\ub(\bm{u}\bm{v}^{\top})^{-}}, \textcolor{purple}{\ub(\delta)^{-}\lb(\bm{u}\bm{v}^{\top})^{-}}, \textcolor{purple}{\ub(\delta)^{-}\ub(\bm{u}\bm{v}^{\top})^{-}}\}\,.\\
    \end{aligned}
\end{equation}
Lastly, by substituting each term by the minimum of the terms with the same color on \cref{eq:lb_rank1_expanded}:
\begin{equation}
    \begin{aligned}
        \bm{M} \leq \textcolor{blue}{\ub(\delta)^{+}\ub(\bm{u}\bm{v}^{\top})^{+}} + \textcolor{orange}{\lb(\delta)^{+}\ub(\bm{u}\bm{v}^{\top})^{-}} + \textcolor{green}{\ub(\delta)^{-}\lb(\bm{u}\bm{v}^{\top})^{+}} + \textcolor{purple}{\lb(\delta)^{-}\lb(\bm{u}\bm{v}^{\top})^{-}} = \ub(\bm{M})
    \end{aligned}
\end{equation}}
\end{proof}
\subsection{Properties of the $\mn$ operator}
\label{subsec:mn_operator}
\rebuttal{We also define the operator $\mn(\cdot)=\max\{|\lb(.)|, |\ub(.)|\}$ and certain useful properties about it.}
\begin{lemma}
\rebuttal{Let $\bm{M} = \mn\left(\bm{u}\bm{v}^{\top}\right)$ be a matrix defined by vectors $\bm{u} \in \realnum^{d}$ and $\bm{v} \in \realnum^{d}$. Let $\bm{u}$ ($\bm{v}$) be lower and upper bounded by $\lb(\bm{u})$ and $\ub(\bm{u})$ ($\lb(\bm{v})$ and $\ub(\bm{v})$). Let $\hat{\bm{u}} = \mn(\bm{u})$ and $\hat{\bm{v}} = \mn(\bm{v})$. The matrix $\bm{M}$ can be expressed as:
\begin{equation}
    \begin{aligned}
        \bm{M} = \hat{\bm{u}}\hat{\bm{v}}^{\top}\,,
    \end{aligned}
\end{equation}
resulting in a rank-1 matrix given by vectors $\hat{\bm{u}}$ and $\hat{\bm{v}}$.}
\label{lem:maxnorm_matrix}
\end{lemma}
\begin{proof}[Proof of \cref{lem:maxnorm_matrix}]
Starting with the definition of $\bm{M}$, we have
\begin{equation}
    \begin{aligned}
        \bm{M} & \coloneqq \max\{|\lb(\bm{u}\bm{v}^{\top})|, |\ub(\bm{u}\bm{v}^{\top})|\}\\
        & = \max\{|\min\{\lb(\bm{u})\lb(\bm{v})^{\top}, \lb(\bm{u})\ub(\bm{v})^{\top}, \ub(\bm{u})\lb(\bm{v})^{\top}, \ub(\bm{u})\ub(\bm{v})^{\top}\}|,\\
        & \quad \quad \quad \quad |\max\{\lb(\bm{u})\lb(\bm{v})^{\top}, \lb(\bm{u})\ub(\bm{v})^{\top}, \ub(\bm{u})\lb(\bm{v})^{\top}, \ub(\bm{u})\ub(\bm{v})^{\top}\}|\} & \text{[Def. \cref{eq:IBP_basic_rules}]}\\
        & = \max\{|\lb(\bm{u})\lb(\bm{v})^{\top}|, |\lb(\bm{u})\ub(\bm{v})^{\top}|, |\ub(\bm{u})\lb(\bm{v})^{\top}|, |\ub(\bm{u})\ub(\bm{v})^{\top}|\}\\
        & = \max\{|\lb(\bm{u})||\lb(\bm{v})^{\top}|, |\lb(\bm{u})||\ub(\bm{v})^{\top}|, |\ub(\bm{u})||\lb(\bm{v})^{\top}|, |\ub(\bm{u})||\ub(\bm{v})^{\top}|\}\\
        & = \max\{\max\{|\lb(\bm{u})||\lb(\bm{v})^{\top}|, |\lb(\bm{u})||\ub(\bm{v})^{\top}|\},\\
        & \quad \quad \quad \quad \max\{|\ub(\bm{u})||\lb(\bm{v})^{\top}|, |\ub(\bm{u})||\ub(\bm{v})^{\top}|\}\}\\
        & = \max\{|\lb(\bm{u})|\max\{|\lb(\bm{v})^{\top}|, |\ub(\bm{v})^{\top}|\}, |\ub(\bm{u})|\max\{|\lb(\bm{v})^{\top}|, |\ub(\bm{v})^{\top}|\}\}\\
        & = \max\{|\lb(\bm{u})|, |\ub(\bm{u})|\}\max\{|\lb(\bm{v})^{\top}|, |\ub(\bm{v})^{\top}|\}\\
        & = \hat{\bm{u}}\hat{\bm{v}}^{\top}\,,
    \end{aligned}
\end{equation}
which concludes the proof.
\end{proof}
\begin{lemma}
\rebuttal{Let $\bm{M} = \mn\left(\bm{A} + \bm{B}\right) \in \realnum^{d_1 \times d_2}$ with $\bm{A} \in \realnum^{d_1 \times d_2}$ and $\bm{B} \in \realnum^{d_1 \times d_2}$ being lower and upper bounded matrices. The matrix $\bm{M}$ satisfies:
\begin{equation}
    \begin{aligned}
        \bm{M} \leq \mn(A) + \mn(B)\,.
    \end{aligned}
\end{equation}}
\label{lem:maxnorm_matrix_sum}
\end{lemma}
\begin{proof}[Proof of \cref{lem:maxnorm_matrix_sum}]
\rebuttal{Starting with the definition of $\bm{M}$:
\begin{equation}
    \begin{aligned}
        \bm{M} & \coloneqq \max\{|\lb(\bm{A} + \bm{B})|, |\ub(\bm{A} + \bm{B})|\} = \max\{|\lb(\bm{A}) + \lb(\bm{B})|, |\ub(\bm{A}) + \ub(\bm{B})|\}\\
        & \leq \max\{|\lb(\bm{A})| + |\lb(\bm{B})|, |\ub(\bm{A})| + |\ub(\bm{B})|\}\\
        & \leq \max\{\max\{|\lb(\bm{A})|, |\ub(\bm{A})|\} + |\lb(\bm{B})|, \max\{|\lb(\bm{A})|, |\ub(\bm{A})|\} + |\ub(\bm{B})|\}\\
        & \leq \max\{\max\{|\lb(\bm{A})|, |\ub(\bm{A})|\} + \max\{|\lb(\bm{B})|, |\ub(\bm{B})|\},\\
        & \quad \quad \quad \quad \max\{|\lb(\bm{A})|, |\ub(\bm{A})|\} + \max\{|\lb(\bm{B})|, |\ub(\bm{B})|\}\}\\
        & = \max\{|\lb(\bm{A})|, |\ub(\bm{A})|\} + \max\{|\lb(\bm{B})|, |\ub(\bm{B})|\} = \mn(\bm{A}) + \mn(\bm{B})\,,
    \end{aligned}
\end{equation}}
we conclude the proof.
\end{proof}
\begin{lemma}
\rebuttal{Let $\bm{M} = \mn\left(\delta\bm{A}\right) \in \realnum^{d_1 \times d_2}$ with $\bm{A} \in \realnum^{d_1 \times d_2}$ and $\delta \in \realnum$ being lower and upper bounded. The matrix $\bm{M}$ satisfies:
\begin{equation}
    \begin{aligned}
        \bm{M} = \mn(\delta)\mn(A)\,.
    \end{aligned}
\end{equation}}
\label{lem:maxnorm_matrix_mult}
\end{lemma}
\begin{proof}[Proof of \cref{lem:maxnorm_matrix_mult}]
\rebuttal{Starting with the definition of $\bm{M}$:
\begin{equation}
    \begin{aligned}
        \bm{M} & \coloneqq \max\{|\lb(\delta\bm{A})|, |\ub(\delta\bm{A}|\}\\
        & \quad = \max\{|\min\left\{\lb(\delta)\lb(\bm{A}), \lb(\delta)\ub(\bm{A}), \ub(\delta)\lb(\bm{A}), \ub(\delta)\ub(\bm{A})\right\}|,\\
        & \quad \quad \quad \quad \quad |\max\left\{\lb(\delta)\lb(\bm{A}), \lb(\delta)\ub(\bm{A}), \ub(\delta)\lb(\bm{A}), \ub(\delta)\ub(\bm{A})\right\}|\} & \text{[\cref{eq:IBP_basic_rules}]}\\
        & \quad = \max\left\{|\lb(\delta)\lb(\bm{A})|, |\lb(\delta)\ub(\bm{A})|, |\ub(\delta)\lb(\bm{A})|, |\ub(\delta)\ub(\bm{A})|\right\}\\
        & \quad = \max\left\{|\lb(\delta)||\lb(\bm{A})|, |\lb(\delta)||\ub(\bm{A})|, |\ub(\delta)||\lb(\bm{A})|, |\ub(\delta)||\ub(\bm{A})|\right\}\\
        & \quad = \max\left\{\max\left\{|\lb(\delta)||\lb(\bm{A})|, |\lb(\delta)||\ub(\bm{A})|\right\}\right.,\\
        & \quad \quad \quad \quad \quad \left.\max\left\{|\ub(\delta)||\lb(\bm{A})|, |\ub(\delta)||\ub(\bm{A})|\right\}\right\}\\
        & \quad = \max\left\{|\lb(\delta)|\max\left\{|\lb(\bm{A})|, |\ub(\bm{A})|\right\}\right.,\\
        & \quad \quad \quad \quad \quad \left.|\ub(\delta)|\max\left\{|\lb(\bm{A})|, |\ub(\bm{A})|\right\}\right\}\\
        & \quad = \max\left\{|\lb(\delta)|, |\ub(\delta)|\right\}\max\left\{|\lb(\bm{A})|, |\ub(\bm{A})|\right\} = \mn(\delta)\mn(\bm{A})\,,
    \end{aligned}
\end{equation}}
we finish the proof.
\end{proof}
\begin{proof}[Proof of \cref{theo:mn_PN}]
\rebuttal{Firstly, applying the $\mn$ operator to the Hessian of the verification objective results in:
\begin{equation}
    \begin{aligned}
        \mn\left(\bm{H}_{g}(\bm{z})\right) & \coloneqq \mn\left(\sum_{i = 1}^{k}(c_{ti}-c_{\gamma i})\nabla_{\bm{z} \bm{z}}x_{i}^{(N)}\right) & \text{[\cref{eq:grad_and_hess_objective}]}\\
        & \leq \sum_{i = 1}^{k}\mn\left((c_{ti}-c_{\gamma i})\nabla_{\bm{z} \bm{z}}x_{i}^{(N)}\right) & \text{[\cref{lem:maxnorm_matrix_sum}]}\\
        & = \sum_{i = 1}^{k}|c_{ti}-c_{\gamma i}|\mn\left(\nabla_{\bm{z} \bm{z}}x_{i}^{(N)}\right) & \text{[\cref{lem:maxnorm_matrix_mult}]}\,.\\
    \end{aligned}
\end{equation}
Secondly, we can proceed analogously with the recursive relationship of the Hessian at different layers. For $n = 2, ..., N$:
\begin{equation}
    \begin{aligned}
        \mn\left(\nabla_{\bm{z} \bm{z}}x_{i}^{(n)}\right) & \coloneqq \mn\left(\nabla_{\bm{z}}x_{i}^{(n-1)} \bm{w}_{[n]:i}^{\top} + \bm{w}_{[n]:i}{\nabla_{\bm{z}}x_{i}^{(n-1)}}^{\top}\right.\\
        & \quad \quad \quad \quad + \left.(\bm{w}_{[n]:i}^{\top}\bm{z} + 1) \nabla_{\bm{z} \bm{z}}x_{i}^{(n-1)}\right) & \text{[\cref{eq:hess_CCP_recursive}]}\\
        & \leq \mn\left(\nabla_{\bm{z}}x_{i}^{(n-1)} \bm{w}_{[n]:i}^{\top}\right) + \mn\left(\bm{w}_{[n]:i}{\nabla_{\bm{z}}x_{i}^{(n-1)}}^{\top}\right)\\
        & \quad \quad \quad \quad + \mn\left((\bm{w}_{[n]:i}^{\top}\bm{z} + 1) \nabla_{\bm{z} \bm{z}}x_{i}^{(n-1)}\right) & \text{[\cref{lem:maxnorm_matrix_sum}]}\\
        & = \mn\left(\nabla_{\bm{z}}x_{i}^{(n-1)}\right)|\bm{w}_{[n]:i}^{\top}| + |\bm{w}_{[n]:i}|\mn\left({\nabla_{\bm{z}}x_{i}^{(n-1)}}^{\top}\right) & \text{[\cref{lem:maxnorm_matrix}]}\\
        & \quad \quad \quad \quad + \mn(\bm{w}_{[n]:i}^{\top}\bm{z} + 1)\mn\left(\nabla_{\bm{z} \bm{z}}x_{i}^{(n-1)}\right) & \text{[\cref{lem:maxnorm_matrix_mult}]}\,.\\
    \end{aligned}
\end{equation}
Lastly, for $n=1$, by definition $\nabla_{\bm{z} \bm{z}}x_{i}^{(1)} = \bm{0}_{d \times d}$, which means $\mn\left(\nabla_{\bm{z} \bm{z}}x_{i}^{(1)}\right) = \bm{0}_{d \times d}$}
\end{proof}
\subsection{Convergence of the BaB algorithm to the global minima}

In this Section we demonstrate a key property for verification: convergence to the global minima of \cref{eq:verif_problem}. Let us firstly define some concepts of the BaB algorithm.

Let $S_{k}$ be the subset picked at iteration $k$ of the BaB algorithm and $\{{S_{k}}_{q}| q = 0,1,...\}$ be the sequence of recursive subsets, so that ${S_{k}}_{q} \subset {S_{k}}_{q-1}$ and ${S_{k}}_{0} = S_{k}$. Let $\lb({S_{k}}_{q})$ be the lower bound of subset ${S_{k}}_{q}$ and ${\lb_{k}}_{q} = \min\{\lb({S_{k}}_{q}) | q = 0,1,...\}$ the lower bound in the branch rooted by subset $S_{k}$, we analogously define $\ub({S_{k}}_{q})$ and ${\ub_{k}}_{q}$. Finally, let a \textbf{fathomed} or \textbf{pruned} set ${S_{k}}_{q}$ be a set where $\lb({S_{k}}_{q}) > {\ub_{k}}_{q}$. From \citet[ Definition IV.4]{Horst1996GlobalOpt}: 

\begin{definition}
A bounding operation is called \textbf{consistent} if at every step any unfathomed subset can be further split, and if any infinitely decreasing sequence $\{{S_{k}}_{q}| q = 0,1,...\}$ for successively refined partition elements satisfies:
\begin{equation}
    \lim_{q \to \infty}{({{\ub}_{k}}_{q} - \lb({S_{k}}_{q}))} = 0 \,.
    \label{eq:consistency}
\end{equation}
\label{def:consistency}
\end{definition}
{\bf Remark.} Because $\ub({S_{k}}_{q}) \geq {\ub_{k}}_{q} \geq \lb({S_{k}}_{q})$, it suffices to prove that $\lim_{q \to \infty}{(\ub({S_{k}}_{q}) - \lb({S_{k}}_{q}))}$ to show a bounding operation is consistent.

Another relevant property is the subset selection being \textbf{bound improving}. Let $\mathcal{P}_{k}$ be the set of unfathomed subsets at iteration k ({\tt probs} variable in \cref{alg:bab}), from \citet[Definition IV.6]{Horst1996GlobalOpt}:

\begin{definition}
A subset selection operation is called \textbf{bound improving} if, at least each time after a finite number of steps, $S_{k}$ satisfies the relation:
\begin{equation}
    S_{k} = \argmin\{\lb(S): S \in \mathcal{P}_k\} \,.
    \label{eq:bound_improving}
\end{equation}
Then, this ensures at least one partition element where the actual lower bound is attained is selected for further partition in step $k$ of the algorithm.
\label{def:bound_improving}
\end{definition}

Finally, \citet[Theorem IV.3]{Horst1996GlobalOpt} cover global convergence of general BaB algorithms.

\begin{theorem}
In a BB procedure, suppose that the bounding operation is consistent and the subset selection operation is bound improving. Then the procedure is convergent:
\label{theo:bab_convergence}
\begin{equation}
    \lb \coloneqq \lim_{k \to \infty}{\lb}_{k} = \lim_{k \to \infty} f(\bm{z}^{(k)}) = \min_{\bm{z} \in C_{\text{in}}} f(\bm{z}) =  \lim_{k \to \infty} {\ub}_{k} = \ub\,,
    \label{eq:convergence_BaB}
\end{equation}
where $C_{\text{in}}$ is the feasible set of the initial problem, $\lb_k$ and $\ub_k$ are the global lower and upper bounds at iteration $k$ and $\bm{z}^{(k)}$ is the point where the upper bound $\ub_k$ is attained.
\end{theorem}
\textbf{Remark.} In \cref{theo:bab_convergence}, $\lb_k$ and $\ub_k$ refer to variables $\mtt{global\_lb}$ and $\mtt{global\_ub}$ respectively in \cref{alg:bab}.

\begin{lemma}
Selecting the subset with the lowest lower bound at every iteration of a BaB algorithm is a bound improving subset selection strategy.
\label{lem:lowest_is_bound_improving}
\end{lemma}
\begin{proof}
By definition, when selecting the subset with the lowest lower bound, we are selecting at every iteration $S_{k} = \argmin\{\lb(S): S \in \mathcal{P}_k\}$, which means that \cref{eq:bound_improving} holds at every iteration $k$ and the strategy is bound improving.
\end{proof}

\begin{definition}
Let subset $S_{k}$, $\{{S_{k}}_{q}| q = 0,1,...\}$ the sequence of recursive subsets rooted at ${S_{k}}_0 = S_{k}$ %
a branching operation is convergent iff %
$\lim_{q \to \infty}|{S_k}_q| = 0$.
\label{def:convergent}
\end{definition}

\begin{lemma}
Selecting the widest interval index $i = \argmax{\bm{u}-\bm{l}}$ to split a problem, is a convergent branching operation.
\label{lem:widest_interval_convergence}
\end{lemma}
\begin{proof}
Supose at subset ${S_k}_q$, we have bounds $\bm{l}^{(q)}$ and $\bm{u}^{(q)}$ and indexes $i_1, i_2, ..., i_d$ so that  $u^{(q)}_{i_{1}} - l^{(q)}_{i_{1}} \geq u^{(q)}_{i_{2}} - l^{(q)}_{i_{2}} \geq \dots \geq u^{(q)}_{i_{d}} - l^{(q)}_{i_{d}}$ is the decreasing ordered list of interval widths, then $|{S_{k}}_q| = ||\bm{u}^{(q)} - \bm{l}^{(q)}||_{2} = \sqrt{\sum_{j = 1}^{d}(u^{(q)}_{j} - l^{(q)}_{j})^2} \leq \sqrt{\sum_{j = 1}^{d}(u^{(q)}_{i_{1}} - l^{(q)}_{i_{1}})^2} = \sqrt{d(u^{(q)}_{i_1} - l^{(q)}_{i_1})^2} = (u^{(q)}_{i_1} - l^{(q)}_{i_1})\sqrt{d}$. %
Then, at subset ${S_k}_{q+1}$, with bounds $\bm{l}^{(q+1)}$  and $\bm{u}^{(q+1)}$ and new indices $j_1, j_2, ..., j_d$, $u^{(q+1)}_{i_1} - l^{(q+1)}_{i_1} = \frac{u^{(q)}_{i_1} - l^{(q)}_{i_1}}{2}$ and then $j_{1}$ is either equal to $i_{1}$ or to $i_{2}$, depending on whether $\frac{u^{(q)}_{i_1} - l^{(q)}_{i_1}}{2} > u^{(q)}_{i_2} - l^{(q)}_{i_2}$ or not. Finally, as:
\begin{equation}
    \left\{
    \begin{matrix}
    u^{(q)}_{i_1} - l^{(q)}_{i_1} > \frac{u^{(q)}_{i_1} - l^{(q)}_{i_1}}{2} = u^{(q+1)}_{j_1} - l^{(q+1)}_{j_1} & \text{if}~\frac{u^{(q)}_{i_1} - l^{(q)}_{i_1}}{2} > u^{(q)}_{i_2} - l^{(q)}_{i_2} \\
    u^{(q)}_{i_1} - l^{(q)}_{i_1} > u^{(q)}_{i_2} - l^{(q)}_{i_2} = u^{(q+1)}_{j_1} - l^{(q+1)}_{j_1} & \text{if}~\frac{u^{(q)}_{i_1} - l^{(q)}_{i_1}}{2} \leq u^{(q)}_{i_2} - l^{(q)}_{i_2}
    \end{matrix}
    \right.
\end{equation}
and $u^{(q)}_{i_1} - l^{(q)}_{i_1} \geq 0$, the sequence $\{u^{(q)}_{i_1} - l^{(q)}_{i_1}\}_{q}$ is strictly decreasing and lower bounded by $0$, then $\lim_{q \to \infty} u^{(q)}_{i_1} - l^{(q)}_{i_1} = 0$ must hold. Then $\lim_{q \to 0}(u^{(q)}_{i_1} - l^{(q)}_{i_1})\sqrt{d} = 0$ and as $(u^{(q)}_{i_1} - l^{(q)}_{i_1})\sqrt{d} \geq |{S_{k}}_q| \geq 0$, the property $\lim_{q \to \infty}|{S_k}_q| = 0$ holds and by \cref{def:convergent}, the branching operation is convergent.
\end{proof}

A consequence of \cref{lem:widest_interval_convergence} is the following Corollary.

\begin{corollary}
For any $\bm{z} \in [\bm{l}^{(q)}, \bm{u}^{(q)}]$ , if $\lim_{q \to \infty}|{S_k}_q| = 0$, in the limit $l^{(q)}_i = z_{i} = u^{(q)}_i~\forall i = 1, ..., d, \forall \bm{z} \in [\bm{l}, \bm{u}]$.
\label{cor:limit_intervals}
\end{corollary}

\begin{lemma}
Let widest interval selection be the branching operation, lower bounds obtained by $\alpha$-convexification with $\alpha \geq \max\{0,-\frac{1}{2}\min \{ \lambda_{\text{min}}(\bm{H}_{f}(\bm{z})) : \bm{z} \in [\bm{l},\bm{u}]\}\}$ and upper bounds obtained by evaluating $\ub({S_{k}}_{q}) = g(\bm{z}^{\text{upper}})$ for any $\bm{z}^{\text{upper}} \in [\bm{l}, \bm{u}]$ are consistent.
\label{lem:alpha_conve_consistent}
\end{lemma}

\begin{proof}
By \cref{def:consistency} is sufficient to check that $\lim_{q \to \infty}{(\ub({S_{k}}_{q}) - \lb({S_{k}}_{q}))} = 0$. By definition, the lower bound $\lb({S_{k}}_{q}))$ is the solution to the function in \cref{eq:alpha_conv}  subject to $\bm{z} \in [\bm{l}, \bm{u}]$, %
which will lead to an optimal $\bm{z}^{\text{opt}}$ and $\lb({S_{k}}_{q})) = g_{\alpha}(\bm{z}^{\text{opt}})$. If the upper bound is given by evaluating the objective function at any point $\bm{z}^{\text{upper}} \in [\bm{l}, \bm{u}]$ e.g., $\bm{z}^{\text{upper}} = \bm{z}^{\text{opt}}$ or in our case $\bm{z}^{\text{upper}} = \bm{z}^{\text{PGD}}$, the point obtained by performing PGD over $g$, $\ub({S_{k}}_{q}) = f(\bm{z}^{\text{upper}})$, and their difference is: $\ub({S_{k}}_{q}) - \lb({S_{k}}_{q}) = g(\bm{z}^{\text{upper}}) - g(\bm{z}^{\text{opt}}) - \alpha\sum_{i = 1}^{d}(z^{\text{opt}}_{i} - l^{(q)}_{i})(z^{\text{opt}}_{i} - u^{(q)}_{i})$. By virtue of \cref{cor:limit_intervals}, in the limit, $\bm{l}^{(q)}  = \bm{z}^{\text{opt}} = \bm{z}^{\text{upper}}= \bm{u}^{(q)}$ and therefore $\lim_{q \to \infty}\ub({S_{k}}_{q}) - \lb({S_{k}}_{q}) = g(\bm{l}) - g(\bm{l}) - \alpha\sum_{i = 1}^{d}(l^{(q)}_{i} - l^{(q)}_{i})(l^{(q)}_{i} - l^{(q)}_{i}) = 0$ and the bounds are consistent.
\end{proof}

\begin{lemma}
Let widest interval selection be the branching operation, lower bounds obtained by IBP and upper bounds obtained by evaluating $\ub({S_{k}}_{q}) = g(\bm{z}^{\text{upper}})$ for any $\bm{z}^{\text{upper}} \in [\bm{l}, \bm{u}]$  are consistent.
\label{lem:inter_prop_consistent}
\end{lemma}

\begin{proof}
By \cref{def:consistency} is sufficient to check that $\lim_{q \to \infty}{(\ub({S_{k}}_{q}) - \lb({S_{k}}_{q}))} = 0$. By definition, the lower bound $\lb({S_{k}}_{q})$ is given by $\lb(g(\bm{z})) = \lb(f(\bm{z})_t) - \ub(f(\bm{z})_\gamma)$, see \cref{subsec:IBP_PN}, and the upper bound is given by $\ub({S_{k}}_{q})) = g(\bm{z}^{\text{upper}})$. %
By virtue of \cref{cor:limit_intervals}, in the limit, $\bm{l}^{(q)} = \bm{z}^{\text{upper}}= \bm{u}^{(q)}$, then $\lim_{q \to \infty}\ub({S_{k}}_{q})) = g(\bm{l})$. Then, for \cref{eq:bounds_xnhat}, it can be easily found that $\lim_{q \to \infty}\lb(\bm{\hat{x}}^{(n)}) = \lim_{q \to \infty}\ub(\bm{\hat{x}}^{(n)}) = \bm{W}_{[n]}^{\top}\bm{l}$. Then, for \cref{eq:bounds_xn}, every element in the sets $S$ will converge to the same value and therefore, $\lb(\bm{x}^{(n)}) = \min{S} = \max{S} = \ub(\bm{x}^{(n)})$ for both CCP and NCP for every layer $n$. Finally, for the network's output, because $\lb(x_{i}^{(N)}) = \ub(x_{i}^{(N)})$, then, from \cref{eq:bounds_output}, $\lb(f(\bm{z})_i) = \ub(f(\bm{z})_i) = f(\bm{l})_i, \forall i \in [o]$. Then, $\lb({S_{k}}_{q})) = \lb(f(\bm{z})_t - f(\bm{z})_{\gamma}) = \lb(f(\bm{z})_t) - \ub(f(\bm{z})_{\gamma}) = f(\bm{l})_t - f(\bm{l})_{\gamma} =  f(\bm{z}^{\text{upper}})_t - f(\bm{z}^{\text{upper}})_{\gamma} = \ub({S_{k}}_{q}))$ and the property holds.

\end{proof}
A consequence of \cref{lem:widest_interval_convergence,lem:alpha_conve_consistent,lem:inter_prop_consistent} is:
\begin{corollary}
Any branch and bound procedure with widest interval selection for branching, lowest lower bound subset selection and a bounding operation of: IBP or $\alpha$-convexification, is convergent.
\end{corollary}

\begin{proof}
Because of \cref{lem:widest_interval_convergence}, we have convergence of the branching procedure. Then, due to \cref{lem:inter_prop_consistent,lem:alpha_conve_consistent}, we have bound consistency for both IBP and $\alpha$-convexification bounding mechanisms. By \cref{lem:lowest_is_bound_improving}, we have that the subset selection strategy is bound improving. Therefore, because of Theorem \cref{theo:bab_convergence}, we have that any BaB algorithm with these properties converges to a global minimizer.%
\end{proof}

\subsection{Lower bound of the minimum eigenvalue of the Hessian of PNs}
\label{subsec:details_min_eig}

\begin{proof}[Proof of \cref{prop:lower bounding_Hessian_vector}]
Using the IBP rules from \cref{subsec:IBP_PN,subsec:IBP_details}, we can develop:
\begin{equation}
    \begin{aligned}
        \lb(\delta \cdot \nabla_{\bm{z} \bm{z}}^{2} x_{i}^{(n)}) & = \lb(\delta \cdot (\nabla_{\bm{z}}x_{i}^{(n-1)} \bm{w}_{[n]:i}^{\top} + \{\nabla_{\bm{z}}x_{i}^{(n-1)} \bm{w}_{[n]:i}^{\top}\}^\top & \\
        & \quad + (\bm{w}_{[n]:i}^{\top}\bm{z} + 1) \nabla_{\bm{z} \bm{z}}^{2} x_{i}^{(n-1)})) &
        \text{[\cref{eq:hess_CCP_recursive}]}\\
        & = \lb(\delta \cdot (\nabla_{\bm{z}}x_{i}^{(n-1)}) \bm{w}_{[n]:i}^{\top}) + \lb(\bm{w}_{[n]:i}\delta \cdot (\nabla_{\bm{z}}x_{i}^{(n-1)})^{\top})  & \\
        & \quad + \lb(\delta \cdot (\bm{w}_{[n]:i}^{\top}\bm{z} + 1) \nabla_{\bm{z} \bm{z}}^{2} x_{i}^{(n-1)}) & \text{[Associativity]}\\
        & = \lb(\delta \cdot (\nabla_{\bm{z}}x_{i}^{(n-1)})) {\bm{w}_{[n]:i}^{+\top}} + \ub(\delta \cdot (\nabla_{\bm{z}}x_{i}^{(n-1)})) {\bm{w}_{[n]:i}^{-\top}} & \\
        & \quad + \bm{w}_{[n]:i}^{+}\lb(\delta \cdot (\nabla_{\bm{z}}x_{i}^{(n-1)})^{\top}) + \bm{w}_{[n]:i}^{-}\ub(\delta \cdot (\nabla_{\bm{z}}x_{i}^{(n-1)})^{\top})   &\\
        & \quad + \lb(\delta \cdot (\bm{w}_{[n]:i}^{\top}\bm{z} + 1) \nabla_{\bm{z} \bm{z}}^{2} x_{i}^{(n-1)}) & \text{[Linearity \cref{eq:IBP_basic_rules}]}\\
        & = \lb(\delta \cdot (\nabla_{\bm{z}}x_{i}^{(n-1)})) {\bm{w}_{[n]:i}^{+\top}} + \ub(\delta \cdot (\nabla_{\bm{z}}x_{i}^{(n-1)})) {\bm{w}_{[n]:i}^{-\top}} & \\
        & \quad + \bm{w}_{[n]:i}^{+}\lb(\delta \cdot (\nabla_{\bm{z}}x_{i}^{(n-1)})^{\top}) + \bm{w}_{[n]:i}^{-}\ub(\delta \cdot (\nabla_{\bm{z}}x_{i}^{(n-1)})^{\top})   &\\
        & \quad + \lb(\delta' \nabla_{\bm{z} \bm{z}}^{2} x_{i}^{(n-1)}) & \text{[Definition of $\delta'$]}\\
    \end{aligned}
\end{equation}
Analogously, for the upper bound:
\begin{equation}
    \begin{aligned}
        \ub(\delta \cdot \nabla_{\bm{z} \bm{z}}^{2} x_{i}^{(n)}) & = \lb(\delta \cdot (\nabla_{\bm{z}}x_{i}^{(n-1)} \bm{w}_{[n]:i}^{\top} + \{\nabla_{\bm{z}}x_{i}^{(n-1)} \bm{w}_{[n]:i}^{\top}\}^\top & \\
        & \quad + (\bm{w}_{[n]:i}^{\top}\bm{z} + 1) \nabla_{\bm{z} \bm{z}}^{2} x_{i}^{(n-1)})) &
        \text{[\cref{eq:hess_CCP_recursive}]}\\
        & = \ub(\delta \cdot (\nabla_{\bm{z}}x_{i}^{(n-1)}) \bm{w}_{[n]:i}^{\top}) + \ub(\bm{w}_{[n]:i}\delta \cdot (\nabla_{\bm{z}}x_{i}^{(n-1)})^{\top})  & \\
        & \quad + \ub(\delta \cdot (\bm{w}_{[n]:i}^{\top}\bm{z} + 1) \nabla_{\bm{z} \bm{z}}^{2} x_{i}^{(n-1)}) & \text{[Associativity]}\\
        & = \ub(\delta \cdot (\nabla_{\bm{z}}x_{i}^{(n-1)})) {\bm{w}_{[n]:i}^{+\top}} + \lb(\delta \cdot (\nabla_{\bm{z}}x_{i}^{(n-1)})) {\bm{w}_{[n]:i}^{-\top}} & \\
        & \quad + \bm{w}_{[n]:i}^{+}\ub(\delta \cdot (\nabla_{\bm{z}}x_{i}^{(n-1)})^{\top}) + \bm{w}_{[n]:i}^{-}\lb(\delta \cdot (\nabla_{\bm{z}}x_{i}^{(n-1)})^{\top})   &\\
        & \quad + \ub(\delta \cdot (\bm{w}_{[n]:i}^{\top}\bm{z} + 1) \nabla_{\bm{z} \bm{z}}^{2} x_{i}^{(n-1)}) & \text{[Linearity \cref{eq:IBP_basic_rules}]}\\
        & = \ub(\delta \cdot (\nabla_{\bm{z}}x_{i}^{(n-1)})) {\bm{w}_{[n]:i}^{+\top}} + \lb(\delta \cdot (\nabla_{\bm{z}}x_{i}^{(n-1)})) {\bm{w}_{[n]:i}^{-\top}} & \\
        & \quad + \bm{w}_{[n]:i}^{+}\ub(\delta \cdot (\nabla_{\bm{z}}x_{i}^{(n-1)})^{\top}) + \bm{w}_{[n]:i}^{-}\lb(\delta \cdot (\nabla_{\bm{z}}x_{i}^{(n-1)})^{\top})   &\\
        & \quad + \ub(\delta' \nabla_{\bm{z} \bm{z}}^{2} x_{i}^{(n-1)}) \,, & \text{[Definition of $\delta'$]}\\
    \end{aligned}
\end{equation}
where $\delta' = \delta \cdot (\bm{w}_{[n]:i}^{\top}\bm{z} + 1)$.
\end{proof}

\subsection{Lower bound of the Hessian's minimum eigenvalue for the product of polynomials case}
To verify a product of polynomials network, we need a lower bound of the minimum eigenvalue of its Hessian. In \cref{prop:min_eig_prod_poly} we propose a valid lower bound.

\begin{proposition}
Let $\bm{x}$ and $\bm{y}$ be the input and output of a polynomial. Let
\begin{equation}
    \hat{\bm{J}}_{\bm{z}}(\bm{x}) = \argmax\{\rho(\bm{J}\bm{J}^{\top}): \bm{J} \in [\lb(\bm{J}_{\bm{z}}(\bm{x})),\ub(\bm{J}_{\bm{z}}(\bm{x}))]\} = \max\{|\lb(\bm{J}_{\bm{z}}(\bm{x}))|,|\ub(\bm{J}_{\bm{z}}(\bm{x}))|\} 
    \label{eq:Jhat}
\end{equation}
be the Jacobian matrix with the largest possible norm. Let $\rho$ be the spectral radius of a matrix. For all $\bm{z} \in [\bm{l}, \bm{u}]$, the minimum eigenvalue of the hessian matrix of every position $i$ ($\lambda_{\text{min}}(\nabla_{\bm{z} \bm{z}}^{2} y_{i} )$) satisfies:
\begin{equation}
    \lambda_{\text{min}}\Big(\nabla_{\bm{z} \bm{z}}^{2} y_{i} \Big) \geq \sum_{j = i}^{k} \lambda_{\text{min}}\Big(\frac{\partial y_{i}}{\partial x_{j}}\nabla_{\bm{z} \bm{z}}^{2} x_{j}\Big) - \rho\Big(\hat{\bm{J}}_{\bm{z}}(\bm{x})\hat{\bm{J}}_{\bm{z}}^{\top}(\bm{x})\Big) \cdot \rho\Big(\bm{L_{H}}(\nabla_{\bm{x} \bm{x}}^{2} y_{i})\Big)\,.
\end{equation}
\label{prop:min_eig_prod_poly}
\end{proposition}
\vspace{-4mm}
\begin{proof}
We give the lower bound of $   \lambda_{\text{min}}\Big(\nabla_{\bm{z} \bm{z}}^{2} y_{i} \Big)$ as
\begin{equation}
\begin{aligned}
    \lambda_{\text{min}}\Big(\nabla_{\bm{z} \bm{z}}^{2} y_{i} \Big) & = \lambda_{\text{min}}\Big(\sum_{j = i}^{k} \frac{\partial y_{i}}{\partial x_{j}}\nabla_{\bm{z} \bm{z}}^{2} x_{j} + \bm{J}_{\bm{z}}^{\top}(\bm{x})\nabla_{\bm{x} \bm{x}}^{2} y_{i} \bm{J}_{\bm{z}}(\bm{x})\Big) & \text{[\cref{eq:grad_hess_prod_poly}]}\\
    & \geq \lambda_{\text{min}}\Big(\sum_{j = i}^{k} \frac{\partial y_{i}}{\partial x_{j}}\nabla_{\bm{z} \bm{z}}^{2} x_{j}\Big) + \lambda_{\text{min}}\Big(\bm{J}_{\bm{z}}^{\top}(\bm{x})\nabla_{\bm{x} \bm{x}}^{2} y_{i} \bm{J}_{\bm{z}}(\bm{x})\Big) & \text{[Weyl's inequality]} \\
    & \geq \sum_{j = i}^{k} \lambda_{\text{min}}\Big(\frac{\partial y_{i}}{\partial x_{j}}\nabla_{\bm{z} \bm{z}}^{2} x_{j}\Big) + \lambda_{\text{min}}\Big(\bm{J}_{\bm{z}}^{\top}(\bm{x})\bm{L_{H}}(\nabla_{\bm{x} \bm{x}}^{2} y_{i}) \bm{J}_{\bm{z}}(\bm{x})\Big) & \text{[\cref{eq:bounds_min_eig}]}\\
    & \geq \sum_{j = i}^{k} \lambda_{\text{min}}\Big(\frac{\partial y_{i}}{\partial x_{j}}\nabla_{\bm{z} \bm{z}}^{2} x_{j}\Big) -\rho \Big( \bm{J}_{\bm{z}}^{\top}(\bm{x})\bm{L_{H}}(\nabla_{\bm{x} \bm{x}}^{2} y_{i}) \bm{J}_{\bm{z}}(\bm{x})\Big) & \text{[Definition of $\rho$]}\\
    & \geq \sum_{j = i}^{k} \lambda_{\text{min}}\Big(\frac{\partial y_{i}}{\partial x_{j}}\nabla_{\bm{z} \bm{z}}^{2} x_{j}\Big) -\rho(\bm{J}_{\bm{z}}(\bm{x})\bm{J}_{\bm{z}}^{\top}(\bm{x}))\rho\Big(\bm{L_{H}}(\nabla_{\bm{x} \bm{x}}^{2} y_{i})\Big)  &\\
    & \geq \sum_{j = i}^{k} \lambda_{\text{min}}\Big(\frac{\partial y_{i}}{\partial x_{j}}\nabla_{\bm{z} \bm{z}}^{2} x_{j}\Big) -\rho\Big(\hat{\bm{J}}_{\bm{z}}(\bm{x})\hat{\bm{J}}_{\bm{z}}^{\top}(\bm{x}))\rho(\bm{L_{H}}(\nabla_{\bm{x} \bm{x}}^{2} y_{i})\Big) \,, & \text{[\cref{eq:Jhat}]}\\
\end{aligned}
\end{equation}
where in the second to last inequality we use $\rho(\bm{B}^{\top}\bm{A}\bm{B}) = \rho(\bm{B}\bm{B}^{\top}\bm{A}) \leq \rho(\bm{B}\bm{B}^{\top})\rho(\bm{A}), \forall \bm{A} \in \realnum^{d_{1} \times d_{1}}, \bm{B} \in \realnum^{d_{1} \times d_{2}}\label{note1}$.
\end{proof}

\end{document}